\documentclass{article}

\usepackage{microtype}
\usepackage{graphicx}
\usepackage{subfigure}
\usepackage{booktabs} %
\usepackage{mwe}

\usepackage[accepted]{icml2024}

\usepackage{mathtools}

\usepackage{nicefrac}       %
\usepackage{xcolor}         %

\usepackage{amsmath,amssymb, amsthm}
\usepackage{mathabx}
\usepackage{soul}
\usepackage{pgfplots}
\pgfplotsset{width=10cm,compat=1.9}
\usepgfplotslibrary{external}
\usepackage{amsopn}
\DeclareMathOperator*{\argmin}{argmin}

\newtheorem{theorem}{Theorem} 
\newtheorem{lemma}[theorem]{Lemma}
\newtheorem{definition}[theorem]{Definition}
\newtheorem{proposition}[theorem]{Proposition}
{Assumption} 
\newtheorem{remark}[theorem]{Remark}
\newcommand{\ZZ}{\mathcal{Z}}
\newcommand{\eps}{\varepsilon}
\newcommand{\wt}{\widetilde}

\newcommand{\VV}{\mathcal{V}}
\newcommand{\ptv}{P_{\theta_v}}
\newcommand{\ptvprime}{P_{\theta_{v'}}}
\newcommand{\thetahat}{\widehat{\theta}}
\newcommand{\thetav}{\theta_v}

\newcommand{\PP}{\mathcal{P}}

\usepackage{hyperref} 
\hypersetup{
	colorlinks=true,
    linkcolor=red,
    citecolor=blue,
    filecolor=black,
    urlcolor=black,
}
\usepackage{cleveref}
\crefname{algorithm}{Algorithm}{Algorithms}
\crefname{assumption}{Assumption}{Assumptions}
\crefname{equation}{}{}
\crefname{figure}{Fig.}{Figs.}
\crefname{table}{Table}{Tables}
\crefname{section}{Section}{Sections}
\crefname{subsection}{Section}{Sections}
\crefname{theorem}{Theorem}{Theorems}
\crefname{lemma}{Lemma}{Lemmmas}
\crefname{proposition}{Proposition}{Propositions}
\crefname{definition}{Definition}{Definitions}
\crefname{corollary}{Corollary}{Corollaries}
\crefname{remark}{Remark}{Remarks}
\crefname{example}{Example}{Examples}
\crefname{appendix}{Appendix}{Appendices}

\usepackage{amsmath}
\usepackage{amsfonts}
\usepackage{bbm}
\usepackage{color}
\usepackage{graphicx}
\usepackage{enumitem}
\usepackage{subfigure}
\usepackage{wrapfig}
\usepackage{svg}
\usepackage{algorithm}
\usepackage{algorithmic}
\newcommand{\XX}{\mathcal{X}}

\newcommand{\FFmu}{\mathcal{F}_{\mu, L, D}}
\newcommand{\FFmuz}{\mathcal{F}_{\mu = 0, L, D}}
\newcommand{\WW}{\mathcal{W}}

\newcommand{\ws}{w^{*}}

\newcommand{\pen}{\frac{d}{n \eps}}

\newcommand{\MMpr}{\mathcal{M}_{\text{priv}}}
\newcommand{\MMpu}{\mathcal{M}_{\text{pub}}}

\newcommand{\hf}{\widehat{F}_{X}}

\newcommand{\expec}{\mathbb{E}}

\newcommand{\bbrho}{\mathbb{B}_{\rho^*(\VV)}}
\newcommand{\psucc}{P_{\text{succ}}}
\newcommand{\dkl}{D_{\text{KL}}}
\newcommand{\wtq}{\widetilde{Q}}

\newcommand{\alg}{\mathcal{A}}

\newcommand{\ERlower}{\mathbb{E}\hf(\alg(X)) - \hf^*}

\newcommand{\xon}{x_{1:n}}
\newcommand{\xoni}{x_{1:n}^i}
\newcommand{\pu}{\text{PrivUnit}}
\newcommand{\sphere}{\mathbb{S}}
\newcommand{\uniform}{\text{Unif}}

\usepackage{mathtools}

\definecolor{Periwinkle}{rgb}{0.94823529, 0.31411765, 0.59411765}

\newcommand{\RR}{\mathcal{R}}

\newcommand{\Al}{\mathcal{A}}

\newcommand{\wpr}{w_{\text{priv}}}

\newcommand{\kpr}{K_{\text{priv}}}
\newcommand{\kpu}{K_{\text{pub}}}
\newcommand{\npr}{n_{\text{priv}}}
\newcommand{\npu}{n_{\text{pub}}}
\newcommand{\emprisk}{\mathcal{M}_{\text{emp}}(\epsilon, \npr, n, d)}
\newcommand{\popLDPrisk}{\mathcal{M}^{\text{loc}}_{\text{pop}}(\epsilon, \npr, n, d)}
\newcommand{\popDPrisk}{\mathcal{M}_{\text{pop}}(\epsilon, \delta, \npr, n, d)}
\newcommand{\popDPriskpure}{\mathcal{M}_{\text{pop}}(\epsilon, \delta = 0, \npr, n, d)}
\newcommand{\xpu}{X_{\text{pub}}}
\newcommand{\xpr}{X_{\text{priv}}}
\newcommand{\lap}{\text{Lap}}
\newcommand{\var}{\text{Var}}
\newcommand{\epsemi}{\eps\text{-semi-DP}}
\newcommand{\pr}{\mathbb{P}}
\newcommand{\BB}{\mathcal{B}}
\newcommand{\hfpu}{\widehat{F}_{pub}}
\newcommand{\hfpr}{\widehat{F}_{priv}}
\newcommand{\wpu}{w^*_{pub}}

\newcommand{\ERMrisk}{\mathcal{R}_{\text{ERM}}(\eps, \npr, n, d, L, D, \mu = 0)}
\newcommand{\ERMriskmu}{\mathcal{R}_{\text{ERM}}(\eps, \npr, n, d, L, D, \mu)}
\newcommand{\SCOriskmu}{\mathcal{R}_{\text{SCO}}(\eps, \delta, \npr, n, d, L, D, \mu)}
\newcommand{\SCOrisk}{\mathcal{R}_{\text{SCO}}(\eps, \delta, \npr, n, d, L, D, \mu = 0)}
\newcommand{\SCOriskdelz}{\mathcal{R}_{\text{SCO}}(\eps, \delta=0, \npr, n, d, L, D, \mu = 0)}
\newcommand{\SCOriskdelzmu}{\mathcal{R}_{\text{SCO}}(\eps, \delta=0, \npr, n, d, L, D, \mu)}
\newcommand{\Renyi}{R\'enyi }

\newcommand{\trunc}{\lfloor q - \theta \rceil_{\eta}}
\newcommand{\trunca}{\lfloor \alg(X) - \theta \rceil_{\eta}}
\newcommand{\pmone}{\{\pm1\}^d}
\newcommand{\pmones}{\{\pm1\}}

\newcommand{\ptheta}{P_\theta}
\newcommand{\attack}{\mathcal{I}}
\newcommand{\duchisample}{\mathcal{M}_{\text{Duchi}}}
\newcommand{\duchiset}{\widetilde{\mathcal{M}}_{\text{Duchi}}}
\newcommand{\semiduchi}{\widetilde{\alg}_{\text{semi-Duchi}}}
\graphicspath{{figures/}{figures/cifar10}{figures/LDP}}

\usepackage[textsize=tiny]{todonotes}

\icmltitlerunning{Optimal Differentially Private Model Training with Public Data}

\begin{document}

\twocolumn[
\icmltitle{Optimal Differentially Private Model Training with Public Data}

\icmlsetsymbol{equal}{*}

\begin{icmlauthorlist}
\icmlauthor{Andrew Lowy}{yyy}
\icmlauthor{Zeman Li}{xxx}
\icmlauthor{Tianjian Huang}{xxx}
\icmlauthor{Meisam Razaviyayn}{xxx}
\end{icmlauthorlist}

\icmlaffiliation{xxx}{Department of Industrial \& Systems Engineering, University of Southern California, Los Angeles, CA, USA}
\icmlaffiliation{yyy}{University of Wisconsin-Madison, Wisconsin Institute of Discovery, Madison, WI, USA}

\icmlcorrespondingauthor{Andrew Lowy}{alowy@wisc.edu}

\icmlkeywords{Machine Learning, ICML}

\vskip 0.3in
]

\printAffiliationsAndNotice{} %

\begin{abstract}
Differential privacy (DP) ensures that training a machine learning model does not leak private data. 
In practice, we may have access to auxiliary public data that is free of privacy concerns. 
In this work, we assume access to a given amount of public data and settle the following fundamental open questions: \textit{1. What is the optimal (worst-case) error of a DP model trained over a private data set while having access to side public data?
2. How can we harness public data to improve DP model training in practice?} We consider these questions in both the local and central models of pure and approximate DP. To answer the first question, we prove tight (up to log factors) lower and upper bounds that characterize the optimal
error rates of three fundamental problems: mean estimation, empirical risk minimization, and stochastic convex optimization.  
We show that the optimal error rates can be attained (up to log factors) by either discarding private data and training a public model, or treating public data like it is private  and using an optimal DP algorithm. To address the second question, we develop novel algorithms that are ``even more optimal'' (i.e. better constants) than the asymptotically optimal approaches described above. For local DP mean estimation, our algorithm is \ul{optimal including constants}. Empirically, our algorithms show benefits over the state-of-the-art.  %
\end{abstract}

\vspace{-0.4cm}
\section{Introduction}
\vspace{-0.2cm}
Training machine learning models on people's data can leak sensitive information, violating their privacy~\cite{inversionfred, shokri2017membership, carlini2021extracting}. \textit{Differential Privacy (DP)} prevents such leaks by providing a rigorous guarantee that no attacker can learn too much about any individual's data~\cite{dwork2006calibrating}. DP has been successfully deployed by various companies~\cite{apple-differential-privacy, thakurta2017learning,rappor14,ding2017collecting}, and by government agencies
~\cite{census-differential-privacy}. However, a major hindrance to more widespread adoption of DP  is that DP-trained models are less accurate than their non-private counterparts.

Leveraging \textit{public data}---that is free of privacy concerns---appears to be a promising and practically important avenue for closing the accuracy gap between DP and non-private models~\cite{papernotsemi, avent2017blender, feldman2018privacy, amid2022public}. For example, large language models (LLMs) are often pre-trained on public data and fine-tuned on private data~\cite{llmkerrigan2020differentially,llmli2021large,llmyu2021differentially}. 
Public data may be provided by people who volunteer (e.g. product developers or early testers)~\cite{church2005personal, feldman2018privacy} or sell their data. Data that is generated synthetically~\cite{torkzadehmahani2019dp, vietri2020new, boedihardjo2022covariance, he2023algorithmically} or released through a legal process~\cite{klimt2004enron} may serve as additional sources of public data.

The power and limitations of public data vary depending on the particular learning problem and loss function/hypothesis class. To calibrate the effectiveness of a public-data-assisted DP algorithm, we can compare it against two \textit{na\"ive baselines}: 1) ``throw away'' the private data and run an optimal non-private algorithm on the public data; 2) use an optimal DP algorithm on the full data set, treating the public data like it is private data. Some works have identified problems where significant improvements over the na\"ive approaches are possible. For example, \citet{alon2019limits} show that for agnostic PAC learning with a hypothesis class of finite VC-dimension, it is possible to achieve asymptotically smaller sample complexity than the na\"ive baselines. \citet{bassily2020privatequery} show a similar result for private query release with finite VC-dimension. 
On the other hand, for certain problems  it is impossible to do better than the na\"ive approaches: e.g., releasing binary decision stumps~\cite{bassily2020privatequery}.
Understanding what improvements, if any, over the na\"ive approaches are possible for other problems (e.g. optimization) and function classes is interesting.

This work considers DP \textit{model training} with side access to public data: given a loss function, %
find model parameters to approximately minimize the expected training or test loss. We consider \textit{empirical risk minimization} (ERM) and \textit{stochastic convex optimization} (SCO), which correspond to minimizing training loss and expected test loss, respectively. For population mean estimation and SCO, we assume access to in-distribution public data. For %
ERM, the public data may be out-of-distribution. We answer a fundamental question: 
\begin{center}
\vspace{-0.2cm}
\noindent\fbox{
    \parbox{0.95\linewidth}{
    \vspace{-0.cm}
\textbf{Question 1.} What is the 
optimal (minimax) error
of DP model training with side access to public data? Is it possible to achieve smaller error than the na\"ive baselines?
}}  \vspace{-0.2cm}
    \end{center}

\paragraph{Contribution 1: Limitations of Public Data for DP Model Training} To answer \textbf{Question 1}, we characterize the optimal minimax error (up to constants or logarithms)
of \textit{semi-DP}~\cite{beimel2013private,alon2019limits} algorithms---algorithms that are DP w.r.t. private data, but not necessarily DP w.r.t. public data (Definition~\ref{def: semiDP}). \textit{We provide tight 
lower and upper bounds for three fundamental  training problems}: mean estimation of bounded random variables\footnote{Our $\eps$-semi-DP analysis extends to unbounded/heavy-tailed distributions with bounded $k$-th order moment.}, ERM and SCO with Lipschitz convex functions.\footnote{Our results for ERM also cover non-convex loss functions.}  
We consider both the \textit{local}~\cite{whatcanwelearnprivately, duchi13} and \textit{central}~\cite{dwork2006calibrating} models of \textit{pure} ($\delta = 0$) and \textit{approximate} ($\delta > 0$) semi-DP. We prove nine sets of lower and upper bounds: see~\cref{table: summary of minimax rates approx,table: summary of minimax rates,app table: summary of pure rates}.
Our lower bounds imply that \textit{it is impossible to obtain asymptotic improvements over the na\"ive approaches for semi-DP model training in the worst case.}

\begin{figure*}[h!]
    \vspace{-0.1in}
    \begin{center}
        \includegraphics[width = 0.95\textwidth]{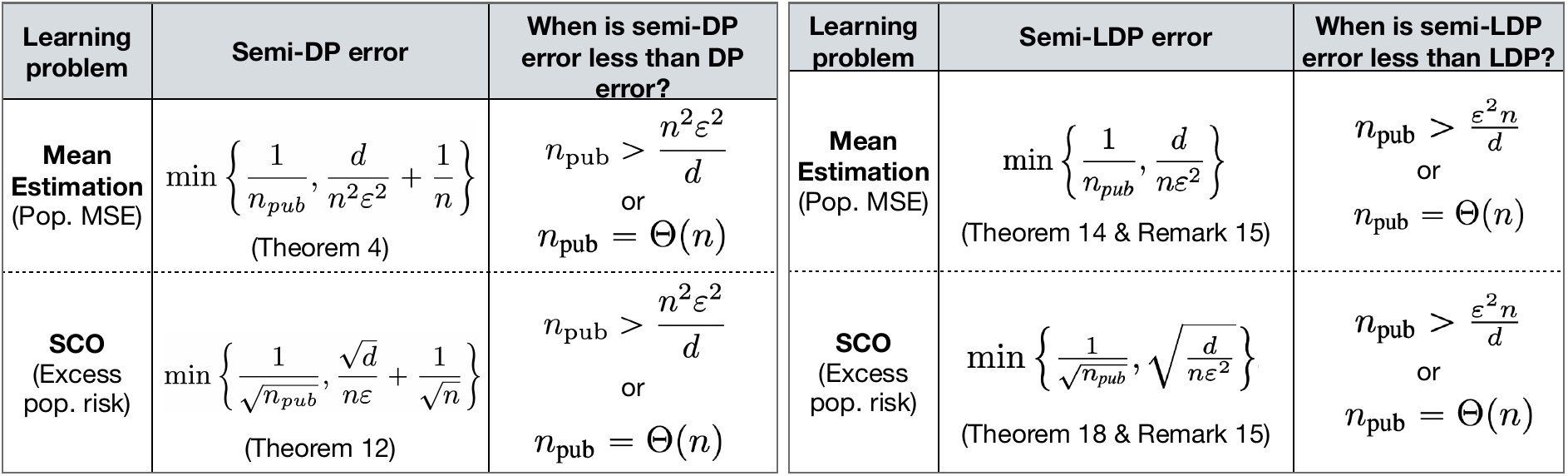}
         \end{center}
      \vspace{-0.15in}
     \caption{\footnotesize 
Minimax optimal error rates for central $(\eps, \delta)$-semi-DP (up to logs) and (local) $(\eps, \delta)$-semi-LDP. 
$n = \npr + \npu$, where $\npr$ ($\npu$) denotes the number of private (public) samples. 
Dependence on $\delta$, range and Lipschitz parameters, constraint set diameter omitted. 
See Appendix for strongly convex SCO results.
}\label{table: summary of minimax rates approx}
\end{figure*}

\begin{figure}
  \centering
  \includegraphics[width = 
  0.46
  \textwidth]{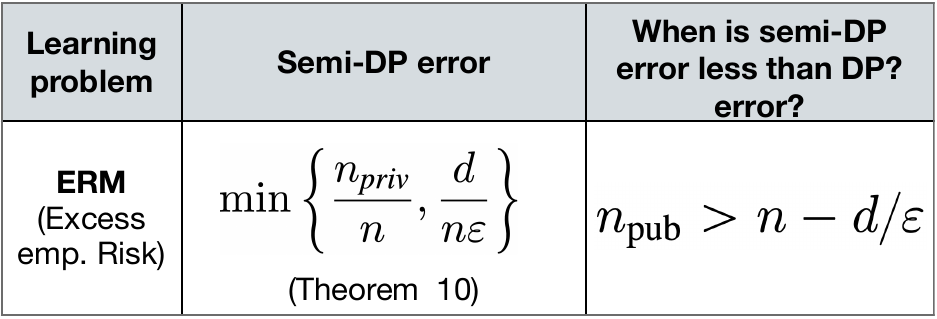}
  \vspace{-.4cm}
  \caption{ \footnotesize
Minimax optimal error rates for $\eps$-semi-DP ERM. See Table~\ref{app table: summary of pure rates} in Appendix for more $\eps$-semi-(L)DP results (e.g. SCO). 
}\label{table: summary of minimax rates}
 \vspace{-.18in}
  \end{figure}

In light of these negative results, it is natural to wonder whether/how one can harness public data for more effective DP model training. 
This leads us to \textbf{Question 2}:
\begin{center}
\vspace{-0.3cm}
\noindent\fbox{
    \parbox{0.95\linewidth}{
    \vspace{-0.cm}
\textbf{Question 2.} 
Can we provide improved performance (e.g. theoretically smaller error and/or superior empirical results) over the na\"ive baselines?
    }
    }
    \vspace{-0.3cm}
    \end{center}
Some prior works have tackled \textbf{Question 2} by imposing additional assumptions and/or shrinking the problem class. For example, \cite{amid2022public} shows that under certain distributional assumptions, public data permits benefits over DP-SGD in linear regression. Also, \cite{zhou2020bypassing, kairouz2021nearly} show that by imposing certain ``low-rank subspace'' assumptions on the gradients in DP-SGD, public data can help attain dimension-independent rates for DP ERM. In contrast, we consider \textbf{Question 2} \textit{without imposing any additional assumptions and without shrinking the loss function/data distribution class.}

\paragraph{Contribution 2: Power of Public Data for DP Model Training}
To address \textbf{Question 2}, we develop novel (central and local) semi-DP algorithms that
 add less noise than would be necessary to privatize the full data set (including public). By doing so, we can achieve \textit{optimal (worst-case) error bounds with significantly improved constants} over the asymptotically optimal na\"ive algorithms. 
Our local semi-DP mean estimation algorithm is \ul{optimal including constants}. 

We complement our theoretical analyses with extensive numerical experiments. Our experiments show that \textit{our algorithms outperform the na\"ive approaches, even when the optimal DP algorithm is pre-trained on the public data}. For example, our~\cref{alg: weighted dp-sgd} achieves a significant improvement in CIFAR-10 image classification tasks, reducing test error by $8-9\%$ for logistic regression and by at most $18.9\%$ for Wide-ResNet, compared with the na\"ive approaches. We also identify a linear regression problem in which \textit{DP-SGD diverges, but our algorithm converges} with small error using $\npu = 0.1n$ public samples: see Figure~\ref{fig:special_case}. Moreover, \textit{our algorithm consistently outperforms the state-of-the-art} public-data-assisted mirror-descent (PDA-MD) of \cite{amid2022public}. %

\subsection{Techniques and Challenges}
\paragraph{Lower bounds:}
We develop and utilize a variety of techniques to prove our lower bounds.

To prove our central $(\eps, \delta)$-semi-DP SCO lower bound in~\cref{thm: approx DP SCO}, we build upon the techniques of~\citet{dwork2015robust, bassily2020privatequery}. A key challenge is finding a distribution whose mean has small norm and proving that this distribution is still hard enough for any semi-DP algorithm to estimate in $\ell_2$-distance. To accomplish this, we make three main innovations: First, we modify the \textit{tracing attack} of \cite{dwork2015robust} by incorporating more aggressive truncation; this is used to infer membership of many individuals in the data set, even under weak $\ell_2$-accuracy guarantees. Second, we construct a novel distribution in which the prior (mean) is drawn from a shorter interval than the posterior (data). This innovation is crucial for obtaining our tight bounds, but also complicates the analysis. Thus, we provide a novel generalization of the \textit{fingerprinting lemma}~\cite{bun2014fingerprinting, bsu17, kms22} that permits analysis of our construction.
We provide more details on these proofs in~\cref{app: CDP}.

For our central pure $\eps$-semi-DP population lower bounds in \cref{thm: pure semi-DP mean estimation,thm: pure sdp sco}, we develop a \textit{Semi-DP Fano's method} (\cref{lem: semiDP Fano}). In combination with the reduction from estimation to testing, \cref{lem: semiDP Fano} generalizes DP Fano's method~\cite{acharya2021differentially} and strengthens classical Fano's method~\cite{yu1997assouad}. 
We build on the tools of~\citet{bd14} in our proof of~\cref{lem: semiDP Fano}.

Our ERM lower bound (\cref{thm: DP convex ERM}) uses a novel semi-DP \textit{packing argument}: we construct $2^{d/2}$ ``hard'' data sets with well-separated sample means, and use the group privacy property of DP to show that any semi-DP algorithm must make large error on at least one of these data sets. Such arguments have long been used to prove pure DP lower bounds~\cite{hardt2010geometry, bst14}. 
However, to the best of our knowledge, our proof is the first to extend packing arguments to the semi-DP setting. The main challenge is in carefully constructing the private and public data sets so as to force semi-DP $\alg$ to make large error, even when $\alg$ has access to public data. 

For our local semi-DP lower bounds (\cref{thm: LDP mean estimation,thm: LDP SCO convex}), we build on the sophisticated techniques of~\citet{duchi2019lower}. The main idea of our proofs is to combine \textit{Assouad's method}~\cite{duchi2021lecture} with bounds on the mutual information between the input and output of semi-DP algorithms. 

\vspace{-0.4cm}
\paragraph{Algorithms:}
We develop novel algorithms to obtain smaller error (including constants) than the asymptotically optimal na\"ive algorithms. Our algorithms are simple. For example, we propose a central $(\eps, \delta)$-semi-DP estimator that puts different weights on the private and public samples and adds (Gaussian) noise that is calibrated to the private weight. 
We choose the weights depending on the privacy level, dimension, and number of public samples to trade off smaller sensitivity (less noise) with larger variance on the public data. An optimal choice of weights minimizes the variance of our unbiased estimator, leading to smaller error than the optimal DP estimator. Our local semi-DP algorithm simply applies the optimal local randomizer of~\citet{bhowmick2018protection} \textit{only} to the private samples, and averages the noisy private data with the raw public data. 
For SCO, we develop semi-DP stochastic gradient methods that use our mean estimation algorithms to estimate the gradient of the loss function in each iteration of training.

\subsection{Preliminaries and Notation}
Let $\| \cdot \|$ be the $\ell_2$ norm and
$\XX$ denote a data universe. 
Function $g: \WW \to \mathbb{R}$ is \textit{$\mu$-strongly convex} if $g(\alpha w + (1- \alpha) w') \leq \alpha g(w) + (1 - \alpha) g(w') - \frac{\alpha (1-\alpha) \mu}{2}\|w - w'\|^2$ for all $\alpha \in [0,1]$ and all $w, w' \in \WW$. If $\mu = 0,$ we say $g$ is \textit{convex}. 
Function $f: \WW \times \XX \to \mathbb{R}$ is \textit{uniformly $L$-Lipschitz} in $w$ if $\sup_{x \in \XX} |f(w, x) - f(w',x)| \leq L\|w - w'\|$. $\mathbb{B} = \{x \in \mathbb{R}^d | \|x\| \leq 1\}$ denotes the unit $\ell_2$-ball in $\mathbb{R}^d$. 

\begin{definition}[Differential Privacy~\cite{dwork2006calibrating}]
\label{def: DP}
Let $\varepsilon \geq 0, ~\delta \in [0, 1).$ Randomized algorithm $\Al: \XX^n \to \mathcal{W}$ is \textit{$(\varepsilon, \delta)$-differentially private} (DP) if for all %
$X, X' \in \XX^n$ differing in one sample
and all measurable subsets $S \subseteq \WW$, we have
$\mathbb{P}(\alg(X) \in S) \leq e^\varepsilon \mathbb{P}(\alg(X') \in S) + \delta$. 
\end{definition}
Definition~\ref{def: DP} prevents attackers from learning much more about an individual's data than if the data had not been used for training.
If $\delta = 0$, we write $\eps$-DP and say ``pure'' DP.

\begin{definition}[Zero-Concentrated Differential Privacy (zCDP)~\cite{bun16}]
A randomized algorithm $\mathcal{A}: \XX^n \to \mathcal{W}$ satisfies $\rho$-zero-concentrated differential privacy ($\rho$-zCDP) if for all 
pairs of adjacent data sets $X, X' \in \XX^n$
and all $\alpha \in (1, \infty)$, we have $D_\alpha(\alg(X) || \alg(X')) \leq \rho \alpha$,
where $D_\alpha(\alg(X) || \alg(X'))$ is the $\alpha$-\Renyi divergence\footnote{For distributions $P$ and $Q$ with probability density/mass functions $p$ and $q$, $D_\alpha(P || Q) := \frac{1}{\alpha - 1}\ln \left(\int p(x)^{\alpha} q(x)^{1 - \alpha}dx\right)$~\cite{renyi}.} between the distributions of $\alg(X)$ and $\alg(X')$.
\end{definition}

zCDP is weaker than $\varepsilon$-DP, but stronger than $(\varepsilon, \delta)$-DP: see Proposition~\ref{prop:bun1.3} in Appendix~\ref{app: approx DP weaker than zcdp}. 

We now define \textit{semi-DP}~\cite{beimel2013private, alon2019limits}, a relaxation of DP that permits $\alg$ to violate the privacy of the public data: 
\begin{definition}[Semi-Differential Privacy~\cite{beimel2013private, alon2019limits}]
\label{def: semiDP}
Let $n = \npu + \npr$. 
Consider an algorithm $\Al: \XX^n \to \WW$ that takes data $X = (\xpr, \xpu) \in \XX^{\npr} \times \XX^{\npu}$
as input. $\alg$ is (centrally) \textit{$(\eps, \delta)$-semi-DP} if $\alg(\cdot, \xpu)$ is $(\eps, \delta)$-DP for all $\xpu \in \XX^{\npu}$.
\vspace{-0.3cm}
\end{definition}
We define \textit{$\rho$-semi-zCDP} analogously. We define \textit{semi-LDP} in~\cref{sec: LDP}, which is a similar relaxation of local DP (LDP)~\cite{whatcanwelearnprivately, duchi13}. 
To distinguish Definition~\ref{def: semiDP} from semi-LDP, we sometimes refer to algorithms that satisfy Definition~\ref{def: semiDP} as \textit{centrally} semi-DP.

\subsection{Roadmap}
We study the central model of semi-DP in~\cref{sec: CDP}. We give tight error bounds for mean estimation in~\cref{sec: CDP ME,app sec: pure semi-DP mean estimation}, and a novel algorithm with better constants than the asymptotically optimal algorithms in~\cref{sec: even more optimal CDP ME}. In~\cref{sec: DP ERM,sec: DP SCO}, we characterize the optimal excess risk of semi-DP ERM and SCO respectively. We give an improved semi-DP algorithm for SCO in~\cref{even better DP SCO}. 
In~\cref{sec: LDP}, we turn to the local model of semi-DP. We characterize the optimal error rates for semi-LDP mean estimation in~\cref{sec: LDP ME} and SCO in~\cref{sec: LDP SCO}. We give semi-LDP algorithms with improved error 
in~\cref{sec: even more optimal LDP estimation,sec: even more optimal LDP SCO}. We 
experimentally evaluate our algorithms in~\cref{sec: experiments,app: experiments}.
In~\cref{app: intro}, we discuss related works in more detail.
Due to the page limit, some results and all proofs are presented in the Appendix.

\section{Optimal Centrally Private Model Training with Public Data}
\label{sec: CDP}

\subsection{Optimal Semi-DP Mean Estimation}
\label{sec: CDP ME}

In this section, we determine the minimax optimal semi-DP error rates for estimating 
the mean of a bounded distribution.

Consider the following problem: given $\npr$ private samples~$\xpr \subseteq \mathbb{B}$ and $\npu$ public samples~$\xpu \subseteq \mathbb{B}$,
drawn i.i.d. from an unknown distribution $P$ on $\mathbb{B}$, estimate the population mean $\alg(X) \approx \expec_{x \sim P}[x]$ subject to the constraint that $\alg$ satisfies semi-DP. Defining $n = \npr + \npu$, we will characterize the minimax squared error of population mean estimation under $(\varepsilon, \delta)$-semi-DP: 
\begin{multline}
\mathcal{M}_{\text{pop}}(\varepsilon, \delta, \npr, n, d) := \inf_{\mathcal{A} \in \mathbb{A}_{\varepsilon, \delta}(\mathbb{B})} \sup_{P 
} \\
\expec_{\mathcal{A}, X \sim P^n 
}\left[\| \mathcal{A}(X) - \expec_{x \sim P}[x]\|^2\right],
\end{multline}
where $\mathbb{A}_{\varepsilon, \delta}(\mathbb{B})$ denotes the set of all $(\varepsilon, \delta)$-semi-DP estimators $\alg: \mathbb{B}^n \to \mathbb{B}$, and $|\xpr| = \npr$.

\begin{theorem}
\label{thm: CDP mean estimation}
    Let $\eps \lesssim 1/\log(nd)$, $\delta \ll 1/\npr$. Then, there is a constant $C > 0$ such that 
\begin{multline}
\label{eq: delta > 0}
\ell(d, n) \min\left\{\frac{1}{\npu}, \frac{d}{n^2 \eps^2} + \frac{1}{n}\right\} \leq \popDPrisk \\ 
\leq C \min\left\{\frac{1}{\npu}, \frac{d \ln(1/\delta)}{n^2 \eps^2} + \frac{1}{n}\right\},
\end{multline}
where $1/\ell(d,n)$ is logarithmic in $d$ and $n$.
\end{theorem}

\cref{app: CDP pop mean est} has the proof of~\cref{thm: CDP mean estimation} 
and the pure $\eps$-semi-DP result. 
\begin{remark}
\label{rem: symmetric assumption}
Technically, our lower bound proof requires us to assume that $\alg = (\alg^1, \ldots, \alg^d)$ is \textit{symmetric}, meaning $\alg^j = \alg^l$ for all $j, l \in [d]$. This is a very reasonable assumption: to our knowledge, every algorithm  
that has been proposed in the literature (for $\ell_2$ geometry) is symmetric. Further, the concurrent work of~\citet{ullah2024public} gives an alternative proof that eliminates this assumption.\footnote{The work of~\citet{ullah2024public} appeared on arXiv March 6, 2024. The first version of our paper appeared on arXiv on June 26, 2023, while v2 added the $d$-dimensional $(\eps, \delta)$-semi-DP lower bounds (\cref{thm: CDP mean estimation} and \cref{thm: approx DP SCO}) and appeared on February 14, 2024.}
\end{remark}
Na\"ive algorithms attain the optimal rates: the throw-away estimator $\alg(X) = \frac{1}{\npu} \sum_{x \in \xpu} x$ has MSE $O(1/\npu)$, and the DP (hence semi-DP) Gaussian mechanism has MSE $\widetilde{O}(d/(\eps n)^2 + 1/n)$. In the next subsection, we show that these two algorithms have suboptimal constants: we provide 
improved (smaller error)
estimators. 

\subsection{An ``Even More Optimal'' Semi-DP Algorithm for Mean Estimation}
\label{sec: even more optimal CDP ME}
Before presenting our improved semi-DP algorithms, we precisely describe the worst-case error of the optimal na\"ive algorithms discussed in the preceding subsection.  We will consider $\rho$-\textit{semi-zCDP},
which facilitates a sharp characterization of the privacy of the Gaussian mechanism. Note that the lower bound in~\cref{eq: delta > 0} also holds for semi-zCDP, since $\eps^2/\ln(1/\delta)$-zCDP implies $(O(\eps), \delta)$-DP, by Proposition~\ref{prop:bun1.3}.

\begin{definition}
\label{def: P(B,V)}
Let 
$\PP(B, V)$ be the collection of all distributions $P$ on $\mathbb{R}^d$ such that for any 
$x \sim P$, we have $\var(x) = V^2$and $\|x\| \leq B$, $P$-almost surely. 
\end{definition}

\begin{lemma}
\label{lem: zcdp naive}
The error of the $\rho$-semi-zCDP throw-away algorithm $\alg(X) = \frac{1}{\npu}\sum_{x \in \xpu} x$ is 
\[
\sup_{P \in \PP(B,V)} \expec_{X \sim P^n
}\left[\| \mathcal{A}(X) - \expec_{x \sim P}[x]\|^2\right] = \frac{V^2}{\npu}. 
\]
Further, let $\Bar{X}$ be the average of the public and private samples. The minimax error of the $\rho$-zCDP Gaussian mechanism $\mathcal{G}(X) = \Bar{X} + \mathcal{N}\left(0, \sigma^2\mathbf{I}_d \right)$ is 
\begin{multline*}
\inf_{\rho\text{-zCDP}~
\mathcal{G}} \sup_{P \in \PP(B,V)} \expec_{\mathcal{G}, X \sim P^n
}\left[
\| \mathcal{G}(X) - \expec_{x \sim P}[x]\|^2
\right]
\\= \frac{2 d B^2}{\rho n^2} + \frac{V^2}{n}. 
\end{multline*}
\end{lemma}
Intuitively, it seems like the na\"ive estimators in~Lemma~\ref{lem: zcdp naive} do not harness the public and private data in the most effective way possible, despite being optimal up to constants: Throw-away fails to utilize the private data at all, while the Gaussian mechanism gives equal weight to $\xpr$ and $\xpu$ (regardless of $\rho, d, \npr$), and provides unnecessary privacy for $\xpu$. We now present a $\rho$-semi-zCDP estimator that is ``even more optimal'' than the na\"ive estimators, 
meaning our estimator has smaller worst-case error (accounting for constants). We define the family of \textit{Weighted-Gaussian} estimators: 
\begin{multline}
\label{eq: weighted gauss}
\alg_r(X) := \sum_{x \in \xpr} rx + \sum_{x \in \xpu} \left(\frac{1 - \npr r}{\npu}\right)x \\
+ \mathcal{N}\left(0, \frac{2 B^2 r^2}{\rho} 
\mathbf{I}_d\right),
\end{multline}
for $r \in [0, 1/\npr]$. This estimator
can recover both the throw-away and standard Gaussian mechanisms by choosing $r = 0$ or $r = 1/n$. 
Intuitively, as $\rho/d$ shrinks, the accuracy cost of adding privacy noise grows, so we should choose smaller $r$ to reduce the sensitivity of $\alg_r$. On the other hand, smaller $r$ increases the variance of $\alg_r$ on $\xpu$. 
By choosing $r$ optimally (depending on $\rho, d, \npr, B, V$), $\alg_r$ achieves smaller MSE than both throw-away and the Gaussian mechanism:\footnote{We find the optimal choice of $r^*$ explicitly in the proof of Proposition~\ref{prop: even more optimal ME} in~\cref{app: even more optimal ME}.} 
\begin{proposition}
\label{prop: even more optimal ME}
$\alg_r$ is $\rho$-semi-zCDP, and $\exists r > 0$ such that
\begin{multline}
\label{eq: always better}
\sup_{P \in \PP(B, V)} \expec_{X \sim P^n}\left[\| \alg_r(X) - \expec_{x \sim P}[x] \|^2\right] \\
< \min\left(\frac{V^2}{\npu},  \frac{2 d B^2}{\rho n^2} + \frac{V^2}{n} \right).
\end{multline}

Further, if $\frac{V^2}{\npu} \leq \frac{2 d B^2}{\rho n^2}$, then the advantage of $\alg_r$ is 
\begin{multline}
\label{eq: case I}
\sup_{P \in \PP(B, V)} \expec_{X \sim P^n}\left[\| \alg_r(X) - \expec_{x \sim P}[x] \|^2\right] \\
\leq \left(\frac{q}{q + s^2} \right) \min\left(\frac{V^2}{\npu},  \frac{2 d B^2}{\rho n^2} + \frac{V^2}{n} \right),
\end{multline}
where $q = 2 + \frac{\npr \rho V^2}{d B^2}$ and $s = \frac{V \npr \sqrt{\rho}}{B \sqrt{d \npu}}$. 
\end{proposition}

When $\frac{V^2}{\npu} \leq \frac{2 d B^2}{\rho n^2}$, the throw-away estimator outperforms the DP Gaussian mechanism and our Weighted Gaussian estimator outperforms both of these estimators by a factor of at least $q/(q + s^2)$. Also, $q/(q + s^2) \in [1/2, 1]$ for allowable $s, q$. For example, if $n= 10,000$, $d = n/100$, $B = 25 V$, $\rho = 0.1$, and $\npu = 0.008n$, then the MSE of our Weighted Gaussian $\alg_r$ is smaller than the MSE of throw-away and standard Gaussian by a multiplicative factor of $\approx 1.98$.

Figures \ref{fig:mean-est-d-1000-n-500-appendix}-\ref{fig:mean-est-d-1000-n-2000-appendix} in \cref{app:semi-dp-mean-estimation} show that our estimator outperforms both na\"ive baselines for $d$-dimensional Bernoulli data with $\rho = 0.5$ (regardless of whether or not 
throw-away outperforms the Gaussian mechanism). 

For pure $\epsemi$, using Laplace noise instead of Gaussian noise in~\cref{eq: weighted gauss} yields an estimator with smaller error than the $\eps$-DP Laplace mechanism and throw-away.

\subsection{Optimal Semi-DP Empirical Risk Minimization}
\label{sec: DP ERM}
For a given (fixed) $X = (\xpr, \xpu) \in \XX^n
$ and parameter domain $\WW$, 
consider the ERM problem: \[
\min_{w \in \WW} \left(\widehat{F}_X(w) :=  \frac{1}{n}\sum_{j=1}^{n} f(w,x_j)\right),\]
where $f(\cdot, x)$ is a loss function and $\npr = |\xpr|$ samples are private. 
We discuss practical applications of semi-DP ERM beyond ML in~\cref{app: practical applications of ERM}. We measure the (in-sample) performance of a training algorithm $\alg: \XX^n \to \WW$ on the data set $X$ by its \textit{excess empirical risk} \[
\expec_{\alg}\hf(\alg(X)) - \hf^* = \expec_{\alg}\hf(\alg(X)) - \min_{w' \in \WW} \hf(w').
\]

\begin{definition}
Let $\FFmu$ be the set of all functions $f: \WW \times \XX \to \mathbb{R}$ that are uniformly $L$-Lipschitz and $\mu$-strongly convex ($\mu \geq 0$) in $w$ for some convex compact $\WW \subset \mathbb{R}^d$ with $\ell_2$-diameter bounded by $D > 0$ and some set $\XX$.
\end{definition}

Let $\mathbb{A}_\varepsilon$ contain all $\epsemi$ algorithms $\alg: \XX^n \to \WW$ for some $\XX, \WW$. Define the minimax excess empirical risk of $\varepsilon$-semi-DP (strongly) convex ERM as
\begin{align}
\label{eq:DP minimax excess empirical risk}
&\mathcal{R}_{\text{ERM}}(\varepsilon, \npr, n, d, L, D, \mu)  \\
&:=\inf_{\mathcal{A} \in \mathbb{A}_\varepsilon} \sup_{f \in \FFmu} \sup_{
\left\{
X\in \XX^{\npr} \times \XX^{\npu} 
\right\}
}
\expec_{\alg} \hf(\alg(X)) - \hf^*. \nonumber
\end{align}

\begin{theorem}
\label{thm: DP convex ERM}
There are absolute constants $0 < c \leq C$ s.t.
\begin{multline}
\label{eq: convex ERM excess risk}
c LD\min\left\{ \frac{\npr}{n}, \pen \right\} \leq \ERMrisk \\
\leq C LD  \min\left\{ \frac{\npr}{n}, \pen \right\}.
\end{multline}
\end{theorem}
See~\cref{app: ERM} for the $\mu > 0$ result and proofs.

\begin{remark}
\label{rem: nonconvex ERM}
The same minimax risk bound~\cref{eq: convex ERM excess risk} holds up to a logarithmic factor if we replace $\FFmuz$ by set of all Lipschitz \textit{non-convex} loss functions in the definition~\cref{eq:DP minimax excess empirical risk}. However, the optimal semi-DP algorithms are inefficient for non-convex loss functions. See~\cref{app: ERM} for details.
\end{remark}

\subsection{Optimal Semi-DP Stochastic Convex Optimization}
\label{sec: DP SCO}
In stochastic convex optimization (SCO), we are given $n$ i.i.d. samples from an unknown distribution $X \sim P^n$ (with $\npr$ of them being private), and aim to  approximately minimize the expected \textit{population loss} $F(w) := \expec_{x \sim P}[f(w, x)]$. We measure the quality of a learner $\alg$ by its \textit{excess population risk} 
\begin{multline*}
\expec_{\alg, X \sim P^n} F(\alg(X)) - F^* \\
:= \expec_{\alg, X \sim P^n} \expec_{x \sim P}[f(\alg(X), x)] 
- \min_{w' \in \WW} \expec_{x \sim P} f(w', x).
\end{multline*}
Denote the minimax optimal semi-DP excess risk by 
\begin{multline}
\label{eq:DP minimax excess pop risk}
\mathcal{R}_{\text{SCO}}(\varepsilon, \delta, \npr, n, d, L, D, \mu) \\
:= \inf_{\mathcal{A} \in \mathbb{A}_{\varepsilon, \delta}} \sup_{f \in \FFmu} \sup_{P
} \expec_{\alg, X\sim P^n} F(\alg(X)) - F^*,
\end{multline}

where $\mathbb{A}_{\eps, \delta}$ contains all $(\eps, \delta)$-semi-DP algorithms $\alg: \XX^n \to \WW$ for some $\XX, \WW$, and $|\xpr| = \npr$. 

\begin{theorem}
\label{thm: approx DP SCO}
Let $\eps \lesssim 1/\log(nd)$ and $\delta \ll 1/n$. Then, there is a constant $C>0$ such that
{\small
\begin{multline*}
\ell(d,n) LD\min\left\{\frac{1}{\sqrt{\npu}}, \frac{\sqrt{d}}{n \eps} + \frac{1}{\sqrt{n}}\right\} \\
\leq \SCOrisk \\
\leq C LD\min\left\{ \frac{1}{\sqrt{\npu}}, \frac{\sqrt{d \ln(1/\delta)}}{n \eps} + \frac{1}{\sqrt{n}}\right\},
\end{multline*}
}
where $1/\ell(d,n)$ is logarithmic in $d$ and $n$. 
\normalsize
\end{theorem}

We provide the $\delta = 0$ and $\mu$-strongly convex results ($\mu>0$), and proofs in~\cref{app: DP SCO}. Remark~\ref{rem: symmetric assumption} also applies to~\cref{thm: approx DP SCO} 

Let us compare the semi-DP bound for SCO in~\cref{thm: approx DP SCO} with the ERM bound in~\cref{thm: DP convex ERM} when $d = 1 = L = D$. Depending on the values of $\eps$ and $\npr$, the minimax excess population risk (``test loss'') of SCO may either be larger \textit{or smaller} than the excess empirical risk (``training loss'') of ERM. For example, if $\eps \approx 1$, 
then the semi-DP excess empirical risk $\Theta(1/n)$ is smaller than the excess population risk $\Theta(1/\sqrt{n})$. On the other hand, suppose $\eps \approx 1/n$  and $\npr \approx n^{2/3}$: then \textit{the semi-DP excess empirical risk $\Theta(1/n^{1/3})$ is larger than the excess population risk $\Theta(1/\sqrt{n})$}. This is surprising: for both non-private learning 
and DP learning (with $\npu = 0$), the optimal error of ERM is never larger than that of SCO. While it may seem counter-intuitive that minimizing the training loss can be harder than minimizing test loss, there is a natural explanation: For SCO, a small amount of public data gives us free information about the private data, since $X \sim P^n$ is i.i.d. by assumption. By contrast, for ERM, the public data does not give us any information about the private data, since $X$ is not i.i.d. 

\subsection{Semi-DP SCO with an ``Even More Optimal'' Gradient Estimator}
\label{even better DP SCO}
Our~\cref{alg: weighted dp-sgd} is a noisy stochastic gradient method that uses the ``even more optimal'' Weighted-Gaussian estimator~\cref{eq: gauss} to estimate $\nabla F(w_t)$ in iteration $t$.\footnote{In~\cref{alg: weighted dp-sgd}, we re-parameterize by setting $r = \frac{\alpha}{\kpr}$ for $\alpha \in [0,1]$.} In~\cref{alg: weighted dp-sgd}, $\text{clip}_C(x) := \argmin_{y \in \mathbb{R}^d, \|y\|\leq C}\|x - y\|$ is the Euclidean projection onto the centered $\ell_2$-ball of radius $C$.

We give privacy and excess risk guarantees for~\cref{alg: weighted dp-sgd} and describe an accelerated variant of~\cref{alg: weighted dp-sgd} in~\cref{app: even better DP SCO}. The excess risk of our algorithm is smaller than the state-of-the-art excess risk for a linear-time DP algorithm whose privacy analysis does not require convexity~\cite{lr21fl}. We empirically evaluate our algorithm in~\cref{sec: experiments}.

\begin{algorithm}[H]
\caption{Semi-DP-SGD via Weighted-Gaussian Gradient Estimation}
\label{alg: weighted dp-sgd}
\begin{algorithmic}[1]
\STATE {\bfseries Input:} $T \in \mathbb{N},$ 
clip threshold $C > 0$, 
stepsizes $\{\eta_t\}_{t \in [T]}$, batch sizes $\kpr \in [\npr]$, $\kpu \in [\npu]$,  weight parameter $\alpha \in [0,1]$, 
noise parameter $\sigma^2 > 0$. 
 \STATE Initialize $w_0 \in \WW$.
 \FOR{$t \in \{0, 1, \cdots, T-1\}$} 
  \STATE Draw random batch of $\kpr$ private samples $B_t^{priv}$. 
  \STATE Draw random batch of $\kpu$ public samples $B_t^{pub}$. 
  \STATE Draw privacy noise $v_t \sim \mathcal{N}\left(0, \sigma^2\mathbf{I}_d\right)$.
  \STATE $\widetilde{g}_t \gets \alpha \left[\frac{1}{\kpr}\sum_{x \in B_t^{priv}} \text{clip}_C(\nabla f(w_t, x)) + v_t \right] + \frac{1 -  \alpha}{\kpu}\sum_{x \in B_t^{pub}} \nabla f(w_t, x)$. 
 \STATE Update $w_{t+1} := \Pi_{\WW}[w_t - \eta_t \widetilde{g}_t]$.
 \ENDFOR \\
\STATE {\bfseries Output:} $w_T$ or an average of the iterates $\{w_t\}_{t \in [T]}$.
\end{algorithmic}
\end{algorithm}

\section{Optimal Locally Private Model Training with Public Data}
\label{sec: LDP}

We now turn to a stronger privacy notion that we refer to as \textit{semi-local DP} (semi-LDP). Semi-LDP guarantees \textit{privacy for each private $x_i$, without requiring person $i$ to trust others} (e.g. central server). 
Semi-LDP
generalizes LDP~\cite{whatcanwelearnprivately,duchi13}, which has been deployed in industry
~\cite{apple-differential-privacy, rappor14, ding2017collecting}.

Following~\citet{duchi2019lower}, we permit algorithms to be \textit{fully interactive}: algorithms may adaptively query the same person $i$ multiple times over the course of $T$ communication rounds. For example, $n$ cell phone users send messages to a server over $T$ rounds, and message $Z_{i,t} \in \ZZ$ sent by user $i$ in round $t$ can depend on the previous messages $\{Z_{j,t}\}_{j \leq n, t' \leq t}$. 
Semi-LDP requires the messages $\{Z_{i,t}\}_{t\in [T]}$ 
to be semi-DP:

\begin{definition}[Semi-Local Differential Privacy]
\label{def: semiLDP}
The $T$-round interactive algorithm $\Al$ is
$(\eps, \delta)$-semi-LDP if the transcript 
$Z = \{Z_{i,t}\}_{i \leq n, t \leq T}$ is $(\eps, \delta)$-semi-DP: i.e. for all $\xpu \in \XX^{\npu}$, all 
adjacent 
$\xpr \sim \xpr'$ and all $S \subset \ZZ^{n T}$, 
$\pr(Z \in S | X = (\xpr, \xpu)) \leq 
\;\;\;\; e^\eps \pr(Z \in S | X = (\xpr', \xpu)) + \delta.  $
\end{definition}
Definition~\ref{def: semiLDP} is stronger than Definition~\ref{def: semiDP}, since the latter only requires the final output of $\alg$ to be semi-DP. 

\subsection{Optimal Semi-LDP Mean Estimation}
\label{sec: LDP ME}
We will characterize the minimax squared error of $\varepsilon$-semi-LDP $d$-dimensional mean estimation:
\begin{multline}
\label{eq:LDP minimax risk}
\mathcal{M}^{\text{loc}}_{\text{pop}}(\varepsilon, \npr, n, d)\\
:= \inf_{
\mathbb{A}^{\text{loc}}_\varepsilon(\mathbb{B})
} \sup_{P} \expec_{\mathcal{A}, X \sim P^n}\left[\| \mathcal{A}(X) - \expec_{x \sim P}[x]\|^2\right],
\end{multline}
where $\mathbb{A}^{\text{loc}}_\varepsilon(\mathbb{B})$ contains all (fully interactive) $\varepsilon$-semi-LDP estimators $\alg: \mathbb{B}^n \to \mathbb{B}$, 
and $|\xpr| = \npr$.

\begin{theorem}
\label{thm: LDP mean estimation}
Let $\eps \in (0, 1]$. There are absolute constants $0 < c \leq C$ s.t.
\begin{multline}
c\min\left\{\frac{1}{\npu}, \frac{d}{n \eps^2}\right\}\leq \mathcal{M}^{\text{loc}}_{\text{pop}}(\varepsilon, \npr, n, d) \\ \leq C \min\left\{\frac{1}{\npu}, \frac{d}{n \eps^2}\right\}.
\end{multline}
\end{theorem}

\begin{remark}[Approximate Semi-LDP]
\label{rem: delta > 0}
\cref{thm: LDP mean estimation} still holds if we replace $\mathbb{A}_{\eps}^{\text{loc}}$ in the definition of $\popLDPrisk$ by the set of all
$(\eps, \delta)$-semi-LDP estimators $\alg$ for which either $\delta < 1/2$ and $\alg$ is ``compositional''~\cite{duchi2019lower} (e.g. sequentially interactive~\cite{duchi13}) or $\delta < 1/2^d$. 
\end{remark}

The upper bound in~\cref{thm: LDP mean estimation} is the minimum of the error of the throw-away estimator and the optimal $\eps$-LDP 
estimator of~\citet{duchi13}. The LDP estimator of~\citet{duchi13} takes the form \[
\widetilde{\mathcal{M}}_{\text{Duchi}}(X) = \frac{1}{n}\sum_{i=1}^n \mathcal{M}_{\text{Duchi}}(x_i),
\]   
where $\mathcal{M}_{\text{Duchi}}(x_i)$ samples a vector uniformly from a carefully chosen subset of $\mathbb{B}$, depending on $x_i$. 

\subsection{An ``Even More Optimal'' Semi-LDP Estimator}
\label{sec: even more optimal LDP estimation}
By applying $\mathcal{M}_{\text{Duchi}}$ only to the private samples, we obtain an $\eps$-semi-LDP algorithm
with smaller error than the asymptotically optimal $\widetilde{\mathcal{M}}_{\text{Duchi}}$. 
Define the \textit{Semi-LDP $\alg_{\text{Semi-Duchi}}$}:
\begin{equation}
\label{eq: semiduchi}
\alg_{\text{Semi-Duchi}}(X) = \frac{1}{n} \left[\sum_{x \in \xpr} \mathcal{M}_{\text{Duchi}}(x) + \sum_{x' \in \xpu} x' \right].
\end{equation}
Let $P$ be a distribution on $\mathbb{B}$ with $V^2 := \expec\|x - \expec_{x \sim P}[x]\|^2$.  

\begin{lemma}
\label{lem: privunit error}
Let $c > 0$ such that $\expec_{x \sim P}\| \duchisample(x) - \expec_{x \sim P}[x]\|^2 = \frac{c d}{n \eps^2}$, so that $\expec_{X \sim P^n}\| \duchiset(X) - \expec_{x \sim P}[x]\|^2 = \frac{c d}{n \eps^2} + \frac{V^2}{n}$. Then, 
\begin{multline*}
\expec_{X \sim P^n}\left[\left\|\alg_{\text{Semi-Duchi}}(X) - \expec_{x \sim P}[x] \right\|^2\right] \\ =  \frac{\npr}{n} \cdot \frac{c d}{n \eps^2} + \frac{\npu}{n} \cdot \frac{V^2}{n}.
\end{multline*}
\end{lemma}
The constant $c$ in Lemma~\ref{lem: privunit error} that bounds the error of $\duchisample$ may depend on the distribution $P$. Lemma~\ref{lem: privunit error} shows that given any $P$, the error of our semi-LDP estimator is smaller than the error of $\duchiset$. Quantitatively, \textit{the MSE of our estimator is smaller than the MSE of $\duchiset$ by a factor of $\npr/n$} if the privacy noise error term is dominant (e.g. if $d \gg \eps^2)$.

\subsection{A Semi-LDP Estimator with Optimal Constants}
In this subsection, we consider the task of estimating the average of data $X$ on the unit sphere $\mathcal{S}^{d-1} \subset \mathbb{R}^d$. We give a semi-LDP estimator that is \textit{truly optimal}---i.e. \textit{our estimator has the smallest MSE, including constants}---among a large class of unbiased semi-LDP estimators of $\Bar{X}$. Our semi-LDP estimator, $\alg_{\text{semi-PrivU}}$ takes a similar shape to $\semiduchi$, but uses \textit{PrivUnit}~\cite{bhowmick2018protection} instead of $\duchisample$ as the LDP randomizer in~\eqref{eq: semiduchi}. We recall PrivUnit in~\cref{alg:pu} in Appendix~\ref{app: LDP}.

\begin{proposition}
\label{prop: semi-LDP PrivUnit is truly optimal}
Let $\RR: \mathbb{S}^{d-1} \to \ZZ$ be an $\eps$-LDP randomizer, $\MMpr$ and $\MMpu$ be aggregation protocols, and $\alg(X) = \frac{1}{n} \left[\MMpr(\RR(x_1), \ldots, \RR(x_{\npr})) + \MMpu(\xpu) \right]$. Assume $\expec\left[\MMpr(\RR(x_1), \cdots, \RR(x_{\npr})) | \xpr \right] = \sum_{x \in \xpr} x$ and $\expec\left[\MMpu(\xpu) | \xpr \right] = \sum_{x \in \xpu} x$~$\forall X = (\xpr, \xpu) \in \left(\mathbb{S}^{d-1}\right)^n$ Then, 
\begin{multline*}
\sup_{X \in \left(\mathbb{S}^{d-1}\right)^n} \expec_{\alg_{\text{semi-PrivU}}}\left[\left\|\alg_{\text{semi-PrivU}}(X) - \Bar{X}\right\|^2\right] \\\leq \sup_{X \in \left(\mathbb{S}^{d-1}\right)^n} \expec_{\alg}\|\alg(X) - \Bar{X}\|^2.
\end{multline*}
\end{proposition}

Proposition~\ref{prop: semi-LDP PrivUnit is truly optimal} is proved by extending the analysis of~\citet{asilocal} to the semi-LDP setting.

\subsection{Optimal Semi-LDP Stochastic Convex Optimization}
\label{sec: LDP SCO}
We will characterize the minimax optimal excess population risk of semi-LDP SCO
\begin{multline}
\label{eq:LDP minimax excess pop risk}
\mathcal{R}_{\text{SCO}}^{\text{loc}}(\varepsilon, \npr, n, d, L, D, \mu) \\ := \inf_{\mathcal{A} \in \mathbb{A}_{\varepsilon}^{\text{loc}}} \sup_{f \in \FFmu} \sup_{P
} \expec_{\alg, X \sim P^n} F(\alg(X)) - F^*,
\end{multline}
where $\mathbb{A}^{\text{loc}}_\varepsilon$ denotes the set of all algorithms $\alg: \XX^n \to \WW$ that are $\varepsilon$-semi-LDP for some $\XX$ and $\WW$,
and exactly $\npr$ samples in $X$ are private. 

\begin{theorem}
\label{thm: LDP SCO convex}
Let $\eps \in (0, 1]$ and $h(\varepsilon, \npr, n, d, L, D) := LD \min\left\{\frac{1}{\sqrt{\npu}}, \sqrt{\frac{d}{n \eps^2}}\right\}$. 
There are absolute constants~$c$ and $C$ with $0 < c \leq C$, such that 
\begin{multline*}
c \ h(\varepsilon, \npr, n, d, L, D) \leq \mathcal{R}_{\text{SCO}}^{\text{loc}}(\varepsilon, \npr, n, d, L, D,  \mu = 0)  \\ \leq C \ h(\varepsilon, \npr, n, d, L, D). 
\end{multline*}
\end{theorem}
See~\cref{app: LDP SCO} for the $\mu > 0$ result and proofs. 
The first term in the upper bound is achieved by throwing away $\xpr$ and running SGD on $\xpu$~\cite{ny}. The second term in the upper bound is achieved by the one-pass \textit{LDP-SGD} of~\citet{duchi13}.
Remark~\ref{rem: delta > 0} also applies to~\cref{thm: LDP SCO convex}.

\subsection{``Even More Optimal'' Semi-LDP SCO Algorithm}
\label{sec: even more optimal LDP SCO}
We give a semi-LDP algorithm, called \textit{Semi-LDP-SGD}, with smaller excess risk than the optimal LDP-SGD of~\citet{duchi13}. Essentially, Semi-LDP-SGD runs as follows: In each iteration $t \in [n]$, we draw a random sample $x_t \in X$ without replacement. If $x_t \in \xpr$, update $w_{t+1} = \Pi_{\WW}\left[w_t - \eta \duchisample\left(\nabla f(w_t, x_t)\right)\right]$; if $x_t \in \xpu$, instead update $w_{t+1} = \Pi_{\WW}\left[w_t - \eta \nabla f(w_t, x_t)\right]$. 
See \cref{alg: semi-Ldp-sgd} in~\cref{app: more optimal LDP SCO} for pseudocode. 

\begin{proposition}
\label{prop: more optimal semi-LDP SCO}
Let $f \in \FFmuz$, let $P$ be any distribution and $\eps \leq d$. \cref{alg: semi-Ldp-sgd} is $\eps$-semi-LDP. Further, 
there is an absolute constant $c$ such that 
the output $\alg(X) = \Bar{w}_n$ of~\cref{alg: semi-Ldp-sgd} satisfies 
\begin{multline}
\expec_{\alg, X \sim P^n} [F(\Bar{w}_n) - F^*] \\ \leq c \frac{LD}{\sqrt{n}} \max\left\{
\sqrt{\frac{d}{\eps^2}} \sqrt{\frac{\npr}{n}}, \sqrt{\frac{\npu}{n}}\right\}.
\end{multline}

\end{proposition}

Thus, \cref{alg: semi-Ldp-sgd} has smaller excess risk than LDP-SGD, roughly by a factor of $\sqrt{\npr/n}$.

\section{Numerical Experiments}
\label{sec: experiments}
In this section, we 
empirically evaluate the performance of 
four different semi-DP algorithms: 1. \textit{Throw-away} %
(i.e. minimize the public loss).  
2. \textit{DP-SGD}~\cite{abadi16,de2022unlocking}.  
3. \textit{PDA-MD}~\cite{amid2022public}, which is the state-of-the-art semi-DP algorithm for training convex models.
4. \textit{Our~\cref{alg: weighted dp-sgd}}. 
Unless otherwise noted, we evaluate all algorithms with \textit{``warm start,''} which means finding a minimizer $w_{\text{pub}}$ of the public loss and initializing training at $w_{\text{pub}}$. 
The hyperparameters of each algorithm were carefully tuned. See~\cref{app: experiments} for details on the experimental setups and additional results.

\begin{figure}[!h]
    \begin{minipage}{.4\textwidth}
    \centering
    \includegraphics[width=2.5in]{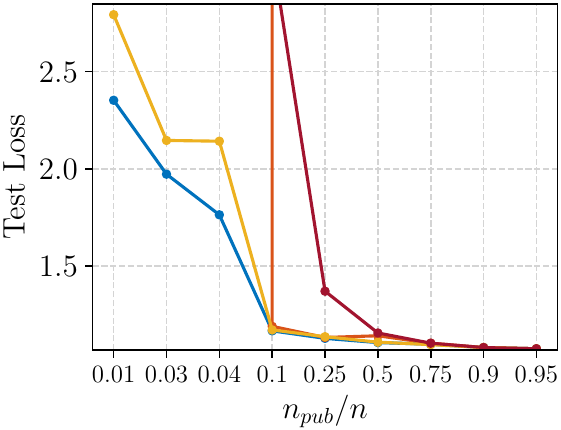}
    \llap{\shortstack{%
        \includegraphics[scale=0.6]{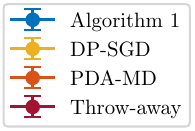}\\
        \rule{0ex}{1.2in}%
      } \rule{0.05in}{0ex}}
    \vspace{-0.15in}
        \caption{\footnotesize Test loss vs. $\npu/n$. $\eps=2$.}
        \label{fig:acc_vs_ratio_eps_2}
    \end{minipage}
\end{figure}

\begin{figure}[!h]
    \begin{minipage}{.4\textwidth}
    \centering
    \includegraphics[width=2.5in]{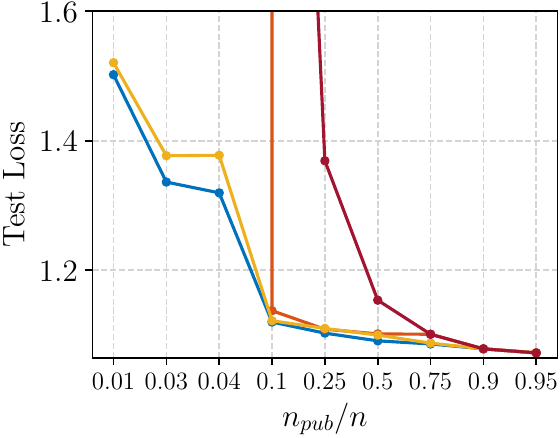}
    \llap{\shortstack{%
        \includegraphics[scale=0.6]{figures/legend_cropped.pdf}\\
        \rule{0ex}{1.2in}%
      } \rule{0.05in}{0ex}}
    \vspace{-0.15in}
    \caption{\footnotesize Test loss vs. $\npu/n$. $\eps=4$.}
    \label{fig:acc_vs_ratio_eps_4}
    \end{minipage}
\end{figure}

\begin{figure}[!h]
    \hspace{0.15in}
    \begin{minipage}{.4\textwidth}
        \includegraphics[width=2.4in]{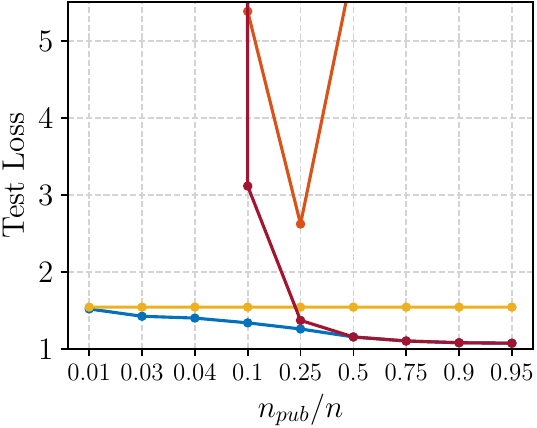}
        \llap{\shortstack{%
        \includegraphics[scale=0.6]{figures/legend_cropped.pdf}\\
        \rule{0ex}{1.2in}%
      } \rule{0.05in}{0ex}}
        \vspace{-.15in}
        \caption{\footnotesize Test loss vs. $\npu/n$. $\eps=4$, without warm-start.}
        \label{fig:acc_vs_ratio_eps_4_non_warm}
    \end{minipage}
\end{figure}

\begin{figure}[!h]
    \vspace{0.1in}
    \begin{minipage}{.4\textwidth}
        \centering
        \includegraphics[width=2.5in]{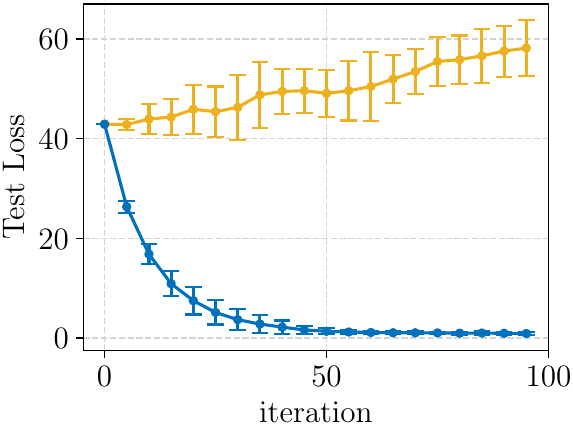}
        \llap{\shortstack{%
        \includegraphics[scale=0.3]{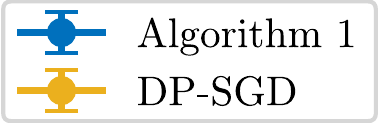}\\
        \rule{0ex}{0.9in}%
      } \rule{0.14in}{0ex}}
        \vspace{-.15in}
        \caption{\footnotesize Test loss vs. iterations. $\frac{n_\text{pub}}{n}=0.1$}
        \label{fig:special_case}
    \end{minipage}
\end{figure}

\textit{\textbf{Our~\cref{alg: weighted dp-sgd} achieves the smallest test loss}} among the semi-DP baselines across different levels of $\eps$ (privacy) and $\npu$: Figures~\ref{fig:acc_vs_ratio_eps_2}-\ref{fig:acc_vs_ratio_eps_4} show results for $(\eps, \delta = 10^{-5})$-semi-DP linear regression with synthetic Gaussian data. 
In the Appendix, we evaluate the algorithms in several other tasks: e.g., logistic regression and Wide-ResNet: see Figures \ref{fig:cifar_acc_vs_ratio_1}-\ref{fig:cifar_acc_vs_eps_2} and \ref{fig:loss_vs_ratio_cifar_10_non_cvx_1}-\ref{fig:loss_vs_ratio_cifar_10_non_cvx_2}. Our results indicate that \textit{\cref{alg: weighted dp-sgd} consistently outperforms all baselines}.

\label{app:linear-regression-special}

\textit{\textbf{Our~\cref{alg: weighted dp-sgd} can converge even when DP-SGD diverges:}} Figure~\ref{fig:special_case} gives an example in which DP-SGD diverges but \cref{alg: weighted dp-sgd} converges. We used the following parameters: $d=50$, $n=1000$, $\varepsilon=0.01$, and $\frac{n_\text{pub}}{n}=0.1$. 

\textit{\textbf{\cref{alg: weighted dp-sgd} performs well even without ``warm start''}}, whereas PDA-MD performs poorly: e.g., Figures~\ref{fig:acc_vs_ratio_eps_4_non_warm} and (in Appendix) \ref{fig:acc_vs_eps_ratio_0_1_non_warm} show that PDA-MD even performs worse than throw-away. By contrast, \cref{alg: weighted dp-sgd} is resilient to the ``cold start'' condition and outperforms all baselines.
In certain practical applications, such as advertising and healthcare, samples are often obtained in an online/streaming fashion, precluding the possibility of ``warm start''. 

\textit{\textbf{More public data always improves the performance of \cref{alg: weighted dp-sgd}}}, but does not always benefit PDA-MD. The main reason for this is that our algorithm more effectively handles the increasing privacy noise that is needed to maintain semi-DP with increasing $\npu/n$. We give details and numerical evidence of this explanation in~\cref{Appendix: effect of different ratio}.

Appendix~\ref{app: experiments} contains other findings, too. For example, Figure~\ref{fig:hessian_reg} shows that PDA-MD is sensitive to Hessian regularization parameter, which requires extra tuning on complicated tasks. By contrast, \cref{alg: weighted dp-sgd} does not use any Hessian information. 

\section{Conclusion}
We considered training DP models with side access to public data. Theoretically, we characterized the optimal error bounds (up to constants) for three fundamental problems:  mean estimation, empirical risk minimization, and stochastic convex optimization. We show that it is impossible to improve over the na\"ive semi-DP algorithms asymptotically, in the worst case. 
Algorithmically, we developed new optimal methods for semi-DP learning that have smaller error than the asymptotically optimal algorithms. Empirically, we showed that our algorithms are effective in training semi-DP models. Our work raises interesting open questions. For instance, why do certain learning/optimization problems benefit more from public data than others? Is there some general underlying property that (don't) permit asymptotic benefits over the na\"ive baselines? Also, what can be said about semi-DP learning with \textit{out-of-distribution} public data? Lastly, it would be useful to have an extensive empirical study that evaluates the efficacy of combining our semi-DP algorithms with various other techniques, such as dimensionality reduction~\cite{yu2021not,pillar24}. 

\section*{Impact Statement}
Our work provides algorithms for protecting the privacy of individuals who contribute training data. 
Privacy is commonly regarded in a positive light and is even enshrined as a fundamental right in various legal systems. However, there is a risk that corporations or governments could exploit our algorithms for nefarious purposes, such as unauthorized collection of personal data. Furthermore, the use of semi-privately trained models may result in decreased accuracy compared to non-private models, which can have adverse consequences. For instance, if a semi-DP model is utilized to forecast the effects of pollution, but yields less precise and overly optimistic outcomes, it could provide pretext for a government to unjustly dismantle environmental safeguards. Nonetheless, we firmly believe that the dissemination of privacy-preserving machine learning algorithms, coupled with enhanced understanding of these algorithms, ultimately offers a net benefit to society.

Another potential misuse of our work would be using the accuracy benefits of public data to argue for less stringent data privacy policies, laws, or regulations. However, we want to emphasize that even though public data can enhance the accuracy of models, we firmly believe that privacy laws and corporate policies should not be weakened. Differentially private synthetic data generation~\cite{torkzadehmahani2019dp, vietri2020new, boedihardjo2022covariance, he2023algorithmically} is one possible avenue for generating public data in an ethical, privacy-preserving manner. 

\section*{Acknowledgements}
The authors would like to thank Gautam Kamath, Adam Smith, Thomas Steinke, and Jonathan Ullman for very helpful pointers and explanations related to existing lower bound proof techniques. We thank Arnold Pereira for discussions and providing feedback at various stages of this project. We thank Michael Menart for helpful feedback on an earlier version of this manuscript. We thank the TPDP and ICML reviewers for their helpful feedback. This work was supported in part by a gift from Meta and the USC-Meta Center for Research and Education in AI \& Learning. AL's work was supported in part by NSF award DMS-2023239 and AFOSR award FA9550-21-1-0084.

\bibliography{icml}
\bibliographystyle{icml2024}

\newpage
\appendix
\onecolumn

\section*{Appendix}
\section{Related Work}
\label{app: intro}
Many works have considered a variety of semi-DP learning problems, empirically and theoretically~\cite{bassily2018model, feldman2018privacy, bassily2020privatequery, kairouz2021nearly, wang2020differentially, zhou2020bypassing, li2021large, liu2021leveraging, papernotsemi, papernot2018scalable, yu2021not, alon2019limits, bie2022private, ferrando2021combining, amid2022public}. Here we discuss the works that are most closely related to our own. 

\paragraph{Sample complexity bounds for semi-DP learning and estimation:} The works of~\citet{alon2019limits, bassily2020privatequery} give sample complexity bounds for semi-DP PAC learning and query release. Their upper bounds show that for hypothesis classes with finite VC-dimension, asymptotic improvements over the na\"ive approaches are possible. These results stand in contrast with our lower bounds for model training/optimization with Lipschitz loss functions. Both of these works also provide lower bounds. Below, we compare their results with our own lower bounds.

The negative result of Theorem 4.2 in~\cite{alon2019limits} is similar in spirit to our lower bounds: no semi-DP algorithm can achieve error better than the minimum of the optimal DP error and a term that depends on $\npu$. That being said, there are some significant and consequential differences between our lower bounds and \citep[Theorem 4.2]{alon2019limits}:
\begin{enumerate}
    \item \textit{Different learning problems}: \cite{alon2019limits} considers agnostic PAC learning/binary classification, whereas we consider model training (mean estimation and optimization). We are not aware of any way to obtain our lower bounds from the results of \cite{alon2019limits}. 
    \item \textit{Quantitative differences in the lower bounds}: The lower bound implied by \citep[Theorem 4.2]{alon2019limits} is of the form $\Omega\left(\min\left\{1/\npu, \text{optimal DP error}\right\} \right)$. Our lower bounds do not always take this form. For example, consider \cref{thm: DP convex ERM}: the first term in our lower bound is $\npr/n$, which can imply a much larger error than the bound in \cite{alon2019limits} (e.g., if $d = \eps n$ and $\npr = n/2$). This illustrates how different learning problems can benefit more from public data than others.
    \item \textit{Central vs. Local Semi-DP}: \cite{alon2019limits} only covers central semi-DP, whereas we cover both central and local semi-DP.
    \item \textit{Pure vs. Approximate Semi-DP}: \cite{alon2019limits}’s proof technique cannot handle approximate semi-DP because their Lemma 2.6 is limited to pure DP. By contrast, we give lower bounds for both pure and approximate semi-DP.
    \item \textit{New Techniques}: Our techniques differ substantially from those in \cite{alon2019limits}. For example, we develop a novel semi-DP Fano’s inequality and a novel semi-DP packing argument. For approximate semi-DP, we utilize fingerprinting proofs. Also, our semi-LDP lower bound techniques are completely different from \cite{alon2019limits}’s techniques. We hope that our novel techniques to find applications beyond those in our paper.
\end{enumerate}

The result of~\citet[Theorem 13]{bassily2020privatequery} showed (up to logarithmic factors) that no improvement over the na\"ive approaches is possible for \textit{approximate} $(1, \delta)$-semi-DP releasing decision stumps. This implies a lower bound for $(1, \delta)$-semi-DP mean estimation in the $\ell_\infty$ norm, but does not imply the tight lower bound for the $\ell_2$ setting that we provide in~\cref{thm: CDP mean estimation}. 
Moreover, \cite{bassily2018model}'s results and techniques do not lead to tight lower bounds for \textit{pure} $\eps$-semi-DP decision stumps or mean estimation. Thus, we develop novel techniques (e.g. semi-DP Fano and semi-DP packing arguments) for pure semi-DP estimation and model training.

For semi-DP $d$-dimensional Gaussian mean estimation, \cite{bie2022private} gave sample complexity bounds that do not depend on the range parameters of the distribution if $\npu \geq d + 1$; this is known to be impossible without public data. The concurrent and independent work of~\citet{ben2023private} established a lower bound for this problem. \cite{ben2023private} also explored the connection between semi-DP distribution learning of a class and the existence of a sample compression scheme for that class.  

\paragraph{DP model training (ERM and SCO) with public data:} The works of~\citet{kairouz2021nearly, zhou2020bypassing} considered DP ERM with public data and additional assumptions on the gradients lying in a certain low-dimensional subspace. Under these additional assumptions, \cite{kairouz2021nearly, zhou2020bypassing} show that nearly dimension-independent excess empirical risk bounds are possible (e.g. by using the public data to estimate the low-dimensional subspace and projecting noisy gradients onto this subspace). Our lower bounds show that these additional assumptions are strictly necessary: in general, polynomial dependence on the dimension is necessary for semi-DP ERM and SCO. 
The work of~\citet{wang2020differentially} used public data to adjust the parameters of DP-SGD. Empirically, pre-training on public data sets and privately fine-tuning the model~\cite{li2021large, kerrigan2020differentially, mehta2022large} has shown great promise for large-scale ML.

The work of~\citet{amid2022public} developed a public data-assisted DP mirror descent (PDA-MD) algorithm that sometimes outperforms DP-SGD empirically in training ML models, and theoretically in terms of excess risk for linear regression under certain distributional assumptions. 
We use the PDA-MD of~\citet{amid2022public} as a baseline in our experiments. \cite{amid2022public} also gave an ``efficient approximation'' of their PDA-MD in~\citep[Equation 1]{amid2022public}, which they used for training non-convex models. While finalizing this manuscript, we became aware that this ``efficient approximation'' is nearly equivalent to our~\cref{alg: weighted dp-sgd}. However, there are differences in the implementation of our algorithm. For example, we use a constant weight parameter $\alpha$, whereas \cite{amid2022public} uses decaying weights $\{\alpha_t\}_{t=1}^T$. Also, we clip both the private \textit{and public} gradients in our implementation of~\cref{alg: weighted dp-sgd}, which empirically improves performance: see~\cref{app: linreg setup} for further discussion. Our~\cref{alg: weighted dp-sgd} was derived in a different fashion from the algorithm in~\citep[Equation 1]{amid2022public}: we derived our algorithm as an application of our ``even more optimal'' mean estimator, while theirs was derived as an approximation of their mirror descent method. Moreover, no theoretical analysis was provided for the ``efficient approximation'' in~\cite{amid2022public}. 
In the linear regression setting, the PDA-MD algorithm can be viewed as Newton's algorithm where the Hessian is estimated via public data. Then the Hessian is inverted and multiplied by privatized gradients at each iteration. In other words, the update rule of the algorithm is given by $w^{t+1} = w^t  - \alpha_t (X_{pub}^T X_{pub})^{-1} (g_t + n_t)$ where $X_{pub}$ is the public data, $g_t$ is the sampled gradient from private data, and $n_t$ is the added noise. When the number of public data samples is small, the estimate of the Hessian becomes low rank (or inaccurate) and the (pseudo)inverse of it may introduce additional error (even after proper regularization). On the other hand, when the number of public data samples is large, but the number of private data is small, the PDA-MD algorithm can still suffer if it is not warm-started. This is because, although the Hessian $X^TX$ is estimated accurately in this case, but there is not enough private data to generate enough gradients to converge to optimal solution. This poor performance of PDA-MD when it is not warm started is also observed in our experiments.

The work of~\citet{nasr23a} used public data to train a generative model for data augmentation and to estimate the center of the clipping balls in DP-SGD. They did not provide any code for their experiments and thus we do not compare against them as a baseline in our experiments. However, combining our algorithm with the tricks used in~\cite{nasr23a} could be a promising avenue for future empirical work. 

\paragraph{Personalized DP:} A related line of work is that of \textit{personalized DP} (PDP)~\cite{jorgensen2015conservative, golatkar2022mixed, muhl2022personalized, alireza}, a generalization of DP in which each person may have different privacy parameters $(\eps_i, \delta_i)$. By letting $\eps_i = \infty$ for some person $i$, PDP also generalizes semi-DP. We leverage this connection to borrow techniques from the work of~\citet{alireza}, which considers pure ($\delta_i = 0$) PDP estimation in one dimension ($d=1$). 
We also note that the 1-dimensional pure (central) PDP mean estimation bound of~\citet{alireza} extends easily to a 1-dimensional $\epsemi$ bound. However, our $d$-dimensional semi-DP lower bounds require a different set of techniques.
Additionally, the personalized LDP bound of~\citet{alireza} relies on the assumption $\eps_i \leq 1$ and does not seem to extend to the $\eps$-semi-LDP setting. 

 \paragraph{Concurrent and Subsequent Work:} The work of~\citet{ullah2024public} first appeared on arXiv a few weeks after v2 of our paper appeared.\footnote{The first version of our paper appeared on arXiv more than eight months before~\cite{ullah2024public}. However, v1 of our paper did not contain our tight high-dimensional approximate semi-DP population mean estimation and SCO lower bounds.} \cite{ullah2024public} proves results that are very similar to~\cref{thm: CDP mean estimation,thm: approx DP SCO}. However, their lower bound proofs do not require the (mild) symmetry assumption that our proof requires: see~Remark~\ref{rem: symmetric assumption}. Moreover, in a restricted parameter regime, \citet{ullah2024public} also provide lower bounds that are tighter by a $\log(1/\delta)$ factor. \citet{ullah2024public} complement these lower bounds by giving
 novel algorithms for leveraging unlabeled public data in training private generalized linear models (GLM) with dimension-independent rates. 

The work of~\citet{liu2023coupling}, which appeared on arXiv 3 months after v1 of our paper, couples private and public gradients in non-convex optimization via a similar weighting scheme to our own. They show benefits of their algorithm over standard DP approaches both theoretically and empirically. 

\citet{tang2023differentially}, which appeared on arXiv 18 days before v1 of our paper,  explores how to improve the privacy-utility tradeoff of DP-SGD by learning priors from images generated by random processes and transferring these priors to private data.  

\paragraph{Other Related Works:} The work of~\citet{gu2023choosing} gave an algorithm for selecting an appropriate public dataset that can be used to enhance private optimization by projecting gradients onto a subspace prescribed by the this public datasaet.

\citet{ganesh2023public} provided an explanation for the empirically reported benefits of pre-training on public data, arguing that non-convex optimization algorithms must go through two phases: (i) selecting a good ``basin'' in the loss landscape; (ii) solving an easy optimization problem within that basin. They hypothesize that public pre-training can be helpful in selecting a good basin. They also demonstrated a separation between pretrained and non-pretrained models by constructing a non-convex optimization problem for which public pretraining is necessary to achieve non-trivial error.

 \section{Zero-Concentrated DP vs. Pure and Approximate DP}
 \label{app: approx DP weaker than zcdp}

 \begin{proposition}\cite{bun16}
\label{prop:bun1.3}
If $\Al$ is $\rho$-zCDP, then $\Al$ is  $(\rho + 2\sqrt{\rho \log(1/\delta)}, \delta)$-DP $~\forall \delta > 0$. Moreover, if $\alg$ is $\eps$-DP, then $\alg$ is $\eps^2/2$-zCDP. 
\end{proposition}

\section{Summary table of pure semi-DP and semi-LDP results}
\begin{figure*}[h!]
    \begin{center}
        \includegraphics[width = 0.99\textwidth]{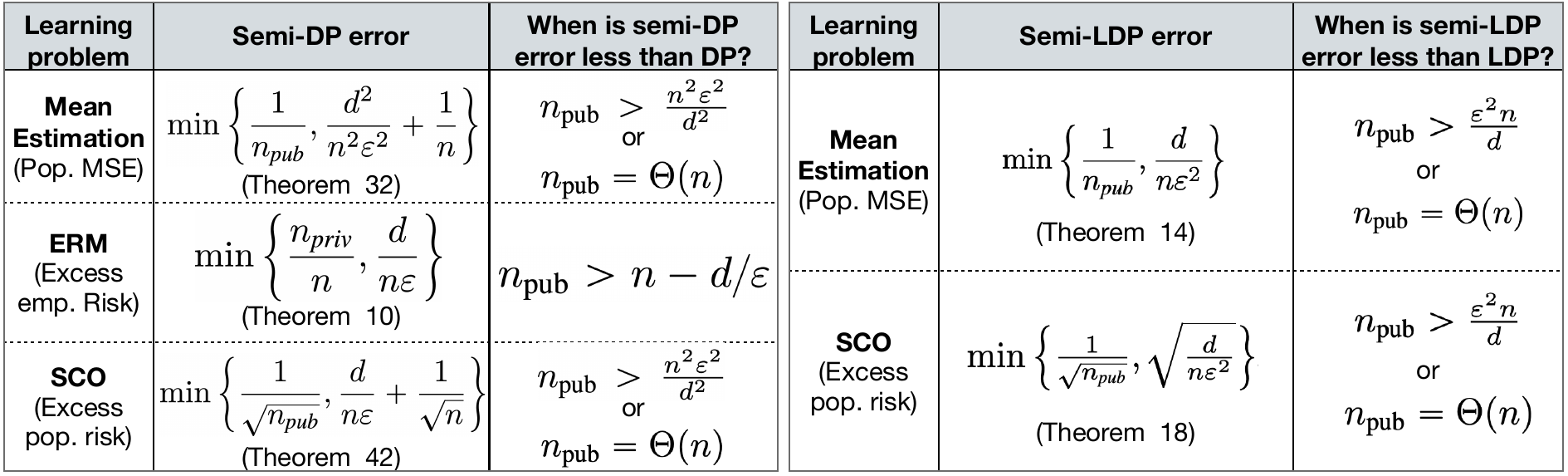}
         \end{center}
      \vspace{-0.15in}
     \caption{\footnotesize 
Minimax optimal error rates for central $\eps$-semi-DP and (local) $\eps$-semi-LDP SCO and mean estimation results. 
Dependence on range and 
Lipschitz parameters, and constraint set diameter omitted. 
Mean estimation and SCO lower bounds are only tight if $\npu = O(n \eps/d)$ or $d = O(1)$. 
Strongly convex ERM and SCO results are included later in this Appendix, but excluded from this table. 
}\label{app table: summary of pure rates}
\vspace{-.1in}
\end{figure*}

\section{Notation}
We recall notation from the main body and include some additional basic definitions below for convenience. 

Let $\| \cdot \|$ be the $\ell_2$ norm. $\WW$ denotes a convex, compact subset of $\mathbb{R}^d$ with $\ell_2$ diameter~$D$. $\XX$ denotes a data universe. Function $g: \WW \to \mathbb{R}$ is \textit{$\mu$-strongly convex} if $g(\alpha w + (1- \alpha) w') \leq \alpha g(w) + (1 - \alpha) g(w') - \frac{\alpha (1-\alpha) \mu}{2}\|w - w'\|^2$ for all $\alpha \in [0,1]$ and all $w, w' \in \WW$. If $\mu = 0,$ we say $g$ is \textit{convex}. For convex $f(\cdot, x)$, denote any \textit{subgradient} of $f(w,x)$ w.r.t. $w$ by $\nabla f(w,x) \in \partial_w f(w,x)$: i.e. $f(w', x) \geq f(w, x) + \langle \nabla f(w,x), w' - w \rangle$ for all $w' \in \WW$. Function $f: \WW \times \XX \to \mathbb{R}$ is \textit{uniformly $L$-Lipschitz} in $w$ if $\sup_{x \in \XX} |f(w, x) - f(w',x)| \leq L\|w - w'\|$. Let $\mathbb{B} = \{x \in \mathbb{R}^d | \|x\| \leq 1\}$ denote the unit $\ell_2$-ball. For functions $a = a(\theta)$ and $b = b(\phi)$ of input parameter vectors $\theta$ and $\phi$, we write $a \lesssim b$ or $a = O(b)$ if there is an absolute constant $C > 0$ such that $a \leq C b$ for all values of input parameter vectors $\theta$ and $\phi$. 
 
\section{Optimal Centrally Private Model Training with Public Data}
\label{app: CDP}
\subsection{Optimal Semi-DP Mean Estimation} 
We begin in~\cref{app: central dp empirical ME} with empirical mean estimation. This subsection was omitted from the main body due to space constraints. \cref{thm: pure DP empirical mean estimation} will be useful for proving our ERM bounds (\cref{thm: DP convex ERM}). Then, in~\cref{app: CDP pop mean est}, we turn to population mean estimation (i.e. the proof of~\cref{thm: CDP mean estimation}). \cref{app: even more optimal ME} contains proofs for~\cref{sec: even more optimal CDP ME}.

\subsubsection{Estimating the Empirical Mean}
\label{app: central dp empirical ME}
For a given data set $X \subset \mathbb{B} := \{x \in \mathbb{R}^d : \|x\| \leq 1\}$, consider the problem of estimating $\Bar{X} = \frac{1}{n}\sum_{i=1}^n x_i$ subject to the constraint that the estimator satisfies semi-DP. We will characterize the minimax squared error of $d$-dimensional empirical mean estimation under $\varepsilon$-semi-DP: 
\begin{equation}
\mathcal{M}_{\text{emp}}(\varepsilon, \npr, n, d) := \inf_{\mathcal{A} \in \mathbb{A}_\varepsilon(\mathbb{B})} \sup_{X \in \mathbb{B}^n, |X_{\text{priv}}| = \npr} \expec_{\mathcal{A}}\left[\| \mathcal{A}(X) - \Bar{X}\|^2\right],
\end{equation}
where $\mathbb{A}_\varepsilon(\mathbb{B})$ denotes the set of all $\varepsilon$-semi-DP estimators $\alg: \mathbb{B}^n \to \mathbb{B}$.

\begin{theorem}
\label{thm: pure DP empirical mean estimation}
Let $\varepsilon > 0$, $n, d \in \mathbb{N}$, $\npr \in [n]$. There exist absolute constants $c$ and $C$ with
$0 < c \leq C$ such that
\[
c\min\left\{ \frac{\npr}{n}, \pen \right\}^2 \leq \emprisk \leq C \min\left\{ \frac{\npr}{n}, \pen \right\}^2.
\]
\end{theorem}
\begin{proof}
\textbf{Lower bound:}
We use a packing argument to prove our lower bound. 
Denote $\XX = \left[\frac{-1}{\sqrt{d}}, \frac{1}{\sqrt{d}}\right]^d$. Let $\alg: \mathbb{B}^n \to \mathbb{B}$ be $\epsemi$. Choose $K = 2^{d/2}$ private data points $\{x_i\}_{i=1}^K \subset \left\{\pm \frac{1}{\sqrt{d}}\right\}^d$ such that $\|x_i - x_j\| \geq 1/8$ for all $i \neq j$. The existence of such a set of points is well-known (see e.g. the Gilbert-Varshamov construction). Let $n^* = \frac{nd}{3 \varepsilon \npr}$. 

\noindent \ul{Case 1:} $n \leq n^*$ (i.e. $d \geq 3 \varepsilon \npr$). 
In this case, we'll show that $\expec\|\alg(X) - \Bar{X}\|^2 \gtrsim \left(\frac{\npr}{n}\right)^2$ for some $X \in \XX^n$. For $i \in [K]$, let $X_i = (x_i, \cdots, x_i, \mathbf{0}_{n - \npr})$  consist of $\npr$ copies of $x_i$ followed by $n - \npr = \npu$ copies of $0 \in \mathbb{R}^d$. Suppose for the sake of contradiction that for every $i \in [K]$, with probability $ \geq 1/3$ we have $\|\alg(X) - \Bar{X}\| < \frac{1}{32} \frac{\npr}{n}$. That is, we are supposing $\pr(\alg(X_i) \in \mathcal{B}_i) \geq 1/3$ for all $i$, where $\mathcal{B}_i = \{x \in \mathbb{B} : \|x - \Bar{X}_i\| \leq \frac{1}{32} \frac{\npr}{n}\}$. Note that the sets $\{\mathcal{B}_i\}_{i=1}^K$ are disjoint by construction. Since $\alg$ is $\epsemi$, group privacy implies that $\mathbb{P}(\alg(X_1) \in \mathcal{B}_i) \geq \frac{1}{3} e^{-\eps \npr}$ for all $i \in [K]$. Thus, \begin{align*}
    K \frac{1}{3} e^{-\eps \npr} \leq \sum_{i=1}^K \pr(\alg(X_1) \in \BB_i) \leq 1,
\end{align*}
where the last inequality follows from disjointness of the balls $\BB_i$. Thus, we obtain $\ln(K/3) \leq \eps \npr$. 
Assume for now that $d \geq 8$. (A $1$-dimensional lower bound  that is tight up to constant factors can be shown easily by following the proof of the $d$-dimensional case but choosing $K = 16$ instead of $K = 2^{d/2}$.) Then $d/2 \leq d - 4 \leq \eps \npr$ implies $d \leq 2 \eps \npr$, contradicting the assumption made in \ul{Case 1}. Thus, we conclude that there exists a data set $X_i$ such that with probability at least $2/3$, $\|\alg(X_i) - \Bar{X}_i\| > \frac{1}{32} \frac{\npr}{n}$. Squaring both sides of this inequality and applying Markov's inequality yields the desired lower bound. 

\noindent \ul{Case 2:} $n > n^*$. 

Additionally, suppose for now that $n^* \leq \npr$. In this case, we'll show that $\expec\|\alg(X) - \Bar{X}\|^2 \gtrsim \left(\frac{d}{n\eps}\right)^2$ for some $X \in \XX^n$. Let $\widetilde{X}_i = (x_i, \cdots, x_i, \mathbf{0}_{n - n^*})$ consist of $n^*$ copies of $x_i$ followed by $n - n^*$ copies of $0 \in \mathbb{R}^d$.\footnote{We assume without loss of generality that $n^* \in \mathbb{N}$. If $n^*$ is not an integer, then choosing $\lceil n^* \rceil$ instead yields the same bound up to constant factors.} Denoting the mean of a dataset $X$ by $q(X) := \frac{1}{n} \sum_{x \in X} x$ for convenience, we see that $q(\wt{X}_i) = \frac{n^*}{n} x_i = \frac{d}{3\npr \eps} x_i$. Define the algorithm $\hat{\alg}: \mathbb{B}^{n^*} \to \mathbb{B}$ by $\hat{\alg}(X) = \frac{n}{n^*} \alg(X, \mathbf{0}_{n - n^*})$. Since $\alg$ is $\epsemi$, we see that $\hat{\alg}$ is $\epsemi$ by post-processing. Also, the domain of $\hat{\alg}$ is $\XX^{n^{*}}$ and $n^* \leq n^*$, so the argument in \ul{Case 1} applies to $\hat{\alg}$. Thus, by applying the result in \ul{Case 1}, there exists $i \in [K]$ such that with probability at least $2/3$, $\|\hat{\alg}(x_i, \cdots, x_i) - q(x_i, \cdots, x_i)\| \geq \frac{1}{32} \frac{\npr}{n}$. (Here $(x_i, \cdots, x_i) \in \XX^{n*}$.) But this implies $\|\alg(\widetilde{X}_i) - q(\wt{X}_i)\| \geq \frac{1}{32} \frac{\npr}{n} \frac{n^*}{n} = \frac{1}{96} \frac{d}{n \eps}$ with probability at least $2/3$. Again, squaring both sides and applying Markov yields the desired lower bound. 

Next, consider the complementary subcase where $n^* > \npr$. Define the algorithm $\hat{\alg}: \mathbb{B}^{n^*} \to \mathbb{B}$ by $\hat{\alg}(X) = \alg(X, \mathbf{0}_{n - n^*})$. Since $\alg$ is $\epsemi$, we see that $\hat{\alg}$ is $\epsemi$ by post-processing. Also, the domain of $\hat{\alg}$ is $\XX^{n^{*}}$ and $n^* \leq n^*$, so the argument in \ul{Case 1} applies to $\hat{\alg}$. Thus, by applying the result in \ul{Case 1}, there exists $i \in [K]$ such that with probability at least $2/3$, $\|\hat{\alg}(x_i, \cdots, x_i) - q(x_i, \cdots, x_i)\| \geq \frac{1}{32} \frac{\npr}{n}$. (Here $(x_i, \cdots, x_i) \in \XX^{n*}$.) But this implies that $\|\alg(X_i) - \Bar{X}_i\| \geq \frac{1}{32} \frac{\npr}{n}$ with probability at least $2/3$. Again, squaring both sides and applying Markov yields the desired lower bound. Thus, the lower bound holds in all cases.

\noindent \textbf{Upper bound:} For the first term in the minimum, consider the algorithm which throws away the private data and outputs $\alg(X) = \frac{1}{n}\sum_{x \in X_{\text{pub}}} x$. Clearly, $\Al$ is semi-DP. Moreover, \[
\| \alg(X) - \Bar{X}\|^2 = \frac{1}{n^2}\left\|\sum_{x \in \xpr} x\right\|^2 \leq \frac{\npr}{n^2} \sum_{x \in \xpr} \|x\|^2 \leq \left(\frac{\npr}{n}\right)^2. 
\]
For the second term, consider the Laplace mechanism $\alg(X) = \Bar{X} + (L_1, \cdots, L_d)$, where $L_i \sim \lap(2\sqrt{d}/n\varepsilon)$ are i.i.d. mean-zero Laplace random variables. We know $\alg$ is $\varepsilon$-DP by~\cite{dwork2014}, since the $\ell_1$-sensitivity is $\sup_{X \sim X'}\| \Bar{X} - \Bar{X}'\|_1 = \frac{1}{n}\sup_{x, x'} \|x - x'\|_1 \leq \frac{2 \sqrt{d}}{n}$. Hence $\alg$ is $\varepsilon$-semi-DP. Moreover, $\alg$ has error \[
\expec\|\alg(X) - \Bar{X}\|^2 = d ~\var(\lap(2\sqrt{d}/n\varepsilon)) = \frac{8d^2}{n^2 \varepsilon^2}.
\]
Combining the two upper bounds completes the proof. 
\end{proof}

\subsubsection{Estimating the Population Mean}
\paragraph{Approximate $(\eps, \delta)$-Semi-DP Mean Estimation}
\label{app: CDP pop mean est}
\begin{theorem}[Formal statement of~\cref{thm: CDP mean estimation}]
\label{thm: CDP mean estimation app}
Let $\eps \lesssim 1/\log(nd)$, $\delta \ll 1/\npr$. Then, there is an absolute constant $C> 0$ such that
\begin{equation}
\label{eq: delta > 0 app}
\ell(d, n) \min\left\{\frac{1}{\npu}, \frac{d}{n^2 \eps^2} + \frac{1}{n}\right\} \leq \popDPrisk \leq C \min\left\{\frac{1}{\npu}, \frac{d \ln(1/\delta)}{n^2 \eps^2} + \frac{1}{n}\right\},
\end{equation}
where $1/\ell(d,n)$ is logarithmic in $d$ and $n$. The lower bound holds for symmetric algorithms $\alg = (\alg^1, \ldots, \alg^d)$ such that $\alg^j = \alg^l$ for all $j, l \in [d]$. 
\end{theorem}

For the lower bound, we will first prove a stronger result in~\cref{thm: approx DP mean estimation with small theta}, in which we construct a hard distribution whose mean is small---scaling with the accuracy lower bound that we aim to prove. This ``small mean'' property will be needed for our semi-DP SCO lower bound (\cref{thm: approx DP SCO}), even though it is not necessary for the proof of~\cref{thm: CDP mean estimation app}. 
\begin{theorem}
    \label{thm: approx DP mean estimation with small theta}
Let $\XX = \left\{\pm \frac{1}{\sqrt{d}}\right\}^d$, $\eps \lesssim 1/\log(nd)$, and $\delta \ll 1/\npr$. Then, for any symmetric $(\eps, \delta)$-semi-DP $\alg$, there exists a product distribution $P$ on $\XX$ with $\|\expec_{x \sim P}[x]\| \leq \min\left(\frac{1}{\sqrt{\npu}}, \frac{\sqrt{d}}{\eps \npr}\right)$ such that \[
\expec\left[\|\alg(X) - \expec_{x \sim P}[x]\|^2\right] = \widetilde{\Omega}\left(\min\left\{\frac{1}{\npu}, \frac{d}{\eps^2 \npr^2} + \frac{1}{n}\right\} \right). 
\]
\end{theorem}

For any symmetric $\alg$ and any distribution $P$ with $X \sim P^n$, we have \begin{align}
\label{eq: bias}
    \expec \|\alg(X) - \expec_{x \sim P}[x]\|^2 &= \expec \|\alg(X) - \expec \alg(X)\|^2 + \|\expec \alg(X) - \expec_{x \sim P}[x]\|^2 \nonumber \\
    &\geq \|\expec \alg(X) - \expec_{x \sim P}[x]\|^2 \nonumber \\
    & = d |\expec \alg^1(X) - \expec[x^1]|,
\end{align}
where the last equality used assumption that $\alg$ is symmetric. For $a \in \mathbb{R}$, scalar random variable $x^j \sim P_a$ with mean $\expec[x^j] = a$ and $X^j \sim P_a^n$ denote 
\[
\text{Bias}_a(\alg) := |\expec \alg^j(X^j) - a|.
\]
\begin{definition}[Low bias algorithms]
We say symmetric $\alg$ is \textit{low bias} if for every $a \in \mathbb{R}$, \[
\text{Bias}_a(\alg)^2 \leq \frac{1}{d}\min\left(\frac{1}{\npu}, \frac{1}{n} + \frac{d}{n^2 \eps^2}\right).
\] 
\end{definition}

We can assume without loss of generality that $\alg$ is low bias when we are proving the lower bound in~\cref{thm: CDP mean estimation app}: if $\alg$ is not low bias, then inequality~\eqref{eq: bias} implies that the worst-case MSE of $\alg$ is lower bounded as in \eqref{eq: delta > 0 app}.

To prove \cref{thm: approx DP mean estimation with small theta} for low bias and symmetric $\alg$, we will use~\cref{thm: DSS15 l2 version}.\footnote{For convenience, we work with data drawn from $\{\pm 1\}^d$ and then re-scale to obtain the final mean estimation lower bound.} This result shows that any sufficiently accurate $\alg$ is vulnerable to an attack that traces many individuals in the data set with high probability. 

\begin{theorem}
\label{thm: DSS15 l2 version}
Let $a \in (0, 1]$. Consider the product distribution on $\pmone$ defined in the following way: for $j \in [d]$, independently draw $\theta^j \sim \textbf{Unif}\left([-a, a]\right)$ and $x_i^j \sim P_\theta$ such that $x_i^j \in \pmones$ with mean $\expec_{x_i^j \sim P_{\theta^j}}[x_i^j] = \theta^j$ for $i \in [n]$. Denote $X = (x_1, \ldots, x_n) \sim P_\theta^n$ where $\theta = (\theta^1, \ldots, \theta^d)$ and $\ptheta = \Pi_{j=1}^d P_{\theta^j}$. Let $\alg: \pmone \to [-a, a]^d$ satisfy $\expec_{X \sim \ptheta^n}[\alg(X)] = \theta$ and $\expec\|\alg(X) - \theta\|^2 \leq \alpha^2$ for $\sqrt{\frac{d}{n}} \leq \alpha \leq \frac{a \sqrt{d}}{100}$. Assume $d > 400 \alpha n \sqrt{\ln(2/\delta)}$. Moreover, assume that $\alg(X) = (\alg^1(X), \ldots, \alg^d(X))$ with $\alg^j = \alg^l$ for all $j, l \in [d]$. Then, the attack $\mathcal{I}: \pmone \times [-a, a]^d \to \{IN, OUT\}$ described in~\cref{alg: attack} satisfies the following properties: a) if $y \sim \ptheta$ independently of $X$, then $P(\attack(y, \widetilde{\alg}(X)) = IN) \leq \delta$; and b) $P(|\{i \in [n]: \attack(x_i, \widetilde{\alg}(X)) = IN \}| \geq \frac{d}{10^6 \alpha^2}) \geq 1 - \delta$,
where 
$\widetilde{\alg}$ 
is $(\widetilde{O}(\eps), \widetilde{O}(\zeta))$-semi-DP if and only if $\alg$ is $(\eps, \zeta)$-semi-DP, and $\expec_{X \sim \ptheta^n}[\widetilde{\alg}(X)] = \theta$.
\end{theorem}

The proof of~\cref{thm: DSS15 l2 version} uses a convenient reduction in~Lemma~\ref{lem: claim}: for any low bias and symmetric $\alg$ that has small expected mean squared error (in $\ell_2$-norm), there exists another low bias mechanism $\widetilde{\alg}$ that has even smaller $\ell_\infty$ accuracy with high probability. Moreover, $\alg$ is $(\eps, \delta)$-semi-DP if and only if $\widetilde{\alg}$ is $(\widetilde{O}(\eps), \widetilde{O}(\delta))$-semi-DP. This lemma allows us to: construct a new product distribution whose mean is small, modify the attack of~\citet{dwork2015robust}, and derive a generalized fingerprinting lemma (Lemma~\ref{lem: fingerprint}) in order to prove~\cref{thm: DSS15 l2 version}. 

With~\cref{thm: DSS15 l2 version} in hand, we leverage the proof technique of~\citet{bassily2020privatequery} to show that any sufficiently accurate $\widetilde{\alg}$ must leak the data of more than $\npu$ individuals and thus cannot be $(\widetilde{O}(\eps), \widetilde{O}(\delta))$-semi-DP. But by Lemma~\ref{lem: claim}, this means that $\alg$ cannot be $(\eps, \delta)$-semi-DP. 
Below, we discuss the attack that we use and then fill in the details of the proof.

The attack in~\cref{alg: attack} is
a modification of the robust tracing attack of~\citet{dwork2015robust}. The attacker in~\cref{alg: attack} receives as input the output $q \sim \alg(X)$ of a mechanism, a target point $y$, and the mean of the data distribution $P$, $\theta = \expec_{x \sim P}[x]$. (An estimate of the true mean or access to sufficiently many independent draws from $P$ would also suffice in lieu of $\theta$.) The target $y$ is either a data point used by $\alg(X)$ (i.e. $y \in X$) or an independent draw from the distribution $P$ that $X$ was drawn from. The attacker aims to infer whether or not $y$ was in $X$. If the attacker outputs IN when y is in $X$ and OUT when $y$ is not in $X$ (i.e. the attack succeeds) with high probability, then the algorithm $\alg$ is not private. The truncation parameter is $\eta = 2 \alpha/\sqrt{d}$, where $\alpha$ is the expected $\ell_2$-error of the mechanism $\alg$. The parameter $\delta$ can be chosen by the attacker. 

\begin{algorithm}
\caption{Tracing Attack Against $\ell_2$-Accurate Mechanisms}
\label{alg: attack}
\begin{algorithmic}[1]
\STATE {\bfseries Input:} Target $y \in \{\pm 1\}^d$, mean $\theta \in [-a,a]^d$, output of mechanism $q \sim \alg(X)$. 
 \STATE Let $\eta := 2\alpha/\sqrt{d}$.
  \STATE Let $\trunc \in [-\eta, \eta]^d$ denote the entrywise truncation of $q - \theta$. 
  \STATE Compute $\langle y - \theta, \trunc \rangle = \sum_{j=1}^d (y^j - \theta^j)\lfloor q - \theta \rceil_{\eta}^j$.
  \IF{$\langle y - \theta, \trunc \rangle > \tau := 2\eta \sqrt{d \ln(1/\delta)}$}
		    \STATE {\bfseries Output:} IN.
		\ELSE
		     \STATE {\bfseries Output:} OUT.
		\ENDIF
\end{algorithmic}
\end{algorithm}

Compared to~\cite{dwork2015robust}, we choose a smaller truncation parameter to handle the $\ell_2$ case ($\eta = 2\alpha/\sqrt{d}$ instead of $2\alpha$). We also scale the ``IN/OUT'' decision threshold $\tau$ accordingly. Moreover, since we are not concerned with the size of the reference sample, we assume that the attacker knows the mean of the data distribution. Thus, we use $\theta$ in place of the average of reference samples to give a simpler attack than the one in~\cite{dwork2015robust}. This is not necessary for our attack to work: the attacker just needs a sufficiently close approximation to the true mean (which can be attained with access to enough reference samples) for our proof to go through.

\cref{thm: DSS15 l2 version} builds on~\citep[Theorem 17]{dwork2015robust}, which gave a similar result for $\ell_\infty$-accurate mechanisms and $\theta$ drawn from a so-called \textit{strong distribution} (see \citep[Definition 5]{dwork2015robust}). By contrast, \cref{thm: DSS15 l2 version} only requires a weaker $\ell_2$-accuracy guarantee and uses a distribution which is not ``strong'' in the sense required by~\citep[Theorem 17]{dwork2015robust}. Our generalization of the \textit{fingerprinting lemma}~\cite{bsu17}, given below, allows us to handle such a distribution, and is the first step towards proving~\cref{thm: DSS15 l2 version}: 

\begin{lemma}[Generalized Fingerprinting Lemma]
\label{lem: fingerprint}
Let $a \in (0, 1]$. Draw $\theta \sim \textbf{Unif}\left([-a, a]\right)$ and $x_i \sim P_\theta$ be such that $x_i \in \pmones$ with mean $\expec_{x_i \sim P_{\theta}}[x_i] = \theta$ for $i = 1, \ldots, n$. Denote $X = (x_1, \ldots, x_n) \sim P_\theta^n$. Then, for any
$f:\pmones^n \to \mathbb{R}$ that does not depend directly on $\theta$, we have 
\begin{align*}
\expec\left[ (f(X) - \theta) \sum_{i=1}^n (x_i - \theta) \right] &\geq \frac{a^2}{3} - \expec \left[(f(X) - \theta)^2 \right] \\
&\;\;\; + \frac{1 - a^2}{2a}\left[\expec_{X \sim P_a^n}[f(X)] - \expec_{X \sim P_{-a}^n}[f(X)]\right],
\end{align*}
where all expectations are over the random draw of $\theta$ and $X \sim P_\theta^n$ unless otherwise indicated. 
\end{lemma}
\begin{proof}
Define $g: [-a, a] \to \mathbb{R}$ by $g(\theta) := \expec_{X \sim P_\theta^n}[f(X)]$. Our first claim is \begin{equation}
\label{eq: A.1}
    \expec_{X \sim \ptheta^n}[(f(X) - \theta) \sum_{i=1}^n (x_i - \theta)] = g'(\theta) (1 - \theta^2). 
\end{equation}
To prove~\cref{eq: A.1}, write \[
g(\theta) = \sum_{X \in \pmones^n} f(X) \prod_{i=1}^n \frac{1 + (\theta/x_i)}{2}
\]
and \begin{align*}
    g'(\theta) &= \sum_{X \in \pmones^n} f(X) \frac{d}{d\theta} \left[\prod_{i=1}^n \left(\frac{1 + (\theta/x_i)}{2} \right) \right] \\
    &= \sum_{X \in \pmones^n} f(X) \sum_{i=1}^n \frac{d}{d\theta} \left( \frac{1 + (\theta/x_i)}{2}\right) \prod_{k \in [n] \setminus \{i\}} \left( \frac{1 + (\theta/x_k)}{2}\right) \\
    &=  \sum_{X \in \pmones^n} f(X) \sum_{i=1}^n \frac{1}{x_i + \theta} \prod_{i \in [n]} \left( \frac{1 + (\theta/x_i)}{2}\right) \\
    &= \expec_{X \sim \ptheta^n}\left[ f(X) \sum_{i=1}^n  \frac{1}{x_i + \theta} \right] \\
    &= \expec_{X \sim \ptheta^n}\left[ f(X) \sum_{i=1}^n  \frac{x_i - \theta}{1 - \theta^2} \right]. 
\end{align*}
Since $\expec_{X \sim \ptheta^n}\left[\theta \sum_{i=1}^n (x_i - \theta) \right] = 0$, \cref{eq: A.1} is proved. 

Next, we claim: for any differentiable function $g:[-a, a] \to \mathbb{R}$, we have \begin{equation}
    \label{eq: A.2}
    \expec_{\theta \sim \textbf{Unif}[-a, a]}[g'(\theta) (1 - \theta^2)] = 2 \expec_{\theta \sim \textbf{Unif}[-a, a]}[g(\theta) \theta] + \frac{1 - a^2}{2a}[g(a) - g(-a)].
\end{equation}
To prove~\cref{eq: A.2}, let $u(\theta) = 1 - \theta^2$ and use the product rule and fundamental theorem of calculus to write \begin{align*}
\expec_{\theta \sim \textbf{Unif}[-a, a]}[g'(\theta) (1 - \theta^2)] &= \frac{1}{2a} \int_{-a}^a g'(\theta) u(\theta) d\theta \\
&= \frac{1}{2a} \int_{-a}^a  \left[\frac{d}{d\theta} \left( g(\theta) u(\theta) \right) - g(\theta) u'(\theta) \right] d\theta \\
&= \frac{1}{2a} \left[ g(a) u(a) - g(-a) u(-a) \right] - \frac{1}{2a} \int_{-a}^a g(\theta) \cdot (-2\theta) d\theta \\
&= \frac{1 - a^2}{2a}[g(a) - g(-a)] + 2\expec_{\theta \sim \textbf{Unif}[-a, a]}[g(\theta) \theta].
\end{align*}

Now we will apply~\cref{eq: A.1} and~\cref{eq: A.2} with the differentiable function $g: [-a, a] \to \mathbb{R}$, $g(\theta) := \expec_{X \sim P_\theta^n}[f(X)]$ to obtain \begin{align*}
\expec\left[(f(X) - \theta) \sum_{i=1}^n (x_i - \theta) \right] &= \expec_{\theta \sim \textbf{Unif}[-a, a]}[g'(\theta) (1 - \theta^2)] \\
&= 2 \expec_{\theta \sim \textbf{Unif}[-a, a]}[\theta g(\theta)] + \frac{1 - a^2}{2a}[g(a) - g(-a)].
\end{align*}
Moreover, \begin{align*}
\expec[(f(X) - \theta)^2] &\geq -2\expec[g(\theta) \theta] + \expec[\theta^2] \\
&= -2\expec[g(\theta) \theta] + \frac{a^2}{3}. 
\end{align*}
Combining the above pieces completes the proof.
\end{proof}

We state a lemma that will allow us to conveniently assume without loss of generality that the given mechanism $\alg$ is $\ell_\infty$-accurate: 
\begin{lemma}
\label{lem: claim}
Consider $\theta \sim \textbf{Unif}([-a,a]^d)$ and $X \sim \ptheta^n$ be distributed as described above. 
Suppose $\expec\|\alg(X) - \theta\|^2 \leq \alpha^2$, where $\alg: \pmone \to [-a,a]^d$, $\alg(X) = (\alg^1(X) \ldots, \alg^d(X))$. Assume that $\alg(X)$ is a low bias estimator of $\theta$ and that $\alg^j = \alg^l$ for all $j, l$. Then, for any $0 < \beta \leq \min\left(\alpha^2/4d, 1/(n^2 d^2)\right)$, there exists a low bias $\widetilde{\alg}$ such that $P(|\widetilde{\alg}^j(X) - \theta^j| > \frac{2 \alpha}{\sqrt{d}}) \leq \beta$, $\widetilde{\alg}^j = \widetilde{\alg}^l$ for all $j, l \in [d]$, and $\expec\|\widetilde{\alg}(X) - \theta\|^2 \leq 5\alpha^2$. Also, $\alg$ is $(\eps, \delta)$-semi-DP if and only if $\widetilde{\alg}$ is $(\widetilde{O}(\eps), \widetilde{O}(\delta))$-semi-DP.
\end{lemma}
\begin{proof}
By the assumption that $\alg^j = \alg^l$ for all $j,l$ and the fact that $X$ is drawn from a product distribution, we have \[
\expec\|\alg(X) - \theta\|^2 = d \expec|\alg^1(X) - \theta^1|^2.
\]
Thus, $\expec\|M^j(X) - \theta^j|^2 \leq \frac{\alpha^2}{d}$ for all $j$.
(We are also assuming without loss of generality here that $\alg^j$ depends only on $X^j$: if this is not the case, then it's easy to see by the choice of distribution that there exists another algorithm with smaller error than $\alg$ satisfying this assumption.) 
Also, \begin{align}
    P\left(|\alg^j(X) - \theta^j| > \frac{2\alpha}{\sqrt{d}}\right) &\leq \frac{\expec|\alg^j(X) - \theta^j|^2}{(2\alpha/\sqrt{d})^2} \leq \frac{1}{4},
\end{align}
by Chebyshev's inequality. 
To construct $\widetilde{\alg}^j$ we will use the well-known ``median trick'': run each $\alg^j$ for $m = O(\log(1/\beta))$ times, each time using a batch $X_i$ of $n/m$ independent samples in $X$. Then take the median of the $m = O(\log(1/\beta))$ outputs: $\widetilde{\alg}^j(X) = \text{median}(\alg_1^j(X_1), \ldots, \alg_m^j(X_m))$. Then, $P(|\widetilde{\alg}^j(X) - \theta^j| > \frac{2\alpha}{\sqrt{d}}) \leq \beta$ by a Chernoff/Hoeffding bound. Applying the law of total expectation with $\beta \leq \min\left(\alpha^2/4d, 1/(n^2 d^2)\right)$ establishes the expected $\ell_2$-accuracy claim for $\widetilde{\alg}$. Moreover, 
$\widetilde{\alg}$ is low bias since for all $j \in [d]$, we have $\widetilde{\alg}^j(X) \in \alg_i^j(X_i)$ for some $i \in [m]$ with probability $1$ and $\alg$ is low bias. 

We now prove the privacy claim. First note that $\alg$ is $(\eps, \delta)$-semi-DP if and only if $\alg^j$ is $(\widetilde{O}(\eps/\sqrt{d}), O(\delta/d))$-semi-DP by the advanced composition theorem~\cite{dwork2014}, since $\alg^j = \alg^l$ for all $j, l \in [d]$. Also, $\alg^j$ is $(\eps', \delta')$-semi-DP if and only if $\widetilde{\alg}^j$ is $(\widetilde{O}(\eps'), \widetilde{O}(\delta'))$-semi-DP by parallel composition of (semi) DP~\cite{mcsherry2009privacy} and the fact that $m = \widetilde{O}(1)$. Since $\widetilde{\alg}^j = \widetilde{\alg}^l$ for all $j, l$ by construction, we can conclude that $\alg$ is $(\eps, \delta)$-semi-DP if and only if $\widetilde{\alg}$ is $(\widetilde{O}(\eps), O(\delta))$-semi-DP. 
\end{proof}

By Lemma~\ref{lem: claim} and the preceding discussion, \textit{it suffices to assume that the given mechanism $\alg$ is low bias, $2\alpha/\sqrt{d}$-$\ell_\infty$-accurate with probability at least $1-\beta$ for any $\beta > 0$, $\alg^j = \alg^l$ for all $j, l$, and that $\alg$ has expected $\ell_2$-mean-squared-error bounded by $\alpha^2$, 
for the remainder of the proof} of~\cref{thm: DSS15 l2 version}. We also assume in what follows that $\theta$ is drawn uniformly at random from $[-a,a]^d$ and that conditional on $\theta$, the data $X$ is drawn i.i.d. from the product distribution $\ptheta$, as described in the statement of~\cref{thm: DSS15 l2 version}. 

Now, we will use Lemma~\ref{lem: fingerprint} to lower bound the sum of expected scores $\sum_{i=1}^n \expec \langle x_i - \theta,  \trunc \rangle$:
\begin{lemma}
\label{lem: lemma19}
Let $1 \geq a > 100 \alpha/\sqrt{d}$. Suppose $\expec\|\alg(X) - \theta\|^2 \leq \alpha^2$, $\alg^j = \alg^l$ is low bias for all $j, l \in [d]$,
and $P(|\alg^j(X) - \theta^j| > \eta) \leq 1/48n$. Then, \[
\expec\left[\sum_{i=1}^n (x_i^j - \theta^j) \trunca^j \right] \geq \frac{1}{24}
\]
for all $j \in [d]$.
\end{lemma}
\begin{proof}
Note that $\expec[(\alg^j(X^j) - \theta^j)^2] \leq \frac{\alpha^2}{d}$ because $\alg^j = \alg^l$ for all $j,l$ and $\expec\|\alg(X) - \theta\|^2 \leq \alpha^2$. 
By Lemma~\ref{lem: fingerprint}, we have \begin{align*}
    \expec\left[\sum_{i=1}^n (x_i^j - \theta^j) (\alg^j(X^j) - \theta^j) \right] &\geq \frac{a^2}{3} - \expec[(\alg^j(X^j) - \theta^j)^2] \\
    &\;\;\; + \frac{1 - a^2}{2a} \left(\expec_{X^j \sim P_a^n}[\alg^j(X^j) - a] - \expec_{X^j \sim P_{-a}^n}[\alg^j(X^j) + a]\right) + 1 - a^2 \\
    &\geq 1 -\frac{2}{3}a^2 - \frac{\alpha^2}{d} 
    \\
    &\geq 1/6,
\end{align*}
since $\alg^j$ is low bias and using the assumptions on the values of $a$ and $\alpha$. 

Also, by the law of total expectation, we have \begin{align*}
\expec\left[\left((\alg(X)^j - \theta^j) - \trunca^j\right)\sum_{i=1}^n (x_i^j - \theta^j) \right] &\leq 2n \expec\left[\left((\alg(X)^j - \theta^j) - \trunca^j\right) \Big | \normalsize |\alg(X)^j - \theta^j| > \eta \right] \frac{1}{48n} \\
&\leq \frac{1}{12}.
\end{align*}
Combining the above inequalities completes the proof. 
\end{proof}
As mentioned earlier, $P(|\alg^j(X) - \theta^j| > \eta) \leq 1/48n$ (and the other assumptions in the lemma) can be ensured by Lemma~\ref{lem: claim}.

Below we obtain a high probability lower bound on the sum of the scores:
\begin{proposition}
\label{prop20}
Let $1 \geq a > 100 \alpha/\sqrt{d}$. Suppose $\expec\|\alg(X) - \theta\|^2 \leq \alpha^2$, $\alg^j = \alg^l$ is low bias for all $j, l \in [d]$,
and $P(|\alg^j(X) - \theta^j| > \eta) \leq 1/48n$. If $d \geq 200 \alpha n \sqrt{\ln(1/\delta)}$, then \[
P\left[\sum_{i=1}^n \langle x_i - \theta, \trunca \rangle \geq \frac{d}{48}\right] \geq 1 - \delta. 
\]
\end{proposition}
\begin{proof}
We use the concentration inequality of~\citet[Theorem 36]{dwork2015robust} and~Lemma~\ref{lem: lemma19}. Specifically, Lemma~\ref{lem: lemma19} implies that the assumptions of~\citet[Theorem 36]{dwork2015robust} are satisfied with $X_{i,j} = x_i^j - \theta^j$, $c = 1/2$, $Y_j = \trunca^j$, $\gamma_j = 1/24$, $\alpha \to \eta$, and $\mathbf{a} = \mathbbm{1}_n$ as the vector of all $1$'s. Thus, for any $\lambda > 0$, we have \begin{align*}
P\left[\sum_{i=1}^n \langle x_i - \theta, \trunca \rangle < \frac{d}{24} - \lambda \right] &\leq \exp\left(- \frac{\lambda^2}{8 d \eta^2 n^2} \right) \\
&= \exp\left(- \frac{\lambda^2}{8 \alpha^2 n^2} \right).  
\end{align*}
Choosing $\lambda = \alpha n \sqrt{8 \ln(1/\delta)} \leq d/48$ completes the proof. 
\end{proof}

Next, we provide a high probability upper bound on the sum of the squared scores: 
\begin{proposition}{\cite{dwork2015robust}}
    \label{prop23}
Fix $\theta \in [-1, 1]^d$ and let $X \sim \ptheta^n$. Assume $d \geq 64(n + \sqrt{\ln(1/\delta)})$. Then, \[
P\left[\sum_{i=1}^n \langle x_i - \theta, \trunca \rangle^2 > 4\eta^2 d^2\right] \leq \delta. 
\]
\end{proposition}
\begin{proof}
    This is immediate from \citet[Lemma 22 and Proposition 23]{dwork2015robust} and the proofs of these results. 
\end{proof}

The next elementary lemma is taken verbatim from~\citep[Lemma 24]{dwork2015robust}:
\begin{lemma}{\cite{dwork2015robust}}
\label{lem24}
Let $\sigma \in \mathbb{R}^n$ satisfy $\sum_{i=1}^n \sigma_i \geq A$ and $\sum_{i=1}^n \sigma_i^2 \leq B^2$. Then, \[
\left|\left\{i \in [n] : \sigma_i \geq \frac{A}{2n} \right\} \right| \geq \left(\frac{A}{2B}\right)^2. 
\]
\end{lemma}

We are now prepared to prove~\cref{thm: DSS15 l2 version}:
\begin{proof}[Proof of~\cref{thm: DSS15 l2 version}]
For notational convenience, we will assume that $\widetilde{\alg} = \alg$ is low bias, symmetric ($\alg^j = \alg^l$), and accurate in both expected $\ell_2$ norm and high probability $\ell_\infty$ norm. This is without loss of generality, by~Lemma and Lemma~\ref{lem: claim}. We first prove a): Assume $y$ is independent of $X$ (drawn from the same distribution). By Hoeffding's inequality, 
\begin{align*}
P\left(\mathcal{I}(y, \alg(X)) = IN\right) &= P\left(\langle y - \theta, \trunca \rangle > \tau = 2\eta \sqrt{d \ln(1/\delta)} \right) \\
&\leq \exp\left(- \frac{2 \tau^2}{d (4\eta)^2}\right) \leq \delta. 
\end{align*}

Next, we prove b): Note that the assumptions on $d$ and $\alpha$ imply that $d \geq 64(n + \sqrt{\ln(1/\delta)})$. Proposition~\ref{prop20} implies that \[
\sum_{i=1}^n \langle x_i - \theta, \trunca \rangle \geq \frac{d}{48} =: A
\]
with probability at least $1 - \delta$. Proposition~\ref{prop23} implies that \[
\sum_{i=1}^n \langle x_i - \theta, \trunca \rangle^2 \leq 4\eta^2 =: B^2 
\]
with probability at least $1 - \delta$. By a union bound, both of the above events occur with probability at least $1 - 2\delta$. Then, Lemma~\ref{lem24} implies that \[
\left|\left\{i \in [n] : \langle x_i - \theta, \trunca \rangle \geq \frac{d}{96n} \right\} \right| \geq \left(\frac{A}{2B}\right)^2 \geq \frac{d}{10^6 \alpha^2}.
\]
Moreover, $\frac{d}{96n} \geq \tau = 4\alpha \sqrt{\ln(1/\delta)}$ by the assumption $d > 400 \alpha n \sqrt{\ln(1/\delta)}$. This completes the proof. 
\end{proof}

Now we are ready to prove~\cref{thm: approx DP mean estimation with small theta}: 
\begin{proof}[Proof of~\cref{thm: approx DP mean estimation with small theta}]
We will first prove a lower bound that is larger than the one stated in~\cref{thm: approx DP mean estimation with small theta} by a factor of $d$ for distributions on $\XX' := \{\pm 1\}^d$ and $\|\expec_{x \sim P}[x]\| \leq \sqrt{d} a$, where $a := \min\left(\frac{1}{\sqrt{\npu}}, \frac{\sqrt{d}}{\eps \npr}\right)$. We will re-scale at the end to obtain~\cref{thm: approx DP mean estimation with small theta}.

By a standard reduction (see e.g.~\citep[Lemma 2.5]{bun2014fingerprinting}), it suffices to prove the lower bound for $\eps = 1$.  Let $\ptheta$ be the product distribution described in \cref{thm: DSS15 l2 version} with $\theta^j \sim \textbf{Unif}([-a,a])$ for $j \in [d]$ and $a = \min\left(\frac{1}{\sqrt{\npu}}, \frac{\sqrt{d}}{\eps \npr}\right)$. Let $\alg = (\alg^1, \ldots, \alg^d)$ be $(1, o(1/\npr))$-semi-DP with $\expec\|\alg(X) - \theta\|^2 \leq \alpha^2 := \sup_{P}\expec_{X \sim P^n}\| \alg(X) - \expec_{x \sim P}[x]\|^2$, where the supremum is taken over all distributions on $\XX'$. By Lemma~\ref{lem: claim} and our earlier discussions, we assume without loss of generality that $\alg$ is low bias, $\alg^j = \alg^l$ for all $j, l \in [d]$, and that $\|\alg(X) - \expec_{x \sim \ptheta}[x]\|_\infty \leq 2\alpha/\sqrt{d}$ with arbitrarily high probability. 

Note that $\alpha^2 \gtrsim \frac{d}{n}$ always holds by the non-private lower bound. Moreover, if $\alpha^2 > da^2/10000$, then we are done. Thus, we may assume $\alpha^2 \leq da^2/10000$. We will derive a contradiction under this assumption. 

Let $\xpr = (x_1, \ldots, x_{\npr})$ denote the private samples in $X$ and $\xpu = (w_1, \ldots, w_{\npu})$ denote the public samples in $X = (\xpr, \xpu) = (z_1, \ldots, z_n) \sim \ptheta^n$. Let $r = \frac{d}{400 \alpha \sqrt{\ln(2/\gamma)}}$, $t = \frac{d}{10^6 \alpha^2}$, and $\gamma \leq \frac{1}{(r+t)^2}$. We will show that if $\alpha \ll \frac{d}{\npr}$ and $\alpha \ll \sqrt{\frac{d}{\npu}}$, then $\alg$ is not $(1, o(1/\npr))$-semi-DP. Assume $\npr < r$, $\npu < t/2 - 1$, and $t < n < r + (t/2 - 1)$. Thus, the assumptions in~\cref{thm: DSS15 l2 version} hold. We will show that the attack given in~\cref{alg: attack} succeeds at identifying at least $\npu + 1$ people in $X$ with high probability: By part b) of~\cref{thm: DSS15 l2 version} with $\delta = \gamma$, we have \begin{align*}
    P(|\{i \in [n]: \attack(z_i, \alg(X)) = IN \}| \geq \npu + 1) &\geq P(|\{i \in [n]: \attack(z_i, \alg(X)) = IN \}| \geq t/2) \\
    &\geq 1 - \frac{1}{(r+t)^2} \\
    &\geq 1 - \frac{1}{\npr^2}. 
\end{align*}
That is, $\attack$ identifies at least $\npu + 1$ individuals in the data set with high probability $\geq 1 - 1/\npr^2$. Let $v_i = \mathbbm{1}_{\{\attack(z_i, \alg(X)) = IN \}}$ be the indicator of the event $\attack(z_i, \alg(X)) = IN$. By Markov's inequality, we have \[
\expec \left[\sum_{i=1}^n v_i\right] = \sum_{i=1}^{\npr} P(\attack(x_i, \alg(X)) = IN) + \sum_{i=1}^{\npu} P(\attack(w_i, \alg(X)) = IN) \geq (\npu + 1)(1 - 1/n^2). 
\]
Now, $\sum_{i=1}^{\npu} P(\attack(w_i, \alg(X)) = IN) \leq \npu$, which implies that \begin{align*}
\sum_{i=1}^n P(\attack(x_i, \alg(X)) = IN) &\geq (\npu + 1)(1 - 1/n^2) - \npu \\
&\geq 1 - 1/\npr\\
&\geq 1/2.
\end{align*}
Thus, there exists a private sample $x_{i^*} \in \xpr$ such that \[
P(\attack(x_{i^*}, \alg(X)) = IN) \geq 1/2\npr.
\]
Now consider the adjacent data set $X'$ obtained by replacing $x_{i^*}$ with an independent sample $y \sim \ptheta$. Then, part a) of \cref{thm: DSS15 l2 version} implies \[
P(\attack(x_{i^*}, \alg(X')) = IN) \leq \frac{1}{(r+t)^2} \leq \frac{1}{n^2} \leq \frac{1}{\npr^2},
\]
where the probability is taken over the random independent draws of all the data points (including $y$) and the mechanism $\alg$. Thus, $\alg$ cannot satisfy $(\eps, 1/4\npr)$-semi-DP unless \[
\frac{1}{2\npr} \leq e^\eps \frac{1}{\npr^2} + \frac{1}{4\npr} \implies \ln(n/4) \leq \eps.
\]
In particular, if $\npr \geq 11$, then $\alg$ cannot be $(1, o(1/\npr))$-semi-DP, which contradicts our earlier hypothesis. This proves the unscaled lower bound on $\XX'$: $\alpha^2 = \widetilde{\Omega}\left(d  \min\left(1/\npu, d/(\eps \npr)^2 + 1/n\right) \right)$.  

Now we will scale the lower bound: Let $\XX = \left\{\pm 1/\sqrt{d} \right\}^d$. Draw $\theta^j$ uniformly from $[-a/\sqrt{d}, a/\sqrt{d}]$ for all $j \in [d]$ and then let $\ptheta$ be the distribution on $\XX$ that has mean $\theta \in [-a/\sqrt{d}, a/\sqrt{d}]^d$. Note that $\|\expec_{x \sim \ptheta}[x]\| \leq \min\left(1/\sqrt{\npu}, \sqrt{d}/\eps \npr\right)$ for any $\theta$. Moreover, for any $(\eps, \delta)$-semi-DP $\alg: \XX^n \to [-a/\sqrt{d}, a/\sqrt{d}]^d$, there exists  $(\eps, \delta)$-semi-DP $\alg': (\XX')^n \to [-a, a]^d$ such that $\alg(X) = \frac{1}{\sqrt{d}} \alg'(X')$ for $X = \frac{1}{\sqrt{d}}X'$. Applying our lower bound for the MSE of $\alg'$, we have \begin{align*}
  \sup_{\theta \in [-a/\sqrt{d}, a/\sqrt{d}]^d} \expec_{X \sim \ptheta^n}\left[\|\alg(X) - \theta\|^2\right] &= \sup_{\theta' \in [-a, a]^d}\expec_{X \sim P_{\theta'}^n}\left\|\frac{1}{\sqrt{d}}(\alg'(X) - \theta')\right\|^2 \\
  &= \frac{1}{d} \sup_{\theta' \in [-a, a]^d}\expec_{X \sim P_{\theta'}^n}\left\|\alg'(X) - \theta'\right\|^2 \\
  & \geq \frac{1}{d} \widetilde{\Omega}\left(d  \min\left(1/\npu, d/(\eps \npr)^2 \right) \right) \\
  & = \widetilde{\Omega}\left(\min\left\{\frac{1}{\npu}, \frac{d}{\eps^2 \npr^2} + \frac{1}{n}\right\} \right),
\end{align*}
as desired. 
\end{proof}

With~\cref{thm: approx DP mean estimation with small theta} in hand, we can easily complete the proof of~\cref{thm: CDP mean estimation app} by giving a matching upper bound (up to log factors):

\begin{proof}[Proof of~\cref{thm: CDP mean estimation app}]

\noindent \textbf{Lower bound:} This was proved in~\cref{thm: approx DP mean estimation with small theta}.

\noindent \textbf{Upper bound:}
First, the throw-away estimator $\alg(X) = \frac{1}{\npu} \sum_{x \in \xpu} x$ is clearly $(0,0)$-semi-DP and has MSE \begin{align*}
\expec_{X \sim P^n} \left\| \alg(X) - \expec_{x \sim P}[x]\right\|^2 &= \expec_{X \sim P^n} \left\| \frac{1}{\npu} \sum_{x \in \xpu} \left(x - \expec_{x \sim P}[x]\right)\right\|^2 \\
&= \frac{1}{\npu^2} \sum_{x \in \xpu} \expec_{x \sim P}\left\|x - \expec_{x \sim P}[x]\right\|^2\\
&\leq \frac{1}{\npu}.
\end{align*} 

To get the second term in the minimum in~\cref{eq: delta > 0 app}, consider the Gaussian mechanism $\alg(X) = \Bar{X} + \mathcal{N}(0, \sigma^2 \mathbf{I}_d)$, where $\sigma^2 = \frac{8 \ln(2/\delta)}{\eps^2 n^2}$. We know $\alg$ is $(\eps, \delta)$-DP (by e.g. \cite{dwork2014}) since the $\ell_2$-sensitivity is $\sup_{X \sim X'}\| \Bar{X} - \Bar{X}'\|_2 \leq \frac{2}{n}$.  Hence $\alg$ is $(\eps, \delta)$-semi-DP. Moreover, the MSE of $\alg$ is \begin{align*}
\expec_{X \sim P^n} \left\|\alg(X) - \expec_{x \sim P}[x]\right\|^2 &= \expec\|\alg(X) - \Bar{X}\|^2 + \expec\left\|\Bar{X} - \expec_{x \sim P}[x]\right\|^2 \\
&\leq \frac{8 d \ln(2/\delta)}{\eps^2 n^2} + \frac{1}{n}.
\end{align*}
This completes the proof of~\cref{thm: CDP mean estimation app}.
\end{proof}

Next, we turn to the pure $\eps$-semi-DP case. 
\paragraph{Pure $\eps$-Semi-DP Mean Estimation}
\label{app sec: pure semi-DP mean estimation}
\begin{theorem}[Pure $\eps$-semi-DP mean estimation]
\label{thm: pure semi-DP mean estimation}
    Let $\eps \leq d/8$ and either $\npu \lesssim \frac{n \eps}{d}$ or $d\lesssim 1$. Then, there exist absolute constants $c$ and $C$, with $0 < c \leq C$, such that
\begin{equation}
\label{eq: delta = 0 app}
c\min\left\{\frac{1}{\npu}, \frac{d^2}{n^2 \eps^2} + \frac{1}{n}\right\} \leq \popDPriskpure \leq C \min\left\{\frac{1}{\npu}, \frac{d^2}{n^2 \eps^2} + \frac{1}{n}\right\}.
\end{equation}
Moreover, the upper bound in~\cref{eq: delta = 0 app} holds for any $\npu, d$. 
\end{theorem}
The restriction on $\npu$ is needed for our lower bound proof---via \cref{lem: semiDP Fano}---to work. It seems challenging to remove this restriction, but we do believe the same lower bound holds in the complementary parameter regime. Proving this may require the invention of new techniques, making it an interesting direction for future work. 

The proof of~\cref{eq: delta = 0 app} will require the following intermediate result, \cref{lem: semiDP Fano}, which can be viewed as a ``semi-DP Fano's inequality.''
\begin{theorem}
\label{lem: semiDP Fano}
Let $\{P_v\}_{v \in \VV} \subset \PP$ be a family of distributions on $\XX$. Let $P_0$ be a distribution and $p \in [0, 1]$ such that $P_{\theta_v} := (1-p) P_0 + p P_v \in \PP$ for all $v \in \VV$. Denote $\theta_v := \expec_{x \sim P_{\theta_v}}[x]$ and $\rho^*(\VV) := \min\left\{ \|\theta_v - \theta_{v'}\| ~\vert~ v, v' \in \VV, v \neq v' \right\}$. Let $\widehat{\theta}: \XX^n \to \XX$ be any $\epsemi$ estimator. Draw $V \sim \text{Unif}(\VV)$; then conditional on $V = v$, draw an i.i.d. sample $X | V = v \sim \ptv^n$ containing $\npr$ private samples and $\npu$ samples. 
Then, 
\begin{equation}
\label{eq: Fano}
\frac{1}{|\VV|} \sum_{v \in \VV} \ptv\left(\left\|\thetahat(X) - \thetav\right\| \geq \rho^*(\VV) \right) \geq \frac{(|\VV| - 1) e^{-\eps \lceil \npr p \rceil} (1 - p)^{\npu}}{2 \left(1 + (|\VV| - 1)  e^{-\eps \lceil \npr p \rceil} \right)}
\end{equation}
\end{theorem}
\begin{remark}
Note that the right-hand-side of~\cref{eq: Fano} is similar to the second term on the right-hand-side of DP Fano's inequality~\citep[Equation 5]{acharya2021differentially}, after aligning notation. The main differences are that~\cref{eq: Fano} has an extra factor of $(1 - p)^{\npu}$, and $\npr$ in place of $n$. 
\end{remark}

\begin{proof}[Proof of~\cref{lem: semiDP Fano}]
Our proof builds on the techniques of~\citet{bd14}.
A key step in the proof is~\cref{eq: Lem 2}: if $A$ is a measurable set and $v, v' \in \VV$, then
\begin{equation}
\label{eq: Lem 2}
\ptv(\thetahat(X) \in A) \geq e^{\eps \lceil \npr p \rceil} \left[\ptvprime(\thetahat \in A) - 1 + \frac{(1 - p)^{\npu}}{2}\right].
\end{equation}
Let us now prove~\cref{eq: Lem 2}: We will use upper case letters to denote random variables and lower case letters to denote the values that the random variables take. Let $B = \{B_i\}_{i=1}^n$ be i.i.d. $\text{Bernoulli}(p)$ random variables. Assume that the random variables $X = \{X_i\}_{i=1}^n$ are generated in the following way: first draw $W_1^0, \ldots, W_n^0 \sim P_0$ i.i.d. and draw $W_1^v, \ldots, W_n^v \sim P_v$ i.i.d. For each $i$, if $B_i = 0$, set $X_i = W_i^0$; if $B_i = 1$, set $X_i = W_i^v$. Thus, conditional on $V = v$, the random variables $X_i$ are each distributed according to $\ptv = (1- p) P_0 + p P_v$. For fixed $v' \in \VV$, generate a different sample $X' = \{X'_i\}_{i=1}^n$ by drawing $W_i^{v'} \sim P_{v'}$ i.i.d. and setting $X'_i = W_i^0 (1 - B_i) + W_i^{v'} B_i$. Note that if $B_i = 0$, then $X_i = X'_i$. Thus, the hamming distance between $X$ and $X'$ is 
\begin{equation*}
d_{\text{ham}}(X, X') \leq B^T \mathbbm{1} = \sum_{i=1}^n B_i.
\end{equation*}
Now let $Q$ denote the conditional distribution of the $\epsemi$ estimator $\thetahat$ given input data ($X$ or $X'$). For notational convenience, assume without loss of generality that $X = (X_1, \ldots, X_{\npr}, \xpu)$ and $X' = (X'_1, \ldots, X'_{\npr}, \xpu')$. Then, for any fixed sequence $b = (b_1, \ldots, b_{n_{\text{priv}}}, \mathbf{0}_{n_\text{pub}}) \in \{0, 1\}^{\npr} \times \{0\}^{\npu}$, we have 
\small
\begin{align}
\label{eq: group privacy}
Q(\thetahat \in A | X_i = W_i^0 (1 - b_i) + W_i^{v} b_i ~\forall i \in [n]) \geq e^{-\eps b^T \mathbbm{1}} Q(\thetahat \in A | X'_i = W_i^0 (1 - b_i) + W_i^{v'} b_i ~\forall i \in [n])
\end{align}
\normalsize
by group privacy, since $b_i = 0 ~\forall i > \npr$ implies $\xpu = \xpu'$. 
Thus, \begin{align*}
&\ptv(\thetahat \in A) \\
= & \sum_{b \in \{0, 1\}^n} P(B = b) \int Q(\thetahat \in A | X_i = w_i^0 (1 - b_i) + w_i^{v} b_i ~\forall i \in [n]) dP_0^n(w_{1:n}^0) dP_v^n(w_{1:n}^v) \\
\geq & \sum_{\substack{b \in \{0, 1\}^{\npr} \times \{0\}^{\npu}, \\ b^T \mathbbm{1} \leq \lceil \npr p \rceil}} P(B = b) \int Q(\thetahat \in A | X_i = w_i^0 (1 - b_i) + w_i^{v} b_i ~\forall i \in [n]) dP_0^n(w_{1:n}^0) dP_v^n(w_{1:n}^v) \\
= & \sum_{\substack{b \in \{0, 1\}^{\npr} \times \{0\}^{\npu}, \\ b^T \mathbbm{1} \leq \lceil \npr p \rceil}} P(B = b) \int Q(\thetahat \in A | X_i = w_i^0 (1 - b_i) + w_i^{v} b_i ~\forall i \in [n]) dP_0^n(w_{1:n}^0) dP_v^n(w_{1:n}^v) dP_{v'}^n(w_{1:n}^{v'})\\
\geq & \sum_{\substack{b \in \{0, 1\}^{\npr} \times \{0\}^{\npu}, \\ b^T \mathbbm{1} \leq \lceil \npr p \rceil}} P(B = b) \int \Big[e^{-\eps b^T \mathbbm{1}} Q(\thetahat \in A | X'_i = w_i^0 (1 - b_i) + w_i^{v'} b_i ~\forall i \in [n]) dP_0^n(w_{1:n}^0) \\
&\;\;\;\;\;\;\;\;\;\;\;\;\;\;\;\;\;\;\;\;\;\;\;\;\;\;\;\;\;\;\;\;\;\;\;\;\;\;\;\;\;\;\;\;\;\;\;\;\;\;\;\;\;\; \cdots dP_v^n(w_{1:n}^v) dP_{v'}^n(w_{1:n}^{v'})\Big] \\
= & \sum_{\substack{b \in \{0, 1\}^{\npr} \times \{0\}^{\npu}, \\ b^T \mathbbm{1} \leq \lceil \npr p \rceil}} P(B = b) e^{-\eps b^T \mathbbm{1}} \ptvprime\left(\thetahat \in A | B= b\right) \\
\geq & \sum_{\substack{b \in \{0, 1\}^{\npr} \times \{0\}^{\npu}, \\ b^T \mathbbm{1} \leq \lceil \npr p \rceil}} P(B = b) e^{-\eps \lceil \npr p \rceil} \ptvprime\left(\thetahat \in A | B= b\right) \\
\geq &  e^{-\eps \lceil \npr p \rceil} \ptvprime\left(\thetahat \in A , B_{(i > \npr)} = \mathbf{0}_{n_{\text{pub}}}, B^T \mathbbm{1} \leq \lceil \npr p \rceil \right) \\
\geq & e^{-\eps \lceil \npr p \rceil} \left[ \ptvprime\left(\thetahat \in A \right) -  P\left(B_{(i > \npr)} \neq \mathbf{0}_{n_{\text{pub}}} \bigcup B^T \mathbbm{1} > \lceil \npr p \rceil \right) \right],
\end{align*}
where the second inequality used~\cref{eq: group privacy} and the last inequality used a union bound.
Now, by independence of $\{B_i\}_{i=1}^n$, we have 
\begin{align*}
& P\left(B_{(i > \npr)} \neq \mathbf{0}_{n_{\text{pub}}} \bigcup B^T \mathbbm{1} > \lceil \npr p \rceil \right) \\
= & 1 - P\left(B_{(i > \npr)} = \mathbf{0}_{n_{\text{pub}}} \bigcap B^T \mathbbm{1} \leq \lceil \npr p \rceil \right) \\
= & 1 - P\left(B_{(i > \npr)} = \mathbf{0}_{n_{\text{pub}}}\right) P\left(B^T \mathbbm{1} \leq \lceil \npr p \rceil \right | B_{(i > \npr)} = \mathbf{0}_{n_{\text{pub}}}) \\
= & 1 - P\left(B_{(i > \npr)} = \mathbf{0}_{n_{\text{pub}}}\right) P\left(B_{(i \leq \npr)}^T \mathbbm{1} \leq \lceil \npr p \rceil \right ) \\
= & 1 - (1 - p)^{\npu} P\left(B_{(i \leq \npr)}^T \mathbbm{1} \leq \lceil \npr p \rceil \right) \\
\leq & 1 - (1 - p)^{\npu} \cdot \frac{1}{2},
\end{align*}
since the median of $B_{(i \leq \npr)}^T \mathbbm{1} \sim \text{Binomial}(\npr, p)$ is no larger than $\lceil \npr p \rceil$. 
Putting together the pieces, we get \begin{align*}
\ptv(\thetahat \in A) \geq e^{-\eps \lceil \npr p \rceil} \left[ \ptvprime\left(\thetahat \in A \right) -  1 + (1 - p)^{\npu} \cdot \frac{1}{2} \right],
\end{align*}
establishing~\cref{eq: Lem 2}. 

Now, with~\cref{eq: Lem 2} in hand, we proceed as in the proof of~\citet[Theorem 3]{bd14}: 
Let $\mathbb{B}_{\alpha}(\theta) := \{\theta' \in \XX : \|\theta - \theta'\| \leq \alpha\}$. Note that the balls $\bbrho(\theta_v)$ are disjoint for $v \in \VV$ by definition of $\rho^*(\VV)$. Denote the average probability of success for the estimator $\thetahat$ by \[
P_{\text{succ}} := \frac{1}{| \VV |} \sum_{v \in \VV} \ptv\left(\thetahat \in \bbrho(\theta_v)\right).
\]
Then by a union bound and disjointness of the balls $\{\bbrho(\theta_v)\}_{v \in \VV}$, we have \[
\psucc \leq 1 - \frac{1}{| \VV |} \sum_{v \in \VV} \sum_{v' \in \VV, v' \neq v} \ptv\left(\thetahat \in \bbrho(\theta_{v'}) \right).
\]
An application of~\cref{eq: Lem 2} yields \begin{align*}
\psucc &\leq 1 - \frac{1}{| \VV |} \sum_{v \in \VV} \sum_{v' \in \VV, v' \neq v} \left[e^{-\eps \lceil \npr p \rceil} \ptvprime\left(\thetahat \in \bbrho(\theta_{v'}) \right) - \left(1 - \frac{(1 - p)^{\npu}}{2}\right)\right] \\
&\leq 1 - e^{-\eps \lceil \npr p \rceil} (|\VV| - 1) \psucc + e^{-\eps \lceil \npr p \rceil} (|\VV| - 1) \left(1 - \frac{(1 - p)^{\npu}}{2}\right).
\end{align*}
Re-arranging this inequality leads to \[
\psucc \leq \frac{1 + (|\VV| - 1)\left(1 - \frac{(1 - p)^{\npu}}{2}\right) e^{-\eps \lceil \npr p \rceil}}{1 + (|\VV| - 1) e^{-\eps \lceil \npr p \rceil}}
\]
and hence 
\[
1 - \psucc \geq \frac{(|\VV| - 1)e^{-\eps \lceil \npr p \rceil} \cdot \frac{(1 - p)^{\npu}}{2}}{1 + (|\VV| - 1) e^{-\eps \lceil \npr p \rceil}}.
\]
This last inequality is equivalent to the inequality stated in \cref{lem: semiDP Fano}.
\end{proof}
While we state \cref{lem: semiDP Fano} for mean estimation, it holds more generally for estimating any population statistic $\theta: \PP \to \Theta$. However, this additional generality will not be necessary for our purposes. With~\cref{lem: semiDP Fano} in hand, we now turn to the proof of~\cref{thm: CDP mean estimation app}.

\begin{proof}[Proof of~\cref{thm: pure semi-DP mean estimation}]
\textbf{Lower Bounds:} 
We begin by proving~\cref{eq: delta = 0 app}. First suppose $\npu \lesssim \frac{n \eps}{d}$ and $d \geq 8$. 

Choose a finite subset $\VV \subset \mathbb{R}^d$ such that $|\VV| \geq 2^{d/2}$, $\|v\| = 1$, and $\|v - v'\| \geq \frac{1}{8}$ for all $v, v' \in \VV, v \neq v'$. The existence of such a set of points is well-known (see e.g. the Gilbert-Varshamov construction). Define $P_0$ to be the point mass distribution on $\{X = 0\}$ and $P_v$ to be point mass on $\{X = v\}$ for $v \in \VV$. For $v \in \VV$, let $\ptv := (1-p) P_0 + p P_v$ for some $p \in [0,1]$ to be specified later. Note that if $X \sim \ptv$, then $\|X\| \leq 1$ with probability $1$. Thus, $\ptv$ is a valid distribution in the class $\PP$ of bounded (by $1$) distributions on $\mathbb{B}$ that we are considering. Also, note that $\theta_v := \expec_{P_{\theta_v}}[X] = pv$. Further, \[
\rho^*(\VV) := \min \left\{ \| \theta_v - \theta_{v'} \| | v, v' \in \VV, v\neq v' \right\} \geq \frac{p}{8}
\]
by construction. 

Now we use the classical reduction from estimation to testing (see \cite{bd14} for details) to lower bound the MSE of any $\epsemi$ estimator $\thetahat$ by \begin{align*}
\sup_{P \in \PP} \expec_{X \sim P^n, \thetahat}\left\|\thetahat(X) - \expec_{x \sim P}[x]\right\|^2 &\geq \rho^{*}(\VV)^2 \frac{1}{|\VV|}\sum_{v \in \VV} \ptv\left(\|\thetahat(X) - \theta_v\| \geq \rho^*(\VV) \right) \\
&\geq \left(\frac{p}{8}\right)^2 \frac{(|\VV| - 1) e^{-\eps \lceil \npr p \rceil} (1 - p)^{\npu}}{2 \left(1 + (|\VV| - 1)  e^{-\eps \lceil \npr p \rceil} \right)} \\
&\geq \frac{p^2}{64} \frac{(2^{d/2} - 1) e^{-\eps \lceil \npr p \rceil} (1 - p)^{\npu}}{2 \left(1 + (2^{d/2} - 1)  e^{-\eps \lceil \npr p \rceil} \right)}
\end{align*}
where we used \cref{lem: semiDP Fano} in the second inequality. Since we assumed $d \geq 8$, we have $2^{d/2} - 1 \geq e^{d/4}$ and hence \begin{align*}
\sup_{P \in \PP} \expec_{X \sim P^n, \thetahat}\left\|\thetahat(X) - \expec_{x \sim P}[x]\right\|^2 &\geq \frac{p^2}{64} \cdot \frac{(1-p)^{\npu}}{2} \cdot \frac{e^{d/4 -\eps \lceil \npr p \rceil}}{1 + e^{d/4 -\eps \lceil \npr p \rceil}} \\
&\geq \frac{p^2}{64} \cdot \frac{(1-p)^{\npu}}{4}.
\end{align*}
for any $p \leq \frac{d}{4 n \eps} - \frac{1}{n}$. 
Now choose \[
p = \min\left(\frac{d}{4 n \eps} - \frac{1}{n}, \frac{1}{2 \sqrt{\npu}}\right).
\]
By assumption, there exists an absolute constant $k$ such that $\npu \leq k \frac{n \eps}{d}$. 
Thus, if $n \eps \geq 2$, then \begin{align*}
    (1 - p)^{\npu} &\geq \left(1 - \frac{d}{\eps n}\right)^{k n \eps/d} \\
    &= \left[\left(1 - \frac{d}{\eps n}\right)^{n \eps/d} \right]^k \\
    &\geq \left[\left(1 - \frac{1}{2} \right)^2\right]^k \\
    &= \frac{1}{4^k}. 
\end{align*}
On the other hand, if $n \eps < 2$, then $\npu \leq 2 k \implies (1 - p)^{\npu} \geq \left(1 - \frac{1}{2} \right)^{2k} = \frac{1}{4^k}$. Also, note that $p \geq \min\left(\frac{d}{8 \eps n}, \frac{1}{2 \sqrt{\npu}} \right)$ by the assumption that $d \geq 8 \eps$. 
Therefore, \[
\sup_{P \in \PP} \expec_{X \sim P^n, \thetahat}\left\|\thetahat(X) - \expec_{x \sim P}[x]\right\|^2 \geq \frac{1}{256 \cdot 4^k} \min\left(\frac{d}{8 \eps n}, \frac{1}{2 \sqrt{\npu}} \right)^2.
\]
Combining the above inequality with the non-private lower bound of $\Omega(1/n)$ for mean estimation proves the lower bound in~\cref{eq: delta = 0 app}. 

Now consider the alternative case in which $d \lesssim 1$ (i.e. $d \leq k$ for some absolute constant $k \in \mathbb{N}$), but $\npu \in [n]$ is arbitrary. Then we will prove that the lower bound in~\cref{eq: delta = 0 app} holds for $d = 1$, for any $\delta \in [0, \eps]$ and $\eps \leq 1$. By taking a $k$-fold product distribution, this will suffice to complete the proof of the lower bound in~\cref{eq: delta = 0 app}. To that end, we will use Le Cam's method and build on the techniques in~\cite{bd14, alireza}. The key novel ingredient is the following extension of~\citet[Lemma 3]{alireza} to the $(\eps, \delta)$-semi-DP setting: 
\begin{lemma}
\label{lem: alireza lem3}
Let $\thetahat: \XX^n \to \XX$ be $(\eps, \delta)$-semi-DP and let $P_1, P_2$ be distributions on $\XX$ such that $P_1$ is absolutely continuous w.r.t. $P_2$. Denote the conditional distribution of $\thetahat$ given $X$ by $Q$ and let $Q_j(A) = \int Q(\thetahat \in A | x_{1:n}) dP_j^n(x_{1:n})$ for any measurable set $A$. Then \[
\|Q_1 - Q_2\|_{\text{TV}} \leq \min \left\{ \sqrt{\frac{n}{2} D_{\text{KL}}(P_1, P_2)}, ~2\|P_1 - P_2\|_{\text{TV}} ~\npr (e^{\eps} - 1 + \delta) + \sqrt{\frac{\npu}{2} \dkl(P_1, P_2)}\right\}.
\]
\end{lemma}
Let us defer the proof of Lemma~\ref{lem: alireza lem3} for now. We will now use Lemma~\ref{lem: alireza lem3} to prove the lower bound in~\cref{eq: delta = 0 app} for $d = 1$ and $\eps \leq 1$. Define distributions $P_1, P_2$ on $\{-1, 1\}$ as follows: \[
P_1(-1) = P_2(1) = \frac{1 + \gamma}{2}, ~~ P_1(1) = P_2(-1) = \frac{1- \gamma}{2}
\]
for some $\gamma \in [0, 1/2]$ to be chosen later. Clearly $P_1, P_2 \in \PP$ (i.e. they are bounded by $1$ with probability 1). Also, $\expec_{P_1}[x] = - \gamma$ and $\expec_{P_2}[x] = \gamma$, so $\{P_1, P_2\}$ is a \textit{$\gamma$-packing} of $\{-1, 1\}$. Thus, by Le Cam's method (see~\cite{bd14} for details), for any $(\eps, \delta)$-semi-DP $\thetahat$, we have \[
\sup_{P \in \PP} \expec_{X \sim P^n, \thetahat}\left\|\thetahat(X) - \expec_{x \sim P}[x] \right\|^2 \geq \frac{\gamma^2}{8}\left(1 - \|Q_1 - Q_2\|_{\text{TV}} \right).
\]
Now, applying Lemma~\ref{lem: alireza lem3} and the assumption $\delta \leq \eps \leq 1$ yields \begin{align}
\label{eq: zzz}
    & \sup_{P \in \PP} \expec_{X \sim P^n, \thetahat}\left\|\thetahat(X) - \expec_{x \sim P}[x] \right\|^2 \\
    \geq & \frac{\gamma^2}{8}\left[1 - \min \left\{ \sqrt{\frac{n}{2} D_{\text{KL}}(P_1, P_2)}, ~6 \|P_1 - P_2\|_{\text{TV}} ~\npr \eps + \sqrt{\frac{\npu}{2} \dkl(P_1, P_2)}\right\} \right] \nonumber \\
    \geq & \frac{\gamma^2}{8}\left[1 - \min \left\{ \sqrt{\frac{n}{2} 3 \gamma^2 (P_1, P_2)}, ~6 \gamma \npr \eps + \sqrt{\frac{\npu}{2} 3 \gamma^2}\right\} \right] \nonumber \\
    = &\frac{\gamma^2}{8} \left[1 - \gamma \min\left\{\sqrt{\frac{3n}{2}}, ~6 \npr \eps + \sqrt{\frac{3\npu}{2}}\right\} \right]. 
\end{align}
In the second inequality we used the fact that $\|P_1 - P_2\|_{\text{TV}} = \frac{1}{2}\left(\left|\frac{1 + \gamma}{2} - \frac{1-\gamma}{2} \right| \cdot 2 \right) = \gamma$ and $\dkl(P_1, P_2) \leq 3 \gamma^2$ for $\gamma \leq 1/2$. 

Now we will choose $\gamma$ to (approximately) maximize the right-hand side of \cref{eq: zzz}. 
Suppose $\min\left\{\sqrt{\frac{3n}{2}}, 6 \npr \eps + \sqrt{\frac{3\npu}{2}}\right\} = \sqrt{\frac{3n}{2}}$. Then choosing $\gamma = \frac{1}{3} \sqrt{\frac{2}{3n}}$ yields \[
\sup_{P \in \PP} \expec_{X \sim P^n, \thetahat}\left\|\thetahat(X) - \expec_{x \sim P}[x] \right\|^2 \geq \frac{k}{n}
\]
for some absolute constant $k > 0$. Our assumption that $\min\left\{\sqrt{\frac{3n}{2}}, 6 \npr \eps + \sqrt{\frac{3\npu}{2}}\right\} = \sqrt{\frac{3n}{2}}$ implies that there exists $k' > 0$ such that $n \leq k' \max\left(\npr^2 \eps^2, \npu\right)$. Thus, $k/n \geq \frac{k}{k'} \min\left(\frac{1}{\npr^2 \eps^2}, \frac{1}{\npu}\right)$, which gives the desired lower bound in~\cref{eq: delta = 0 app}.

Suppose instead that $\min\left\{\sqrt{\frac{3n}{2}}, 6 \npr \eps + \sqrt{\frac{3\npu}{2}}\right\} = 6 \npr \eps + \sqrt{\frac{3\npu}{2}}$. Then choose $\gamma = \frac{2}{3}\left(6 \npr \eps + \sqrt{\frac{3 \npu}{2}} \right)^{-1}$. Then, there are constants $k, c > 0$ such that \begin{align*}
\sup_{P \in \PP} \expec_{X \sim P^n, \thetahat}\left\|\thetahat(X) - \expec_{x \sim P}[x] \right\|^2 &\geq \frac{\gamma^2}{8} \left[1 - \gamma \min\left\{\sqrt{\frac{3n}{2}}, 6 \npr \eps + \sqrt{\frac{3\npu}{2}}\right\} \right] \\
&\geq \frac{k}{\npr^2 \eps^2 + \npu} \\
&\geq c \min\left(\frac{1}{\npr^2 \eps^2}, \frac{1}{\npu}\right).
\end{align*}
Combining the above inequality with the non-private lower bound $\Omega(1/n)$ completes the proof of the lower bound in~\cref{eq: delta = 0 app}, assuming the truth of Lemma~\ref{lem: alireza lem3}. 

It remains to prove Lemma~\ref{lem: alireza lem3}. To that end, fix = $k \in \{0, \npr\}$ and denote by $\widetilde{Q}$ the marginal distribution of $\thetahat$ given $X_1, \ldots, X_k \sim P_1$ (i.i.d.) and $X_{k+1}, \ldots, X_n \sim P_2$ (i.i.d.); i.e. for measurable $A$, \[
\wtq(A) := \int Q(\thetahat \in A | X_{1:n} = x_{1:n}) dP_1^k(x_{1:k}) dP_2^{n-k}(x_{k+1:n}).
\]
Note that if $k = 0$, then $\wtq = Q_2$. We have \begin{equation*}
\|Q_1 - Q_2 \|_{\text{TV}} \leq \underbrace{\|Q_1 - \wtq \|_{\text{TV}}}_{\textcircled{a}} + \underbrace{\|\wtq - Q_2\|_{\text{TV}}}_{\textcircled{b}}.
\end{equation*}
Also, \begin{align}
\label{eq: thing}
&\min_{k \in \{0, \npr\}} \sqrt{\frac{n - k}{2} \dkl(P_1, P_2)} + 2\|P_1 - P_2\|_{\text{TV}} \min_{k \in \{0, \npr\}} k (e^\eps - 1 + \delta) \\
&\;\;\;\;\leq \min \left\{ \sqrt{\frac{n}{2} D_{\text{KL}}(P_1, P_2)}, ~2\|P_1 - P_2\|_{\text{TV}} ~\npr (e^{\eps} - 1 + \delta) + \sqrt{\frac{\npu}{2} \dkl(P_1, P_2)}\right\}, \nonumber
\end{align}
so it suffices to upper bound the sum of the terms of \textcircled{a} + \textcircled{b} by the left-hand-side of~\cref{eq: thing}. First, we deal with \textcircled{a}: for any $k$, we have \begin{align}
\|Q_1 - \wtq\|_{\text{TV}}^2 &\leq \|P_1^n - P_1^k P_2^{n-k}\|_{\text{TV}}^2 \\
&\leq \frac{1}{2} \dkl\left(P_1^n, P_1^k P_2^{n-k}\right) \\
&\leq \frac{n-k}{2}\dkl(P_1, P_2), 
\end{align}
by the data processing inequality for $f$-divergences, Pinsker's inequality, and the chain-rule for KL-divergences (see, e.g.~\cite{duchi2021lecture} for a reference on these facts).
Thus, it remains to show \begin{equation}
\label{eq:abi}
\|\wtq - Q_2\|_{\text{TV}} \leq 2\|P_1 - P_2\|_{\text{TV}} \min_{k \in \{0, \npr\}} k (e^\eps - 1 + \delta)
\end{equation}
for $k \in \{0, \npr\}$. If $k = 0$, \cref{eq:abi} is trivial. Assume $k = \npr$. Now, for any measurable $A$, we may write 
\begin{equation}
\label{eq:abu}
\wtq(A) - Q_2(A) = \int_{\mathbb{R}^{\npu}} \Delta(x_{\npr + 1:n}) dP_2^{\npu}(x_{\npr + 1: n}),
\end{equation}
where \[
\Delta(x_{\npr + 1:n}) := \int_{\mathbb{R}^{\npr}} Q(\thetahat \in A | X_{1:n} = x_{1:n}) \left( dP_1^{\npr}(x_{1:\npr}) - dP_2^{\npr}(x_{1:\npr}) \right).
\]
By~\cref{eq:abu}, it suffices to show that $|\Delta(x_{\npr + 1:n})| \leq 2\|P_1 - P_2\|_{\text{TV}} ~\npr (e^\eps - 1 + \delta)$ for all $x_{\npr + 1:n} \in \XX^{\npu}$. To do so, let $x_{1:n}^i := (x_1, \ldots, x_{i-1}, x_i', x_{i+1}, \ldots, x_n)$ for some $i \leq \npr$. Then \begin{equation}
\label{eq:aby}
|Q(\thetahat \in A |X_{1:n} = \xon) - Q(\thetahat \in A | X_{1:n} = \xoni)| \leq (e^{\eps} - 1) Q(\thetahat \in A | X_{1:n} = \xoni) + \delta
\end{equation}
since $\thetahat$ is $(\eps, \delta)$-semi-DP. Moreover, by the proof of~\citet[Lemma 3]{alireza}, we have \begin{align*}
\Delta(x_{\npr + 1:n}) &= \sum_{i=1}^{\npr} \int_{\mathbb{R}^{\npr}} \left[Q(\thetahat \in A | X_{1:n} = \xon) - Q(\thetahat \in A | X_{1:n} = \xoni)\right] \\
&\;\;\;\;\;\;\;\;\;\;\;\; \cdots dP_2^{i-1}(x_{1: i-1}) \left(dP_1(x_i) - dP_2(x_i)\right) dP_1^{\npr - i}(x_{i+1:\npr}).
\end{align*}
Applying the triangle inequality and~\cref{eq:aby}, we get 
\begin{align*}
&|\Delta(x_{\npr + 1:n})| \\
\leq & \sum_{i=1}^{\npr} \int_{\mathbb{R}^{\npr}} \left[(e^{\eps} - 1) Q(\thetahat \in A | X_{1:n} = \xoni) + \delta\right] dP_2^{i-1}(x_{1: i-1}) \left|dP_1(x_i) - dP_2(x_i)\right| dP_1^{\npr - i}(x_{i+1:\npr})\\
\leq & 2\|P_1 - P_2\|_{\text{TV}} ~\npr (e^\eps - 1 + \delta), 
\end{align*}
as desired. This completes the proof of Lemma~\ref{lem: alireza lem3} and hence the proof of the lower bound in~\cref{eq: delta = 0 app}.

\paragraph{Upper bound:} 
First, the throw-away estimator $\alg(X) = \frac{1}{\npu} \sum_{x \in \xpu} x$ is clearly $(0,0)$-semi-DP and has MSE \begin{align*}
\expec_{X \sim P^n} \left\| \alg(X) - \expec_{x \sim P}[x]\right\|^2 &= \expec_{X \sim P^n} \left\| \frac{1}{\npu} \sum_{x \in \xpu} \left(x - \expec_{x \sim P}[x]\right)\right\|^2 \\
&= \frac{1}{\npu^2} \sum_{x \in \xpu} \expec_{x \sim P}\left\|x - \expec_{x \sim P}[x]\right\|^2\\
&\leq \frac{1}{\npu}.
\end{align*}

To get the second term in the minimum in~\cref{eq: delta = 0 app}, consider the Laplace mechanism $\alg(X) = \Bar{X} + (L_1, \ldots, L_d)$, where $L_i \sim \lap(2\sqrt{d}/n\varepsilon)$ are i.i.d. mean-zero Laplace random variables. We know $\alg$ is $\varepsilon$-DP by~\cite{dwork2014}, since the $\ell_1$-sensitivity is $\sup_{X \sim X'}\| \Bar{X} - \Bar{X}'\|_1 = \frac{1}{n}\sup_{x, x'} \|x - x'\|_1 \leq \frac{2 \sqrt{d}}{n}$. Hence $\alg$ is $\varepsilon$-semi-DP. Moreover, $\alg$ has MSE \begin{align*}
\expec_{X \sim P^n} \left\|\alg(X) - \expec_{x \sim P}[x]\right\|^2 &\leq 2\expec\|\alg(X) - \Bar{X}\|^2 + 2\expec\left\|\Bar{X} - \expec_{x \sim P}[x]\right\|^2 \\
&\leq 2 d ~\var(\lap(2\sqrt{d}/n\varepsilon)) + \frac{2}{n} \\
&= \frac{16d^2}{n^2 \varepsilon^2} + \frac{2}{n}. 
\end{align*}
This completes the proof of~\cref{eq: delta = 0 app}. 

\end{proof}

\subsubsection{An ``Even More Optimal'' Semi-DP Algorithm for Mean Estimation}
\label{app: even more optimal ME}
\begin{lemma}[Re-statement of~Lemma~\ref{lem: zcdp naive}]
Recall the definition of $\PP(B, V)$ (Definition~\ref{def: P(B,V)}): $\PP(B, V)$ denotes the collection of all distributions $P$ on $\mathbb{R}^d$ such that for any $x \sim P$, we have $\|x\| \leq B$ $P$-almost surely and $\var(x) = V^2$. 
Then, the error of the $\rho$-semi-zCDP throw-away algorithm $\alg(X) = \frac{1}{\npu}\sum_{x \in \xpu} x$ is \[
\sup_{P \in \PP(B,V)} \expec_{X \sim P^n
}\left[\| \mathcal{A}(X) - \expec_{x \sim P}[x]\|^2\right] = \frac{V^2}{\npu}. 
\]
The minimax error of the $\rho$-semi-zCDP Gaussian mechanism $\mathcal{G}(X) = \Bar{X} + \mathcal{N}\left(0, \sigma^2\mathbf{I}_d \right)$ is \begin{equation}
\label{eq: gauss}
\inf_{\rho\text{-zCDP}~\mathcal{G}} \sup_{P \in \PP(B,V)} \expec_{\mathcal{G}, X \sim P^n
}\left[\| \mathcal{G}(X) - \expec_{x \sim P}[x]\|^2\right] = \frac{2 d B^2}{\rho n^2} + \frac{V^2}{n}. 
\end{equation}
\end{lemma}
\begin{proof}
For throw-away, the i.i.d. data assumption implies \[
\expec \|\alg(X) - \expec x\|^2 = \frac{1}{\npu^2}\sum_{x \in X_{\text{pub}}} \expec \|x - \expec x\|^2 = \frac{V^2}{\npu}.
\]
The Gaussian mechanism $\mathcal{G}^*(X) := \Bar{X} + \mathcal{N}\left(0, \frac{2 B^2}{\rho n^2}\mathbf{I}_d \right)$ is $\rho$-zCDP by~\citep[Proposition 1.6]{bun16} since the $\ell_2$-sensitivity is bounded by $\Delta_2 = \sup_{X \sim X'}\|\Bar{X} - \Bar{X}'\|_2 \leq \frac{2B}{n}$. Moreover, this sensitivity bound is tight: consider any $P$ such that $x = (B, 0_{d-1})$ and $x' = (-B, 0_{d-1})$ are in the support of $P$. Then fix any $y$ in the support of $P$ and consider the adjacent data sets $X = (x, y, \cdots, y)$ and $X' = (x', y, \cdots, y)$. We have $\|\Bar{X} - \Bar{X}'\|_2 = \frac{1}{n}\|x - x'\|_2 = \frac{2B}{n}$. Additionally, if the variance of the additive isotropic Gaussian noise $\sigma^2$ is smaller than $\frac{\Delta_2^2}{2 \rho} = \frac{2 B^2}{\rho n^2}$, then the Gaussian mechanism is not $\rho$-zCDP~\cite{bun16}. Thus, $\mathcal{G}^*(X)$ is the $\rho$-zCDP Gaussian mechanism with the smallest noise variance $\sigma^2$. Hence the infimum in~\cref{eq: gauss} is attained by $\mathcal{G}^*$. Finally, for any $P \in\PP(B,V)$, the MSE of $\mathcal{G}^*$ is
\begin{align*}
\expec_{\mathcal{G}, X \sim P^n}\left[\| \mathcal{G}^*(X) - \expec_{x \sim P}[x]\|^2\right] &= \expec\|\mathcal{G}^*(X) - \Bar{X}\|^2 + \expec\left\|\frac{1}{n}\sum_{i=1}^n x_i - \expec x_i\right\|^2 \\
&= \frac{2 d B^2}{\rho n^2} + \frac{V^2}{n}.
\end{align*}
\end{proof}

\begin{proposition}[Re-statement of~Proposition~\ref{prop: even more optimal ME}]
Recall the definition of $\alg_r$ from~\cref{eq: weighted gauss}. 
$\alg_r$ is $\rho$-semi-zCDP. Also, there exists $r > 0$ such that
\begin{equation}
\label{eq: always better app}
\sup_{P \in \PP(B, V)} \expec_{X \sim P^n}\left[\| \alg_r(X) - \expec_{x \sim P}[x] \|^2\right] < \min\left(\frac{V^2}{\npu},  \frac{2 d B^2}{\rho n^2} + \frac{V^2}{n} \right).
\end{equation}

Further, if $\frac{V^2}{\npu} \leq \frac{2 d B^2}{\rho n^2}$, then the quantitative advantage of $\alg_r$ is \begin{equation}
\label{eq: case I app}
\sup_{P \in \PP(B, V)} \expec_{X \sim P^n}\left[\| \alg_r(X) - \expec_{x \sim P}[x] \|^2\right] \leq \left(\frac{q}{q + s^2} \right) \min\left(\frac{V^2}{\npu},  \frac{2 d B^2}{\rho n^2} + \frac{V^2}{n} \right),
\end{equation}
where $q = 2 + \frac{\npr \rho V^2}{d B^2}$ and $s = \frac{V \npr \sqrt{\rho}}{B \sqrt{d \npu}}$. 
\end{proposition}
\begin{proof}
\textbf{Privacy:} 
Note that the $\ell_2$-sensitivity of $M(X) = \sum_{x \in \xpr} rx + \sum_{x \in \xpu} \left(\frac{1 - \npr r}{\npu}\right)x$ is \begin{align*}
\Delta_2 &= \sup_{\xpr \sim \xpr'} \left\| \sum_{x \in \xpr} rx + \sum_{x \in \xpu} \left(\frac{1 - \npr r}{\npu}\right)x - \sum_{x \in \xpr'} rx - \sum_{x \in \xpu} \left(\frac{1 - \npr r}{\npu}\right)x \right\| \\
& = \sup_{x, x'}\|rx - rx'\| \leq 2rB.\end{align*}
Recall that the Gaussian mechanism guarantees $\rho$-zCDP whenever $\sigma^2 \geq \frac{\Delta_2^2}{2\rho}$~\citep[Proposition 1.6]{bun16}. Thus, $\alg_r$ is $\rho$-semi-zCDP for $\sigma_r^2 \geq \frac{(2rB)^2}{2 \rho} = \frac{2 B^2 r^2}{\rho}$.

\noindent \textbf{Error bounds:} Let $P \in \PP(B, V)$. We have \begin{align}
\label{eq: error of weighted gauss}
\expec_{X \sim P^n}\left[\| \alg_r(X) - \expec_{x \sim P}[x] \|^2\right] &= \frac{2 d B^2 r^2}{\rho} + \expec\left\| \sum_{x \in \xpr} r (x - \expec x) + \sum_{x \in \xpu} \left(\frac{1 - \npr r}{\npu} \right) (x - \expec x) \right\|^2 \nonumber \\
&= \frac{2 d B^2 r^2}{\rho} + \sum_{x \in \xpr} r^2 V^2 + \sum_{x \in \xpu}\left(\frac{1 - \npr r}{\npu} \right)^2 V^2 \nonumber \\
&= \frac{2 d B^2 r^2}{\rho} + \npr r^2 V^2 + \frac{(1 - \npr r)^2}{\npu} V^2,
\end{align} 
using independence of the Gaussian noise and the data, basic properties of variance, and the fact that the data is i.i.d.  

To prove~\cref{eq: always better app}, let \[
J(r) := \frac{2 d B^2 r^2}{\rho} + \npr r^2 V^2 + \frac{(1 - \npr r)^2}{\npu} V^2.
\]
We compute first and second derivatives of $J$: \[
\frac{d}{dr}J(r) = 2r\left(\frac{2 d B^2}{\rho} + \npr V^2\right) -2 \npr\frac{V^2}{\npu}(1 - \npr r)
\]
and \[
\frac{d^2}{dr^2}J(r) = 2\left(\frac{2 d B^2}{\rho} + \npr V^2\right) + 2 \npr^2 \frac{V^2}{\npu}. 
\]
Since $J$ is strongly convex, it has a unique minimizer $r^*$ which satisfies $\frac{d}{dr}J(r^*) = 0$. We find \[
r^* = \frac{\npr V^2}{\npu} \left(\frac{2 d B^2}{\rho} + \npr V^2 + \frac{\npr^2 V^2}{\npu}\right)^{-1}.
\]
One can verify that $r^* \neq 0$ and $r^* \neq \frac{1}{n}$, since $1 \leq \npr < n$ and $\rho < \infty$ by assumption. Thus, $J(r^*) < \min(J(0), J(1/n))$, which yields~\cref{eq: always better app} by~Lemma~\ref{lem: zcdp naive}. 

To prove~\cref{eq: case I app}, we will choose a different $r$: $r := \frac{ K \sqrt{\rho} V}{B \sqrt{d \npu}}$ for $K > 0$ to be determined. Then by~\cref{eq: error of weighted gauss}, we have 
\begin{align*}
\expec_{X \sim P^n}\left[\| \alg_r(X) - \expec_{x \sim P}[x] \|^2\right] &= \frac{2 d B^2 r^2}{\rho} + \npr r^2 V^2 + \frac{(1 - \npr r)^2}{\npu} V^2 \\
&= K^2 \frac{V^2}{\npu}\left(2 + \frac{\npr \rho V^2}{d B^2} \right) + \frac{V^2}{\npu}\left(1 - \frac{K V}{B} \frac{\sqrt{\rho} \npr}{\sqrt{d \npu}}\right)^2 \\
&= \frac{V^2}{\npu} \left( q K^2 + \left(1 - \frac{K V}{B} \frac{\sqrt{\rho} \npr}{\sqrt{d \npu}}\right)^2 \right),
\end{align*}
where $q = 2 + \frac{\npr \rho V^2}{d B^2}$.
Now, letting $s = \frac{V \npr \sqrt{\rho}}{B \sqrt{d \npu}}$ and choosing $K = \frac{s}{q + s^2}$ gives
\begin{align*}
\expec_{X \sim P^n}\left[\| \alg_r(X) - \expec_{x \sim P}[x] \|^2\right] &\leq \frac{V^2}{\npu} \left( q K^2 + \left(1 - \frac{K V}{B} \frac{\sqrt{\rho} \npr}{\sqrt{d \npu}}\right)^2 \right) \\
& \leq \frac{V^2}{\npu} \left(\frac{q^2 + q s^2}{q^2 + 2q s^2 + s^4} \right).
\end{align*}
Finally, the assumption $\frac{V^2}{\npu} \leq \frac{2dB^2}{\rho n^2}$ implies $\frac{V^2}{\npu} = \min\left(\frac{V^2}{\npu},  \frac{2dB^2}{\rho n^2} + \frac{V^2}{n} \right)$, completing the proof. 

\end{proof}

\subsection{Optimal Semi-DP Empirical Risk Minimization}
\label{app: ERM}

\paragraph{Practical Applications of Semi-DP ERM Beyond ML:}
\label{app: practical applications of ERM}
Semi-DP ERM has numerous applications beyond training ML models. For example, consider semi-DP optimization of energy consumption in smart grids or semi-DP optimization of the total capacity of a multi-user wireless communication system. In these systems, the goal is to optimize the current performance of the system given existing users (e.g., optimize current beamforming strategies in wireless communications). Some users may opt-in to share their data (e.g. electricity consumption pattern) and some users may not. Thus, the problem is naturally a semi-DP ERM problem. 

\begin{theorem}[Complete statement of~\cref{thm: DP convex ERM}]
There exist absolute constants $c_0$ and $C_0$, with $0 < c_0 \leq C_0$, such that
\begin{equation*}
c_0 LD\min\left\{ \frac{\npr}{n}, \pen \right\} \leq \ERMrisk \leq C_0 LD  \min\left\{ \frac{\npr}{n}, \pen \right\}.
\end{equation*}
Further, if $\mu > 0$, then there exist absolute constants $0 < c_1 \leq C_1$ such that \begin{equation*}
    c_1 LD\min\left\{ \frac{\npr}{n}, \pen \right\}^2 \leq \ERMriskmu \leq C_1 \frac{L^2}{\mu} \min\left\{ \frac{\npr}{n}, \frac{d \sqrt{\ln(n)}}{n \eps} \right\}^2.
\end{equation*}
\end{theorem}
\begin{proof}
\textbf{Lower Bounds:} Given a lower bound for empirical mean estimation, Bassily et al.~\cite{bst14} show how to prove excess risk lower bounds for convex and strongly convex ERM by reducing these problems to mean estimation. Thus, our lower bounds follow immediately by combining the lower bound in~\cref{thm: pure DP empirical mean estimation} with the reduction in~\cite{bst14}. Roughly, the reduction works as follows:

In the strongly convex case, we simply take $f(w,x) = \frac{1}{2}\|w - x\|^2$  on $\WW = \XX = \mathbb{B}$; $f(\cdot, x)$ is $1$-uniformly-Lipschitz and $1$-strongly convex. Moreover, for any $\epsemi$ $\alg$ with output $\wpr = \alg(X)$, we have \[
\expec \hf(\wpr) - \hf^* = \frac{1}{2} \expec \|\alg(X) - \Bar{X}\|^2.
\]
Applying~\cref{thm: pure DP empirical mean estimation} and then scaling $f \to \frac{L}{D} f$ and $\WW \to D \WW$ and $\XX \to D \XX$ completes the proof.

For the convex case, we take $f(w,x) = - \langle w, x \rangle$ on $\WW = \XX = \mathbb{B}$. Then $\ws := \argmin_{w \in \WW} \hf(w) = \frac{\Bar{X}}{\|\Bar{X}\|}$ and \[
\hf(\wpr) - \hf^* \geq \frac{1}{2}\|\Bar{X}\|\|\wpr - \ws\|^2. 
\]
Also, the proof of~\cref{thm: pure DP empirical mean estimation} shows that there exists a dataset $X \in \XX^n$ such that $\|\Bar{X}\| = M/n := \min\left(\frac{\npr}{n}, \frac{d}{3n\eps} \right)$ and $\expec\|\alg'(X) - \Bar{X}\|^2 \gtrsim \min\left(\frac{\npr}{n}, \frac{d}{n \eps} \right)^2$ for any $\epsemi$ $\alg'$. Note that $\alg' := \frac{M}{n} \wpr$ is $\epsemi$ by post-processing. Thus, \begin{align*}
\expec \hf(\wpr) - \hf^* &\geq \frac{M}{2n} \expec \left[ \|\wpr - \ws\|^2 \right] \\
&= \frac{M}{2n} \left(\frac{n}{M}\right)^2 \expec \left[ \|\alg'(X) - \Bar{X}\|^2 \right] \\
&\gtrsim \frac{n}{M} \min\left(\frac{\npr}{n}, \frac{d}{n \eps} \right)^2\\
&\gtrsim \min\left(\frac{\npr}{n}, \frac{d}{n \eps} \right). 
\end{align*}

A standard scaling argument (see~\cite{bst14} for details) completes the proof. 

\paragraph{Upper Bounds:} The second terms in each minimum follows by running the $\eps$-DP algorithms in~\cite{bst14}: these achieve the desired excess empirical risk bounds and are automatically $\epsemi$.

We now prove the first term in each respective minimum. Denote $\hfpu(w) = \frac{1}{n} \sum_{x \in X_{pub}} f(w,x)$ and $\hfpr(w) = \frac{1}{n} \sum_{x \in X_{priv}} f(w,x)$, so that $\hf = \hfpu + \hfpr$. The algorithm we will use simply returns any minimizer of the public empirical loss: $\alg(X) = \wpu \in \argmin_{w \in \WW} \hfpu(w)$. (It will be easy to see from the proof that any approximate minimizer would also suffice.) $\alg$ is clearly $\epsemi$. Next, we bound the excess risk of $\alg$. Let $\ws \in \argmin_{w \in \WW} \hf(w)$.

\paragraph{Convex Upper Bound:} We have \begin{align*}
    \hf(\wpu) - \hf(\ws) &= \hf(\wpu) - \hfpu(\wpu) + \hfpu(\wpu) - \hfpu(\ws) \\
    &\;\;\; + \hfpu(\ws) - \hf(\ws) \\
    &\leq \frac{1}{n} \sum_{x \in X_{priv}} f(\wpu, x)  + 0 - \frac{1}{n}\sum_{x \in X_{priv}} f(\ws, x) \\
    &\leq \frac{1}{n} \sum_{x \in X_{priv}} L \| \wpu - \ws\| \\
    &= L \| \wpu - \ws\| \frac{\npr}{n}\\
    &\leq LD \frac{\npr}{n}.
\end{align*}

\paragraph{Strongly Convex Upper Bound:}
By the above, $\hf(\wpu) - \hf(\ws) \leq L \| \wpu - \ws\| \frac{\npr}{n}$. Now we will use strong convexity to bound $\| \wpu - \ws\|$. To do so, we use the following lemma, versions of which have appeared, e.g. in~\cite{lowy2021, chaud}:
\begin{lemma}{\cite{lowy2021}}
\label{lemmaB2outpert}
Let $H(w), h(w)$ be convex functions on some convex closed set $\mathcal{W} \subseteq \mathbb{R}^d$ and suppose that $H(w)$ is $\mu_{H}$-strongly convex. Assume further that $h$ is $L_{h}$-Lipschitz. Define $w_{1} = \arg\min_{w \in \mathcal{W}} H(w)$ and $w_{2} = \arg\min_{w \in \mathcal{W}} [H(w) + h(w)]$. Then $\|w_{1} - w_{2}\| \leq \frac{L_{h}}{\mu_H}.$ 
\end{lemma}
We apply the lemma with $h(w) = \hfpr(w)$ and $H(w) = \hfpu(w)$. Then the conditions of the lemma are satisfied with $L_h \leq \frac{\npr}{n} L$ and $\mu_H = \frac{\npu}{n} \mu$. Thus, \[
\| \ws - \wpu \| \leq \frac{L_h}{\mu_H} \leq \frac{L}{\mu} \frac{\npr}{\npu} 
\]
This leads to \[
\hf(\wpu) - \hf(\ws) \leq \frac{L^2}{\mu} \frac{\npr^2}{n \npu}.
\]
Combining the two strongly convex upper bounds with the upper bound $LD \frac{\npr}{n} \leq \frac{L^2}{\mu} \frac{\npr}{n}$ (which holds for any convex function), we have an algorithm $\alg$ with the following excess risk: 
\begin{equation}
\label{eq:a}
\ERlower \lesssim \frac{L^2}{\mu}\min\left(\frac{\npr}{n}, \frac{\npr^2}{n \npu}, \frac{d^2 \ln(n)}{n^2 \eps^2}\right). 
\end{equation}
We will show that \cref{eq:a} is equal to the strongly convex upper bound stated in~\cref{thm: DP convex ERM} up to constant factors. 
First, suppose $\npu \gtrsim n$: i.e. there is a constant $k > 0$ such that $\npu \geq k n$ for all $n \geq 1$. Then, clearly~\cref{eq:a} and the strongly convex upper bound stated in~\cref{thm: DP convex ERM} are both equal to $\Theta\left(\frac{L^2}{\mu}\min\left(\frac{\npr^2}{n^2}, \left(\pen\right)^2 \ln(n) \right) \right)$. 

Next, suppose $\npu \ll n$: i.e. for any $k > 0$, there exists $n \geq 1$ such that $\npu < k n$. Then we claim that $\min\left\{ \frac{\npr}{n}, \frac{d \sqrt{\ln(n)}}{n \eps} \right\}^2 \gtrsim \min\left\{1, \frac{d \sqrt{\ln(n)}}{n \eps} \right\}^2$. If we prove this claim, then we are done. There are two subcases to consider: A) $\npr/n < \frac{d \sqrt{\ln(n)}}{\eps n}$; and B) $\npr/n \geq \frac{d \sqrt{\ln(n)}}{\eps n}$. In subcase B), the claim is immediate. Consider subcase A): if $\npr \gtrsim n$, then we're done. If not, then we have $\npr \ll n$ and $\npu \ll n$, so $n = \npr + \npu \ll n$, a contradiction. This completes the proof. 
\end{proof}

\begin{remark}[Details of Remark~\ref{rem: nonconvex ERM}]
The same minimax risk bound~\cref{eq: convex ERM excess risk} holds up to a logarithmic factor if we replace $\FFmuz$ by the larger class of all Lipschitz \textit{non-convex} (or convex) loss functions in the definition~\cref{eq:DP minimax excess empirical risk}: First, the lower bound in~\cref{thm: DP convex ERM} clearly still holds for non-convex loss functions. For the upper bound, the $\eps$-DP (hence semi-DP) exponential mechanism achieves error $O(LD \pen \ln(n))$~\cite{bst14,gtu22}. Further, the proof of~\cref{thm: DP convex ERM} reveals that convexity is not necessary for the throw-away algorithm to achieve error~$O(LD \npr/n)$. However, the optimal algorithms are inefficient for non-convex loss functions: to the best of our knowledge, all existing polynomial time implementations of the exponential mechanism require convexity for their runtime guarantees to hold. Further, computing $\approx \argmin_{w \in \WW} \widehat{F}_{\text{pub}}(w)$ in the implementation of throw-away may not be tractable in polynomial time for non-convex $\widehat{F}_{\text{pub}}$. 
\end{remark}

\subsection{Optimal Semi-DP Stochastic Convex Optimization}
\label{app: DP SCO}

\paragraph{Approximate $(\eps, \delta)$-Semi-DP SCO}
\begin{theorem}[Complete Version of~\cref{thm: approx DP SCO}]
\label{thm: approx DP SCO app}
Let $\eps \lesssim 1/\log(nd)$ and $\delta \ll 1/n$. Then, there is a constant $C>0$ such that
{\small
\begin{equation*}
\ell(d,n) LD\min\left\{\frac{1}{\sqrt{\npu}}, \frac{\sqrt{d}}{n \eps} + \frac{1}{\sqrt{n}}\right\} \leq \SCOrisk \leq C LD\min\left\{ \frac{1}{\sqrt{\npu}}, \frac{\sqrt{d \ln(1/\delta)}}{n \eps} + \frac{1}{\sqrt{n}}\right\},
\end{equation*}
}
and 
{\small
\begin{equation*}
\ell(d,n) LD\min\left\{\frac{1}{\sqrt{\npu}}, \frac{\sqrt{d}}{n \eps} + \frac{1}{\sqrt{n}}\right\}^2 \leq \SCOriskmu \leq C \frac{L^2}{\mu}\min\left\{ \frac{1}{\sqrt{\npu}}, \frac{\sqrt{d \ln(1/\delta)}}{n \eps} + \frac{1}{\sqrt{n}}\right\}^2,
\end{equation*}
}
where $1/\ell(d,n)$ is logarithmic in $d$ and $n$. Our lower bounds hold for symmetric $\alg = (\alg^1, \ldots, \alg^d)$. 
\normalsize
\end{theorem}
\begin{proof}
\textbf{Lower bounds:}
Let $\alg$ be $(\eps, \delta)$-semi-DP and symmetric, and denote $\wpr = \alg(X)$. 

\noindent \ul{Strongly convex lower bounds:} We begin with the strongly convex lower bounds, which can be proved by reducing strongly convex SCO to mean estimation and applying~\cref{thm: CDP mean estimation}. In a bit more detail, let $f: \WW \times \XX \to \mathbb{R}$ be given by \[
f(w,x) = \frac{L}{2D}\|w - x\|^2,
\]
where $\WW = \XX = D \mathbb{B}$. 
Note that $f(\cdot, x)$ is $L$-uniformly Lipschitz and $\frac{L}{D}$-strongly convex in $w$ for all $x$. Further, $\ws := \argmin_{w \in \WW} \left\{F(w) = \expec_{x \sim P}\left[f(w,x)\right]\right\} = \expec_{x \sim P}[x]$. By a direct calculation (see e.g. \citep[Lemma 6.2]{klz21}), we have \begin{equation}
\label{eq: sc reduction to ME1}
\expec F(\wpr) - F(\ws) = \frac{L}{2D} \expec \|\wpr - \ws\|^2.
\end{equation}
We can lower bound $\expec \|\wpr - \ws\|^2 = \expec \| \alg(X) - \expec_{x \sim P}[x]\|^2$ via~\cref{thm: CDP mean estimation} (and its proof, to account for the re-scaling). Specifically, \[
\expec \|\wpr - \ws\|^2 \geq \ell(d,n) D^2 \min\left\{\frac{1}{\npu}, \frac{d}{n^2 \eps^2} + \frac{1}{n}\right\},
\]
for a logarithmic function $\ell(d,n)$ of $d$ and $n$, by \cref{thm: CDP mean estimation}. 
Applying~\cref{eq: sc reduction to ME1} yields the desired excess risk lower bound for $\delta > 0$.

\noindent \ul{Convex lower bounds:}
We will begin by proving the lower bounds for the case in which $L = D = 1$, and then scale our construction to get the lower bounds for arbitrary $L, D$. 

Let $\XX = \left\{\pm \frac{1}{\sqrt{d}}\right\}^d \subset \mathbb{R}^d$ and $\WW = \mathbb{B}$. Define \[
f(w,x) = - \langle w, x \rangle,
\]
which is convex and $1$-uniformly-Lipschitz in $w$ on $\XX$. Let $\ptheta$ be the hard distribution used to prove~\cref{thm: approx DP mean estimation with small theta}, which satisfies $\expec_{x \sim \ptheta}[x] = \theta \in [-a/\sqrt{d}, a/\sqrt{d}]^d$ and $\|\theta\| \leq a = \min\left(\frac{1}{\sqrt{\npu}}, \frac{\sqrt{d}}{\npr \eps} + \frac{1}{n}\right)$. 
Further, $w_{\theta}^* = \argmin_{w \in \WW} [F_{\theta}(w) := \expec_{x \sim \ptheta} f(w,x) = - \langle w, \theta \rangle] = \frac{\theta}{\|\theta\|}$. 
A direct calculation (see e.g. \citep[Equation 14]{klz21})
shows \begin{align}
\sup_{\theta} \expec F_\theta(\wpr) - F_\theta^* &\geq \frac{1}{2} \expec\left[\|\theta\| \|\wpr - \ws_\theta\|^2 \right] \\
&= \frac{1}{2} \sup_{\theta} \expec\left[\frac{1}{\|\theta\|}\left\|\tilde{w}_{\text{priv}} - \theta \right\|^2 \right],
\end{align}
where $\tilde{w}_{\text{priv}} = \widetilde{\alg}_\theta(X)$ is the output of the algorithm $\widetilde{\alg}_\theta: \XX^n \to \|\theta\| \WW$ defined by $\widetilde{\alg}_\theta(X) := \| \theta \| \alg(X)$. Note that $\widetilde{\alg}_\theta$ is $(\eps, \delta)$-semi-DP by post-processing, for any $\theta$. 
Now, we invoke~\cref{thm: approx DP mean estimation with small theta} to obtain \[
\sup_{\theta} \expec\left[\left\|\tilde{w}_{\text{priv}} - \theta \right\|^2 \right] \geq \widetilde{\Omega}\left(\min\left\{\frac{1}{\npu}, \frac{d}{\npr^2 \eps^2} + \frac{1}{n}\right\}\right)
\]
for $\|\theta\| \leq \min\left(\frac{1}{\sqrt{\npu}}, \frac{\sqrt{d}}{\npr \eps} + \frac{1}{n}\right)$. This implies the desired lower bound when $L=D=1$. 

For general $L$ and $D$, we scale the problem instance as follows: let $\wt{W} = D \WW$, $\wt{X} = L \XX$, and $\wt{x} \sim \wt{\ptheta} \iff \wt{x} = Lx$ for $x \sim \ptheta$. Define $\wt{f}: \wt{\WW} \times \wt{\XX} \to \mathbb{R}$ by $\wt{f}(\wt{w}, \wt{x}) := f(\wt{w}, \wt{x}) = - \langle \wt{w}, \wt{x} \rangle$. Then $\wt{f}(\cdot, \wt{x})$ is $L$-Lipschitz and convex. Moreover, if $F(w) = \expec_{x \sim \ptheta}[f(w,x)]$, $\wt{F}(\wt{w}) = \expec_{\wt{x} \sim \wt{P}}[f(\wt{w}, \wt{x})]$,  $\wt{w} = Dw$, and $\wt{\theta} = \expec_{\wt{x} \sim \wt{\ptheta}}[\wt{x}] = L \theta$,
then $D \ws \in \argmin_{\wt{w} \in \wt{\WW}} \wt{F}(\wt{w})$ and 
\begin{align*}
\wt{F}(\wt{w}) - \wt{F}^* &= -\langle \wt{w}, \wt{\theta}\rangle 
+ \langle\wt{w}^*, \wt{\theta} \rangle  \\
&= D \langle  \wt{\theta}, \ws - w \rangle \\
&=  LD \langle \theta, \ws - w \rangle \\
&= LD \left[F(w) - F^* \right].
\end{align*}
This shows that excess risk scales by $LD$, completing the lower bound proofs.

\noindent \textbf{Upper bounds:}
\ul{Convex upper bounds}: Consider the $0$-semi-DP throw-away algorithm that discards $\xpr$ and runs $\npr$ steps of one-pass SGD (stochastic approximation) using $\xpu$. This algorithm has excess risk $O\left(LD/\sqrt{\npu}\right)$~\cite{ny}. To obtain the second term in the convex $(\eps, \delta)$-semi-DP upper bound, one can use, e.g. $(\eps, \delta)$-DP-SGD~\cite{bft19}.

\noindent\ul{Strongly convex upper bounds:} Consider the $0$-semi-DP throw-away algorithm that discards $\xpr$ and runs $\npr$ steps of one-pass SGD (stochastic approximation) using $\xpu$. This algorithm has excess risk $O\left(\frac{L^2}{\mu \npu}\right)$~\cite{ny}. The second term in the strongly convex $(\eps, \delta)$-semi-DP upper bound can be attained, e.g. by $(\eps, \delta)$-DP-SGD~\cite{lr21fl}. 
\end{proof}

Next, we provide minimax optimal excess risk bounds for pure $\eps$-semi-DP SCO. 
\paragraph{Pure $\eps$-Semi-DP SCO}
\begin{theorem}[Pure $\eps$-Semi-DP SCO]
\label{thm: pure sdp sco}
Suppose $\eps \leq d/8$, and either $\npu \lesssim \frac{n\eps}{d}$ or $d \lesssim 1$. 
If $\mu = 0$ (convex case), then there exist absolute constants $0 < c_0 \leq C_0$ such that
\small
\begin{equation*}
c_0 LD\min\left\{\frac{1}{\sqrt{\npu}}, \frac{d}{n \eps} + \frac{1}{\sqrt{n}}\right\} \leq \SCOriskdelz \leq C_0 LD\min\left\{ \frac{1}{\sqrt{\npu}}, \frac{d}{n \eps} + \frac{1}{\sqrt{n}}\right\}.
\end{equation*}
\normalsize
If $\mu > 0$, there are constants $0 < c_1 \leq C_1$ such that 
\small
\begin{equation*}
c_1 LD\min\left\{\frac{1}{\npu}, \frac{d^2}{n^2 \eps^2} + \frac{1}{n}\right\} \leq \SCOriskdelzmu \leq C_1 \frac{L^2}{\mu}\min\left\{\frac{1}{\npu}, \frac{d^2 \ln(n)}{n^2 \eps^2} + \frac{1}{n}\right\}.
\end{equation*}
\normalsize
The above upper bounds hold for any $\npu, d$. 
\end{theorem}

\begin{proof}[Proof of~\cref{thm: pure sdp sco}]
\textbf{Lower bounds:}
Let $\alg$ be $\eps$-semi-DP and denote $\wpr = \alg(X)$. 

\noindent \ul{Strongly convex lower bound:} We begin with the strongly convex lower bound, which can be proved by reducing strongly convex SCO to mean estimation and applying~\cref{thm: pure semi-DP mean estimation}. In a bit more detail, let $f: \WW \times \XX \to \mathbb{R}$ be given by \[
f(w,x) = \frac{L}{2D}\|w - x\|^2,
\]
where $\WW = \XX = D \mathbb{B}$. 
Note that $f$ is $L$-uniformly Lipschitz and $\frac{L}{D}$-strongly convex in $w$ for all $x$. Further, $\ws := \argmin_{w \in \WW} \left\{F(w) = \expec_{x \sim P}\left[f(w,x)\right]\right\} = \expec_{x \sim P}[x]$. By a direct calculation (see e.g. \citep[Lemma 6.2]{klz21}), we have \begin{equation}
\label{eq: sc reduction to ME}
\expec F(\wpr) - F(\ws) = \frac{L}{2D} \expec \|\wpr - \ws\|^2.
\end{equation}
We can lower bound $\expec \|\wpr - \ws\|^2 = \expec \| \alg(X) - \expec_{x \sim P}[x]\|^2$ via~\cref{thm: pure semi-DP mean estimation} (and its proof, to account for the re-scaling). Specifically, if $\delta = 0$, $\eps \leq \max(1, d/8)$, and either $d = O(1)$ or $\npu = O(n \eps/d)$, then \cref{thm: pure semi-DP mean estimation} and its proof imply \[
\expec \|\wpr - \ws\|^2 \geq c D^2 \min\left(\frac{1}{\npu}, \frac{d^2}{n^2 \eps^2} + \frac{1}{n} \right).
\]
Combining this with~\cref{eq: sc reduction to ME} leads to the desired excess risk lower bound for $\delta = 0$.

\noindent \ul{Convex lower bound:}
We will begin by proving the lower bound for the case in which $L = D = 1$, and then scale our construction to get the lower bounds for arbitrary $L, D$. 

Assume $\delta = 0$, $\eps \leq d/8$, and either $\npu \lesssim n\eps/d$ or $d \lesssim 1$. 
Let $\XX = \{0\} \bigcup \left\{\pm \frac{1}{\sqrt{d}}\right\}^d \subset \mathbb{R}^d$ and $\WW = \mathbb{B}$. Define \[
f(w,x) = - \langle w, x \rangle,
\]
which is convex and $1$-uniformly-Lipschitz in $w$ on $\XX$. Choose $\VV$ to be a finite subset of $\mathbb{R}^d$ such that $|\VV| \geq 2^{d/2}$, $\|v\| = 1$ for all $v$, and $\|v - v'\| \geq 1/8$ for all $v \neq v'$ (see e.g. the Gilbert-Varshamov construction). Following the proof of~\cref{thm: CDP mean estimation}, we define $\ptv = (1-p) P_0 + p P_v$ for all $v \in \VV$, where $p \in [0, 1]$ will be chosen later, $P_0$ is point mass on $\{X = 0\}$ and $P_v$ is point mass on $\{X = v\}$. Denote the mean $\theta_v := \expec_{x \sim \ptv}[x] = pv$. Note that $\|\theta_v\| = p$ for all $v$. Let $F_v(w) = \expec_{x \sim \ptv}[f(w,x)]$ and $\ws_v \in \argmin_{w \in \WW} F_v(w) = \frac{\theta_v}{\|\theta_v\|} = v$. 
A direct calculation (see e.g. \citep[Equation 14]{klz21})
shows \begin{equation}
\label{eq:b}
\expec F_v(w) - F_v^* \geq \frac{1}{2} \expec\left[\|\theta_v\| \|w - \ws_v\|^2 \right]
\end{equation}
for any $w \in \WW, v \in \VV$. Also, \begin{align*}
\rho^*(\VV) = \min\left\{\|\ws_v - \ws_{v'}\| : v, v' \in \VV, v\neq v' \right\} = \min\left\{\|v - v'\| : v, v' \in \VV, v\neq v' \right\} \geq \frac{1}{8}.
\end{align*}
Thus, by combining~\cref{eq:b} with the reduction from estimation to testing and \cref{lem: semiDP Fano} (see the proof of~\cref{thm: CDP mean estimation} for details), we have \begin{align*}
\sup_{v \in \VV} \expec\left[F_v(\wpr) - F_v^* \right] &\geq \frac{1}{2} \sup_{v\in \VV} \expec\left[\|\theta_v\| \|\wpr - \ws_v\|^2 \right]\\
&= \frac{p}{2} \sup_{v \in \VV} \expec\left[\|\wpr - \ws_v\|^2 \right] \\
&\geq \frac{p}{2} \rho^{*}(\VV)^2 \frac{1}{|\VV|}\sum_{v \in \VV} \ptv\left(\|\thetahat(X) - \theta_v\| \geq \rho^*(\VV) \right) \\
&\geq \frac{p}{128} \frac{(|\VV| - 1) e^{-\eps \lceil \npr p \rceil} (1 - p)^{\npu}}{2 \left(1 + (|\VV| - 1)  e^{-\eps \lceil \npr p \rceil} \right)} \\
&\geq \frac{p}{128} \frac{(2^{d/2} - 1) e^{-\eps \lceil \npr p \rceil} (1 - p)^{\npu}}{2 \left(1 + (2^{d/2} - 1)  e^{-\eps \lceil \npr p \rceil} \right)} \\
&\geq \frac{p}{512} (1-p)^{\npu} \min\left\{1, \frac{2^{d/2} - 1}{e^{\eps(\npr p + 1)}} \right\}.
\end{align*}
Now, assume $d \geq 4$ so that $2^{d/2} - 1 \geq e^{d/4}$. Then, as detailed in the proof of~\cref{thm: CDP mean estimation}, choosing \[
p = \min\left(\frac{d}{4n \eps} - \frac{1}{n}, \frac{1}{2\sqrt{\npu}}\right)
\]
and assuming $\npu \leq k n \eps/d$ for some absolute constant $k$ implies \[
\sup_{v \in \VV} \expec\left[F_v(\wpr) - F_v^* \right] \geq c \min\left(\frac{1}{\sqrt{\npu}}, \frac{d}{n \eps}\right)
\]
for some absolute constant $c > 0$. Combining this with the non-private SCO lower bound~\cite{ny} yields  \[
\sup_{P} \expec\left[F(\wpr) - F^* \right] \geq c' \min\left(\frac{1}{\sqrt{\npu}}, \frac{d}{n \eps} + \frac{1}{\sqrt{n}}\right),
\]
where $F(w) := \expec_{x \sim P} [f(w,x)]$. 

Suppose instead that $0 \leq \delta \leq \eps$ and $d \lesssim 1$ (i.e. $d \leq k$ for some absolute constant $k \geq 1$), but $\npu \in [n]$ is arbitrary. We will prove the lower bound for $d=1$; by taking the $k$-fold product distribution, this is sufficient to complete the proof of the unscaled $\eps$-semi-DP lower bound. Define distributions $P_1, P_2$ on $\{-1, 1\}$ as follows: \[
P_1(-1) = P_2(1) = \frac{1 + \gamma}{2}, ~~ P_1(1) = P_2(-1) = \frac{1- \gamma}{2}
\]
for some $\gamma \in (0, 1/2]$ to be chosen later. Note $\theta_1 := \expec_{P_1}[x] = - \gamma$ and $\theta_2 := \expec_{P_2}[x] = \gamma$, so $|\theta_j| = \gamma$ for $j = 1, ~2$. 
Let $F_j(w) = \expec_{x \sim P_j} f(w,x)$ and $\ws_j = \frac{\theta_j}{|\theta_j|} = \frac{\theta_j}{\gamma} \in \argmin_{w \in \WW} F_j(w)$. 
Then by~\cref{eq:b}, we have 
\begin{align*}
\max_{j \in \{1,2\}} \expec\left[F_j(\wpr) - F_j^* \right] &\geq \frac{1}{2} \max_{j \in \{1,2\}} \expec\left[|\theta_j| |\wpr - \ws_j|^2 \right]\\
&= \frac{\gamma}{2} \max_{j \in \{1,2\}} \expec\left[|\wpr - \ws_j|^2 \right] \\
&= \frac{1}{2 \gamma} \max_{j \in \{1,2\}} \expec\left[|\wpr' - \theta_j|^2 \right],
\end{align*}
where $\wpr' := \gamma \wpr$ is semi-DP iff $\wpr$ is semi-DP (by post-processing). Thus, by applying Le Cam's method and Lemma~\ref{lem: alireza lem3} (see the proof of~\cref{thm: CDP mean estimation} for details), we get \begin{align*}
    \max_{j \in \{1,2\}} \expec\left[F_j(\wpr) - F_j^* \right] &\geq \frac{1}{2 \gamma}\left[\frac{\gamma^2}{8}\left(1 - \gamma \min\left(\sqrt{\frac{3n}{2}}, 6 \npr \eps + \sqrt{\frac{3 \npu}{2}} \right) \right) \right]. 
\end{align*}
Now we will choose $\gamma$ to (approximately) maximize the right-hand side of the above inequality. 
If $\min\left\{\sqrt{\frac{3n}{2}}, 6 \npr \eps + \sqrt{\frac{3\npu}{2}}\right\} = \sqrt{\frac{3n}{2}}$, then choosing $\gamma = \frac{1}{3} \sqrt{\frac{2}{3n}}$ yields \[
\max_{j \in \{1,2\}} \expec\left[F_j(\wpr) - F_j^* \right]\geq \frac{k}{\sqrt{n}}
\]
for some absolute constant $k > 0$. If instead $\min\left\{\sqrt{\frac{3n}{2}}, 6 \npr \eps + \sqrt{\frac{3\npu}{2}}\right\} = 6 \npr \eps + \sqrt{\frac{3\npu}{2}}$, then we choose $\gamma = \frac{2}{3}\left(6 \npr \eps + \sqrt{\frac{3 \npu}{2}} \right)^{-1}$. This choice implies \[
\max_{j \in \{1,2\}} \expec\left[F_j(\wpr) - F_j^* \right] \geq k' \min\left(\frac{1}{\npr \eps}, \frac{1}{\sqrt{\npu}}\right)
\]
for some absolute constant $k' > 0$. Combining the pieces above with the non-private SCO lower bound~\cite{ny} yields  \[
\sup_{P} \expec\left[F(\wpr) - F^* \right] \geq c \min\left(\frac{1}{\sqrt{\npu}}, \frac{1}{n \eps} + \frac{1}{\sqrt{n}}\right),
\]
where $F(w) := \expec_{x \sim P} [f(w,x)]$.

A standard scaling argument completes the lower bound proofs (see e.g. the proof of~\cref{thm: approx DP SCO app} for details).

\paragraph{Upper bounds:}

\ul{Convex upper bounds}: Consider the $0$-semi-DP throw-away algorithm that discards $\xpr$ and runs $\npr$ steps of one-pass SGD (stochastic approximation) using $\xpu$. This algorithm has excess risk $O\left(LD/\sqrt{\npu}\right)$~\cite{ny}. To obtain the second term in the convex $\eps$-semi-DP upper bound, use the $\eps$-DP (hence semi-DP) regularized exponential mechanism of~\citet{gtu22}.

\noindent\ul{Strongly convex upper bounds:} Consider the $0$-semi-DP throw-away algorithm that discards $\xpr$ and runs $\npr$ steps of one-pass SGD (stochastic approximation) using $\xpu$. This algorithm has excess risk $O\left(\frac{L^2}{\mu \npu}\right)$~\cite{ny}. To obtain the second term in the strongly convex $\eps$-semi-DP upper bound, one can use, e.g. the $\eps$-DP (hence semi-DP) iterated exponential mechanism of~\citet{gtu22}. 
\end{proof}

\subsubsection{Semi-DP SCO with an ``Even More Optimal'' Gradient Estimator}
\label{app: even better DP SCO}

\begin{proposition}
\label{thm: weighted DP-SGD is DP}
We provide privacy guarantees for~\cref{alg: weighted dp-sgd}: 
\begin{enumerate}
    \item Suppose we sample \textit{with replacement} in line 4 of~\cref{alg: weighted dp-sgd}. Then, there exist constants $c_1, c_2$ such that for any $\eps < c_1 \left(\frac{\kpr}{\npr}\right)^2 T$, \cref{alg: weighted dp-sgd} is $(\eps, \delta)$-semi-DP for any $\delta > 0$ if we choose $\sigma^2 \geq c_2 \frac{C^2 \ln(1/\delta) T}{\eps^2 \npr^2}$.
    \item Suppose we sample \textit{without replacement} in line 4 and choose $T \leq \frac{n}{\kpr}$. Then~\cref{alg: weighted dp-sgd} is $\rho$-semi-zCDP if $\sigma^2 \geq \frac{2 C^2}{\rho \kpr^2}$. 
\end{enumerate}
\end{proposition}
\begin{proof}
Note that the $\ell_2$-sensitivity of the private stochastic gradient query is \[
\Delta = \sup_{\xpr \sim \xpr'} \left\|\frac{\alpha}{\kpr} \sum_{x \in B_t^{priv}} \text{clip}_C(\nabla f(w_t, x)) - \frac{\alpha}{\kpr}\sum_{x' \in B_t^{priv}} \text{clip}_C(\nabla f(w_t, x'))\right\|_2 \leq \frac{2 \alpha C}{\kpr}.\]

1. Consider sampling with replacement. Then we are randomly subsampling from the private data uniformly with sampling ratio $\kpr/a$. Thus, the theorem follows from \citep[Theorem 1]{abadi16}. 

2. Consider sampling without replacement. Then by the $\rho$-zCDP guarantee of the Gaussian mechanism~\citep[Proposition 1.6]{bun16} and the sensitivity bound above, $\widetilde{g}_t$ is $\rho$-semi-zCDP for every $t$. Moreover, since we are sampling without replacement, the privacy of every $x \in \xpr$ is only affected by $\widetilde{g}_t$ for a single $t \in [T]$. Thus, semi-zCDP of \cref{alg: weighted dp-sgd} follows by parallel composition~\cite{mcsherry2009privacy}. 
\end{proof}

Excess risk of 
By~Proposition~\ref{prop: even more optimal ME}, there exists a choice of $\alpha$ such that the variance of our unbiased estimator in line 7 is always less than the variance of both the throw-away gradient estimator $\frac{1}{\kpu} \sum_{x} \nabla f(w_t, x)$ and the DP-SGD estimator $\frac{1}{\kpr + \kpu} \sum_{x \in B_t} \nabla f(w_t, x) + u_t$, where $u_t$ is appropriately scaled (to ensure DP) Gaussian noise. Consequently, if we choose $T$ and $K = \kpr + \kpu$ such that $n = TK$ and sample without replacement (i.e. one pass), then
\textit{~\cref{alg: weighted dp-sgd} always has smaller excess risk than both the throw-away SCO algorithm and one-pass DP-SGD}. 
Moreover, if the loss function has Lipschitz continuous gradient, then one can combine the stochastic gradient estimator of line 7 with \textit{acceleration}~\cite{ghadimilan1} to obtain a linear-time semi-DP algorithm that always outperforms the accelerated DP algorithm of~\citet{lr21fl}. This is because the variance of our gradient estimator (hence our excess risk) is strictly smaller than that of~\citet{lr21fl}, by~\cref{prop: even more optimal ME}. For example, for $\beta$-smooth, $\mu$-strongly convex loss functions, one-pass Algorithm 1 achieves excess risk that is optimal up to a factor of $O(\sqrt{\beta/\mu})$ and improves over ~\cite{lr21fl}. Moreover, \cite{lr21fl} has the smallest excess risk among \textit{linear-time} (one-pass) DP algorithms whose \textit{privacy analysis does not require convexity}. Thus, our algorithm can be used for deep learning. In our numerical experiments, we implement the with-replacement sampling version of~\cref{alg: weighted dp-sgd}.  

We also note that near-optimal excess risk bounds for \textit{non-convex} loss functions that satisfy the \textit{(Proximal) PL inequality}~\cite{polyak, karimi2016linear} can be derived by combining a proximal variation of~\cref{alg: weighted dp-sgd} with the techniques of~\citet{lgr23}. Further, if $f(\cdot, x)$ is not uniformly Lipschitz, but has stochastic gradients with bounded $k$-th order moment for some $k \geq 2$, then excess risk bounds can still be derived for~\cref{alg: weighted dp-sgd} via techniques in~\cite{lowy2022outliers}. Our algorithm can also be extended to a variation of noisy stochastic gradient descent ascent, which could be used, e.g. for \textit{fair} semi-DP model training~\cite{lowy2022DPfair}. We leave it as future work to explore these and other potential applications of our gradient estimator in efficiently training private ML models with public data.

\section{Optimal Locally Private Model Training with Public Data}
\label{app: LDP}
\noindent \textbf{Notation and Setup:} Following~\citet{duchi2019lower}, we permit algorithms to be \textit{fully interactive}. That is, algorithms may adaptively query the same individual $i$ multiple times over the course of $T$ ``communication rounds.'' We denote $i$'s message in round $t$ by $Z_{i,t} \in \ZZ$. Person $i$'s message $Z_{i, t} \in \ZZ$ in round $t$ may depend on all previous communications $B^{(t)} := (Z_{\leq n, t}, B^{(t-1)})$ and on $i$'s own data: $Z_{i,t} \sim Q_{i,t}(\cdot | x_i, Z_{<i, t}, B^{(t-1)})$. If $i$'s data is private, then $Z_{i,t}$ is a randomized view of $x_i$ distributed (conditionally) according to $Q_{i,t}$. If $i$'s data is public, then $Z_{i,t}$
may be deterministic. 
Full interactivity is the most general notion of interactivity. If $T = 1$, then we say the algorithm is \textit{sequentially interactive}. If, in addition, each person's message $Z_{i,1}$ depends only on $x_i$ and not on $x_{j \neq i}$, then we say the algorithm is \textit{non-interactive}. Semi-LDP (Definition~\ref{def: semiLDP}) essentially requires that the messages $\{Z_{i,t}\}_{t\in [T]}$ be DP for all private $x_i \in \xpr$.

\subsection{Optimal Semi-LDP Mean Estimation}
\label{app: LDP ME}
\begin{theorem}[Re-statement of~\cref{thm: LDP mean estimation}]
Let $\eps \in (0, 1]$. There are absolute constants $0 < c \leq C$ s.t.
\[
c\min\left\{\frac{1}{\npu}, \frac{d}{n \eps^2}\right\}\leq \mathcal{M}^{\text{loc}}_{\text{pop}}(\varepsilon, \npr, n, d) \leq C \min\left\{\frac{1}{\npu}, \frac{d}{n \eps^2}\right\}.
\]
\end{theorem}

\begin{proof}
\textbf{Lower bound:} We will actually prove a more general lower bound than the one in~\cref{thm: LDP mean estimation}; namely, we will show a lower bound on the minimax $\ell_1$-error for estimation of distributions on $\XX_r = \{\pm r\}^d$ for $r > 0$.
To that end, let 
$\gamma \in (0, 1)$ and 
\[
P_1:= \begin{cases}
r &\mbox{with probability} \frac{1+\gamma}{2} \\
-r &\mbox{with probability} \frac{1-\gamma}{2}
\end{cases}\]
and \[
P_{-1}:= \begin{cases}
r &\mbox{with probability} \frac{1-\gamma}{2} \\
-r &\mbox{with probability} \frac{1+\gamma}{2}
\end{cases}.\]
 We define our hard distribution on $\XX_r$ by first drawing $V \sim \text{Unif}(\{\pm 1\}^d)$ and then---conditional on $V = v$---drawing $X_{i,j} \sim P_v = \Pi_{j=1}^d P_{v_j}$ for $i \in [n], j \in [d]$, where $P_v$ denotes the product distribution. We have Markov chains $V_j \to X_{i,j} \to Z$ for all $j \in [d], i \in [n]$, where $Z$ is the semi-LDP transcript. Note that $\left| \ln\left(\frac{dP_1}{dP_{-1}}\right) \right| \leq \ln\left( \frac{1+\gamma}{1 - \gamma} \right) \triangleq b$ and $e^b \leq 3$ for any $\gamma \in (0, 1/2]$. Now we will use the following lemma from Duchi \& Rogers (2019): 
\begin{lemma}\citep[Lemma 24]{duchi2019lower}
\label{lem: DR19 lem24}
    Let $V \to X \to Z$ be a Markov chain, where $X \sim P_v$ conditional on $V = v$. If $\left| \ln \frac{dP_v}{dP_{v'}}\right| \leq \alpha$ for all $v, v'$, then \[
    I(V; Z) \leq 2(e^{\alpha} - 1)^2 I(X; Z). 
    \]
\end{lemma}
Thus, for $V_j \sim \text{Unif}(\{\pm 1\})$, Lemma~\ref{lem: DR19 lem24} implies $I(V_j; Z) \leq \frac{8 \gamma^2}{(1 - \gamma)^2} I(X_{i,j}; Z)$. Hence the strong data processing constant~\citep[Definition 9]{duchi2019lower} is $\beta := \beta(P_1, P_{-1}) \leq \frac{8 \gamma^2}{(1 - \gamma)^2}$.

Now, $\theta_{v_j} := \expec_{x \sim P_{v_j}}[x] = \gamma r v_j$ for any $v_j \in \{\pm 1\}$. Moreover, letting $\theta_v = (\theta_{v_1}, \cdots, \theta_{v_d})$ for $v \in \{\pm 1\}^d$ and $\theta \in \mathbb{R}^d,$ we have \[
\|\theta - \theta_v\|_1 = \sum_{j=1}^d |\theta_j - r \gamma v_j| = r\gamma \sum_{j=1}^d \left| \frac{\theta_j}{r \gamma} - v_j \right| \geq r \gamma \sum_{j=1}^d \mathbbm{1}_{\{\text{sign}(\theta_j) \neq v_j\}}.
\]
Thus, $\{\pm 1\}^d$ induces an $r\gamma$-Hamming separation, so Assouad's lemma~\citep[Lemma 1]{duchi2018minimax} yields \[
\inf_{\alg \in \mathbb{A}_\eps^{\text{loc}}} \sup_{P \in \PP_r} \expec\|\alg(X) - \theta(P)\|_1 \geq r\gamma \sum_{j=1}^d \inf_{\hat{V}} \pr(\hat{V}_j(Z) \neq V_j),\]
where $Z$ is the communication transcript of $\alg$, the infimum on the RHS is over all 
estimators of $V$, $\theta(P) = \expec_{x \sim P}[x]$, and $\PP_r$ is the set of distributions on $\XX_r$

Assume WLOG that the private samples are the first $\npr$ samples of $X$: $\xpr = (x_1, \cdots, x_{\npr})$.
To lower bound $\sum_{j=1}^d \inf_{\hat{V}} \pr(\hat{V}_j(Z) \neq V_j)$, we use a slight extension of~\citet[Theorem 10]{duchi2019lower}: \[
\sum_{j=1}^d \inf_{\hat{V}} \pr(\hat{V}_j(Z) \neq V_j) \geq \frac{d}{2}\left[ 1 - \sqrt{\frac{7 (e^b + 1)}{d} \beta \left( I(\xpr; Z | V) + I(\xpu; Z | V) \right)}\right]. 
\]
This follows since $V \to \xpr \to Z$ and $V \to \xpu \to Z$ are both Markov chains and the other assumptions in \citep[Theorem 10]{duchi2019lower} all hold. Combining this bound with Assouad's lemma~\citep[Lemma 1]{duchi2018minimax} and substituting the definitions of $b$ and $\beta$ given above
gives us \[
\inf_{\alg \in \mathbb{A}_\eps^{\text{loc}}} \sup_{P \in \PP_r} \expec\|\alg(X) - \theta(P)\|_1 \geq \frac{r \gamma d}{2} \left[ 1 - \sqrt{\frac{896}{d} \gamma^2 \left( I(\xpr; Z | V) + I(\xpu; Z | V) \right)}\right]
\] for any $\gamma \in (0, 1/2]$. It remains to upper bound the conditional mutual information $I(\xpr; Z | V)$ and $I(\xpu; Z | V)$. 

Now for any $\eps$-semi-LDP algorithm with communication transcript $Z$, we have $I(\xpr; Z | V) \leq \npr \min(\eps, 4 \eps^2)$, by an easy extension of~\citet[Lemma 12]{duchi2019lower}.
Also, $I(\xpu; Z | V) \leq H(\xpu | V) \leq \log(|\XX^{\npu}|) = d \npu$, where $H(\cdot | \cdot)$ denotes conditional entropy. Thus, 
\[
\inf_{\alg \in \mathbb{A}_\eps^{\text{loc}}} \sup_{P \in \PP_r} \expec\|\alg(X) - \theta(P)\|_1 \geq \frac{r \gamma d}{2} \left[ 1 - \sqrt{\frac{4000}{d} \gamma^2 \left(  \npr \min(\eps, \eps^2) + d\npu \right)}\right].
\]
Choosing $\gamma^2 = c \min\left(\frac{1}{\npu}, \frac{d}{\npr \min(\eps, \eps^2)}\right)$ for some small constant $c > 0$ yields \[
\inf_{\alg \in \mathbb{A}_\eps^{\text{loc}}} \sup_{P \in \PP_r} \expec\|\alg(X) - \theta(P)\|_1 \gtrsim r d \min\left(\frac{1}{\sqrt{\npu}}, \sqrt{\frac{d}{\npr \min(\eps, \eps^2)}} \right),
\]
whence 
\[
\inf_{\alg \in \mathbb{A}_\eps^{\text{loc}}} \sup_{P \in \PP_r} \expec\|\alg(X) - \theta(P)\|_2 \gtrsim r \sqrt{d} \min\left(\frac{1}{\sqrt{\npu}}, \sqrt{\frac{d}{\npr \min(\eps, \eps^2)}} \right).
\]
By applying the non-private mean estimation lower bound, we get \[
\inf_{\alg \in \mathbb{A}_\eps^{\text{loc}}} \sup_{P \in \PP_r} \expec\|\alg(X) - \theta(P)\|_2 \gtrsim r \sqrt{d} \min\left(\frac{1}{\sqrt{\npu}}, \sqrt{\frac{d}{\npr \min(\eps, \eps^2)}} + \frac{1}{\sqrt{n}}\right).
\]
Choosing $r = 1/\sqrt{d}$ ensures that $\PP_r \subset \PP(\mathbb{B})$ and yields \[
\inf_{\alg \in \mathbb{A}_\eps^{\text{loc}}} \sup_{P \in \PP(\mathbb{B})} \expec\|\alg(X) - \theta(P)\|_2 \gtrsim \min\left(\frac{1}{\sqrt{\npu}}, \sqrt{\frac{d}{n \min(\eps, \eps^2)}} + \frac{1}{\sqrt{n}}\right).
\]
Applying Jensen's inequality completes the proof of the lower bound in~\cref{thm: LDP mean estimation}. Since we assumed $\eps \leq 1 \leq d$, the minimum in the denominator simplifies to $\eps^2$ and the $1/\sqrt{n}$ term is non-dominant. 

\noindent \textbf{Upper bound:} The first term in the minimum can be realized by the algorithm that throws away the private data and returns $\alg(X) = \frac{1}{\npu} \sum_{x \in X_{\text{pub}}} x$, which is $0$-semi-LDP. Also, \[
\expec \|\alg(X) - \expec x\|^2 = \frac{1}{\npu^2}\sum_{x \in X_{\text{pub}}} \expec \|x - \expec x\|^2 \leq \frac{1}{\npu}.
\]
The second term in the upper bound can be realized by the $\eps$-LDP (hence $\eps$-semi-LDP) estimator of~\citet{duchi13}, which has worst-case MSE upper bounded by $O\left(\frac{d}{n \eps^2}\right)$. 
\end{proof}

\begin{algorithm}
	\caption{$\pu(p,\gamma)$~\cite{bhowmick2018protection}}
	\label{alg:pu}
	\begin{algorithmic}[1]
	\STATE {\bfseries Input:} $v \in \mathbb{S}^{d-1}$, $\gamma \in [0,1]$, $p \in [0,1]$. $B(\cdot; \cdot, \cdot)$ below is the incomplete Beta function $B(x;a,b) = \int_0^x t^{a-1}(1-t)^{b-1} \textrm{d}t$ and $B(a,b) = B(1; a, b)$.
		\STATE Draw $z \sim 
  \text{Bernoulli}
  (p)$
		\IF{$z=1$}
		    \STATE Draw $V \sim \uniform \{u \in \sphere^{d-1}: \langle u,v \rangle \geq \gamma \}$
		\ELSE
		     \STATE Draw $V \sim \uniform \{u \in \sphere^{d-1}: \langle u,v\rangle < \gamma \}$
		\ENDIF
		\STATE Set $\alpha = \frac{d-1}{2}$ and $\tau = \frac{1 + \gamma}{2}$
		\STATE Calculate normalization constant
		\begin{equation*}
		    m = \frac{(1-\gamma^2)^\alpha}{2^{d-2} (d-1)} \left( \frac{p}{B(\alpha,\alpha) - B(\tau; \alpha,\alpha)} + \frac{1-p}{B(\tau; \alpha,\alpha)}   \right)
		\end{equation*}
        \STATE Return $\frac{1}{m} \cdot V$
	\end{algorithmic}
\end{algorithm}

\subsubsection{An ``Even More Optimal'' Semi-LDP Estimator}
\begin{lemma}[Re-statement of Lemma~\ref{lem: privunit error}]
Let $P$ be a distribution on $\mathbb{B}$ with $V^2 = \expec\|x - \expec_{x \sim P}[x]\|.$
Let $c > 0$ such that $\expec_{x \sim P}\| \duchisample(x) - \expec_{x \sim P}[x]\|^2 = \frac{c d}{n \eps^2}$, so that $\expec_{X \sim P^n}\| \duchiset(X) - \expec_{x \sim P}[x]\|^2 = \frac{c d}{n \eps^2} + \frac{V^2}{n}$. Then,
\[
\expec_{X \sim P^n}\left[\left\|\alg_{\text{Semi-Duchi}}(X) - \expec_{x \sim P}[x] \right\|^2\right]  =  \frac{\npr}{n} \cdot \frac{c d}{n \eps^2} + \frac{\npu}{n} \cdot \frac{V^2}{n}.
\]
\end{lemma}
\begin{proof}
We have
\begin{align*}
\expec\|\widetilde{\alg}_{\text{semi-Duchi}}(X) - \expec_{x \sim P}[x]\|^2 &= \frac{1}{n^2}\left[\sum_{x \in \xpr}\expec\|\duchisample(x) - \expec_{x \sim P}[x]\|^2 + \sum_{x \in \xpu}\expec\|x - \expec_{x \sim P}[x]\|^2 \right] \\
&= \frac{n c d}{\eps^2 n^2} + \frac{\npu V^2}{n^2},
\end{align*}
by independence of the data and the assumptions in the statement of the lemma. 
\end{proof}

\subsubsection{A Semi-LDP Estimator with Optimal Constants}
\begin{proposition}[Re-statement of~Proposition~\ref{prop: semi-LDP PrivUnit is truly optimal}]
Let $\alg(X) = \frac{1}{n} \left[\MMpr(\RR(x_1), \cdots, \RR(x_{\npr})) + \MMpu(\xpu) \right]$ be a $\eps$-semi-LDP algorithm, where $\RR: \mathbb{S}^{d-1} \to \ZZ$ is an $\eps$-LDP randomizer and 
$\MMpr: \ZZ^{\npr} \to \mathbb{R}^d$ and $\MMpu: \ZZ^{\npu} \to \mathbb{R}^d$ are aggregation protocols such that $\expec_{\MMpr, \RR}\left[\MMpr(\RR(x_1), \cdots, \RR(x_{\npr}))\right] = \sum_{x \in \xpr} x$ and $\expec_{\MMpu}\left[\MMpu(\xpu)\right] = \sum_{x \in \xpu} x$ for all $X = (\xpr, \xpu) \in \left(\mathbb{S}^{d-1}\right)^n$. Then, 
\[
\sup_{X \in \left(\mathbb{S}^{d-1}\right)^n} \expec_{\alg_{\text{semi-PrivU}}}\|\alg_{\text{semi-PrivU}}(X) - \Bar{X}\|^2 \leq \sup_{X \in \left(\mathbb{S}^{d-1}\right)^n} \expec_{\alg}\|\alg(X) - \Bar{X}\|^2. 
\]
\end{proposition}
\begin{proof}
First, Asi et al.~\citep[Proposition 3.4]{asilocal} showed that PrivUnit (with a proper choice of $(p, \gamma)$) has the smallest worst-case variance among all unbiased $\eps$-LDP randomizers: \begin{equation}
\label{eq: Privunit is optimal}
    \sup_{x \in \mathbb{S}^{d-1}} \expec\|\RR(x) - x\|^2 \geq \sup_{x \in \mathbb{S}^{d-1}} \expec\|\text{PrivUnit}(x) - x\|^2
    \end{equation}
    for all $\eps$-LDP randomizers $\RR$ such that $\expec[\RR(x)] = x$ for all $x \in \mathbb{S}^{d-1}$. 

Now, let $\RR$ be a $\eps$-LDP randomizer and $\MMpr$ and $\MMpu$ be aggregation protocols such that the assumptions in~Proposition~\ref{prop: semi-LDP PrivUnit is truly optimal} are satisfied. We claim that there exists an unbiased $\eps$-LDP randomizer $\RR': \mathbb{S}^{d-1} \to \ZZ$ such that \begin{equation}
\label{eq: canonical protocols are optimal}
\sup_{X \in \left(\mathbb{S}^{d-1}\right)^n} \expec_{\alg}\left\|n \alg(X) - \sum_{x \in X} x\right\|^2 \geq \sup_{X \in \left(\mathbb{S}^{d-1}\right)^n} \expec_{\RR'}\left\|\sum_{x \in \xpr} \RR'(x) 
- \sum_{x \in \xpr} x \right\|^2.
\end{equation}
To prove~\cref{eq: canonical protocols are optimal}, we follow the idea in the proof of~\citet[Proposition 3.3]{asilocal}. Let $P$ denote the uniform distribution on $\mathbb{S}^{d-1}$. We have 
\begin{align*}
&\sup_{X \in \left(\mathbb{S}^{d-1}\right)^n} \expec_{\alg}\left\|n \alg(X) - \sum_{x \in X} x\right\|^2 \geq \expec_{X \sim P^n, \alg}\left\|n \alg(X) - \sum_{x \in X} x\right\|^2 \\
\geq& \expec\left\| \MMpr(\RR(x_1), \cdots, \RR(x_{\npr})) - \sum_{x \in \xpr} x \right\|^2 + \expec\left\| \MMpu(\xpu) - \sum_{x \in \xpu} x \right\|^2,
\end{align*}
since the cross-term (inner product) vanishes by independence of $\xpu$ and $\xpr$, and unbiasedness of $\MMpu$. Now, \citep[Lemma A.1]{asilocal} shows that there exist $\eps$-LDP randomizers $\{\hat{\RR}_x\}_{x \in \xpr}$ such that $\expec[\hat{\RR}_x(v)] = v$ for all $v \in \mathbb{S}^{d-1}$ and 
\[
\expec_{\xpr \sim P^{\npr}} \left\| \MMpr(\RR(x_1), \cdots, \RR(x_{\npr})) - \sum_{x \in \xpr} x \right\|^2 \geq \sum_{x \in \xpr} \expec_{v \sim P}\|\hat{\RR}_x(v) - v\|^2.
\]
Hence \[
\sup_{X \in \left(\mathbb{S}^{d-1}\right)^n} \expec_{\alg}\left\|n \alg(X) - \sum_{x \in X} x\right\|^2 \geq \sum_{x \in \xpr} \expec_{v \sim P}\|\hat{\RR}_x(v) - v\|^2. 
\]
Define $\RR'_x(v) := U^T \hat{\RR}_x(U v)$ for $v \in \mathbb{S}^{d-1}$, where $U$ is a uniformly random rotation matrix such that $U^T U = \mathbf{I}_d$. Note that $\RR'_x$ is an $\eps$-LDP randomizer such that $\expec[\RR'(v)] = v$ for all $v \in \mathbb{S}^{d-1}, x \in \xpr$. Moreover,  for any fixed $v \in \mathbb{S}^{d-1}, x \in \xpr$, we have 
\begin{align*}
\expec\|\RR_x'(v) - v\|^2 &= \expec_U \|\hat{\RR}_x(Uv) - Uv\|^2 \\
&= \expec_{v' \sim P} \|\hat{\RR}_x(v') - v'\|^2 \\
\end{align*}
Let $x^* := \argmin_{x \in \xpr} \sup_{v\in \mathbb{S}^{d-1}} \expec\|\RR_x'(v) - v\|^2$ and $\RR'(v) := \RR_{x^*}'(v)$. 
Then putting the pieces together, we have 
\begin{align*}
\sup_{X \in \left(\mathbb{S}^{d-1}\right)^n} \expec_{\alg}\left\|n \alg(X) - \sum_{x \in X} x\right\|^2 &\geq \sum_{x \in \xpr} \sup_{v \in \mathbb{S}^{d-1}} \expec\|\RR_x'(v) - v\|^2 \\
&\geq \npr \sup_{v \in \mathbb{S}^{d-1}} \expec\|\RR'(v) - v\|^2 \\
&= \sup_{\xpr \in (\mathbb{S}^{d-1})^{\npr}} \expec_{\RR'} \left\| \sum_{x \in \xpr} \RR'(x) - \sum_{x \in \xpr} x \right\|^2,
\end{align*}
by conditional independence of $\{\RR'(x)\}_{x \in \xpr}$ given $X$. This
establishes~\cref{eq: canonical protocols are optimal}. 
Thus, 
\begin{align*}
n^2 \sup_{X \in \left(\mathbb{S}^{d-1}\right)^n} \expec_{\alg}\left\| \alg(X) - \Bar{X}\right\|^2 
&= \sup_{X \in \left(\mathbb{S}^{d-1}\right)^n} \expec_{\alg}\left\|n \alg(X) - \sum_{x \in X} x\right\|^2 \\ 
&\geq \sup_{X \in \left(\mathbb{S}^{d-1}\right)^n} \expec_{\RR'} \left\|\sum_{x \in \xpr} \RR'(x) 
- \sum_{x \in \xpr} x \right\|^2 \\
&\geq \sup_{X \in \left(\mathbb{S}^{d-1}\right)^n} \expec \left\| \sum_{x \in \xpr} \text{PrivUnit}(x) - \sum_{x \in \xpr} x \right\|^2 \\
&= n^2 \sup_{X \in \left(\mathbb{S}^{d-1}\right)^n} \expec_{\alg_{\text{semi-PrivU}}}\|\alg_{\text{semi-PrivU}}(X) - \Bar{X}\|^2,
\end{align*}
where we used~\cref{eq: Privunit is optimal} in the last inequality. 
Dividing both sides of the above inequality by $n^2$ completes the proof. 
\end{proof}

\subsection{Optimal Semi-LDP Stochastic Convex Optimization}
\label{app: LDP SCO}
If $\mu = 0$ (convex case), we denote $\mathcal{R}_{\text{SCO}}^{\text{loc}}(\varepsilon, \npr, n, d, L, D) := \mathcal{R}_{\text{SCO}}^{\text{loc}}(\varepsilon, \npr, n, d, L, D, \mu = 0)$. 
\begin{theorem}[Complete statement of~\cref{thm: LDP SCO convex}]
Let $\eps \in (0, 1]$. There exist absolute constants $c$ and $C$, with $0 < c \leq C$, such that 
\small
\[
c LD \min\left\{\frac{1}{\sqrt{\npu}}, \sqrt{\frac{d}{n \eps^2}} \right\} \leq \mathcal{R}_{\text{SCO}}^{\text{loc}}(\varepsilon, \npr, n, d, L, D) \leq C LD \min\left\{\frac{1}{\sqrt{\npu}}, \sqrt{\frac{d}{n \eps^2}} \right\},
\] 
\normalsize
and
\small
\[
c LD \min\left\{\frac{1}{\npu}, \frac{d}{n \eps^2}\right\} \leq \mathcal{R}_{\text{SCO}}^{\text{loc}}(\varepsilon, \npr, n, d, L, D, \mu) \leq C \frac{L^2}{\mu} \min\left\{\frac{1}{\npu}, \frac{d}{n \eps^2}\right\}.
\]
\normalsize
\end{theorem}
\begin{proof}
\textbf{Lower bounds:}
Let $\alg$ be $\eps$-semi-LDP and denote $\wpr = \alg(X)$. 

\noindent \ul{Strongly convex lower bound:} We begin with the strongly convex lower bounds, which can be proved straightforwardly by reducing strongly convex SCO to mean estimation and applying~\cref{thm: LDP mean estimation}. In a bit more detail, let $f: \WW \times \XX \to \mathbb{R}$ be given by \[
f(w,x) = \frac{L}{2D}\|w - x\|^2,
\]
where $\WW = \XX = D \mathbb{B}$. 
Note that $f$ is $L$-uniformly Lipschitz and $\frac{L}{D}$-strongly convex in $w$ for all $x$. Further, $\ws := \argmin_{w \in \WW} \left\{F(w) = \expec_{x \sim P}\left[f(w,x)\right]\right\} = \expec_{x \sim P}[x]$. By a direct calculation (see e.g. \citep[Lemma 6.2]{klz21}), we have \begin{equation}
\label{eq: sc reduction to ME2}
\expec F(\wpr) - F(\ws) = \frac{L}{2D} \expec \|\wpr - \ws\|^2.
\end{equation}
We can lower bound $\expec \|\wpr - \ws\|^2 = \expec \| \alg(X) - \expec_{x \sim P}[x]\|^2$ via~\cref{thm: LDP mean estimation} (and its proof, to account for the re-scaling). Specifically, there is a distribution $P$ on $\XX$ such that 
\[
\expec_{X \sim P^n, \wpr} \|\wpr - \ws\|^2 \geq c D^2 \min\left(\frac{1}{\npu}, \frac{d}{n \eps^2} \right).
\]
Combining this with~\cref{eq: sc reduction to ME2} leads to the desired excess risk lower bound.

\noindent \ul{Convex lower bound:}
We will begin by proving the lower bounds for the case in which $L = D = 1$, and then scale our construction to get the lower bounds for arbitrary $L, D$. 

Let $\WW = \mathbb{B}$, $\XX = \left\{\pm \frac{1}{\sqrt{d}}\right\}^d$, and 
\[
f(w,x) = - \langle w, x \rangle,
\]
which is convex and $1$-uniformly-Lipschitz in $w$ on $\XX$.
For any $\eps$-semi-LDP $\alg'(X) = \wpr'$ and any $\gamma \in (0, 1)$, the proof of~\cref{thm: LDP mean estimation} constructs a distribution $P_{\gamma}$ on $\XX$ with mean $\expec_{x \sim P_{\gamma}}[x] = \theta \in \{\pm \frac{\gamma}{\sqrt{d}}\}^d$ such that \begin{equation}
\label{eq:c}
\expec_{\wpr', X \sim P_{\gamma}^n}\left[\left\|\wpr' - \theta \right\|^2\right] \geq \frac{\gamma^2}{4}\left[1 - \sqrt{4000 \gamma^2\left(\frac{\npr \eps^2}{d} + \npu \right)}\right].
\end{equation}
Now, let $F_\gamma(w) = \expec_{x \sim P_{\gamma}}[f(w,x)]$ and $\ws_\gamma = \argmin_{w \in \WW} F_\gamma(w) = \frac{\theta}{\|\theta\|}$. 
A direct calculation (see e.g. \citep[Equation 14]{klz21})
shows \begin{align}
\label{eq:d}
\expec F_\gamma(w) - F_\gamma^* &\geq \frac{1}{2} \expec\left[\|\theta \| \|w - \ws_\gamma\|^2 \right] \nonumber \\
&\geq \frac{1}{2 \|\theta\|} \expec\left[\|w' - \theta\|^2 \right] \nonumber \\
&= \frac{1}{2\gamma} \expec\left[\|w' - \theta\|^2 \right]
\end{align}
for any $w \in \WW, w' := w \| \theta\|$. Note that $\alg(X) = \wpr$ is $\eps$-semi-DP if and only if $\alg'(X) = \wpr' := \|\theta\| \wpr = \gamma \wpr$ is $\eps$-semi-DP, by post-processing. Thus, \cref{eq:c,eq:d} together imply that any $\eps$-semi-DP $\wpr'$ has worst-case excess risk that is lower bounded by \begin{align*}
\expec[F_\gamma(\wpr') - F_\gamma^*] \geq \frac{\gamma}{8} \left[1 - \sqrt{4000 \gamma^2\left(\frac{\npr \eps^2}{d} + \npu \right)}\right].
\end{align*}
Choosing $\gamma^2 = c \min\left(\frac{d}{\npr \eps^2}, \frac{1}{\npu} \right)$ for some small $c \in (0, 1/16000)$ implies \[
\expec[F_\gamma(\wpr') - F_\gamma^*] \geq c' \min\left(\sqrt{\frac{d}{\npr \eps^2}}, \frac{1}{\sqrt{\npu}} \right)
\]
for some $c' > 0$. This proves the desired lower bound for the case when $L = D = 1$. In the general case, we scale our hard instance, as in the proof of~\cref{thm: CDP mean estimation}: Let $\wt{W} = D \WW$, $\wt{X} = L \XX$, and $\wt{x} \sim \wt{P} \iff \wt{x} = Lx$ for $x \sim P$. Define $\wt{f}: \wt{\WW} \times \wt{\XX} \to \mathbb{R}$ by $\wt{f}(\wt{w}, \wt{x}) := f(\wt{w}, \wt{x}) = - \langle \wt{w}, \wt{x} \rangle$. Then $\wt{f}(\cdot, \wt{x})$ is $L$-Lipschitz and convex. Moreover, if $F(w) = \expec_{x \sim P}[f(w,x)]$, $\theta := \expec_{x \sim P}[x]$, $\ws := \argmin_{w \in \WW} F(w) = \frac{\theta}{\|\theta\|}$, 
$\wt{F}(\wt{w}) = \expec_{\wt{x} \sim \wt{P}}[f(\wt{w}, \wt{x})]$,  $\wt{w} = Dw$, and $\wt{\theta} := \expec_{\wt{x} \sim \wt{P}}[\wt{x}] = L \theta$,
then $\wt{w}^* = D \ws \in \argmin_{\wt{w} \in \wt{\WW}} \wt{F}(\wt{w})$ and 
\begin{align*}
\wt{F}(\wt{w}) - \wt{F}^* &= -\langle \wt{w}, \wt{\theta}\rangle + \langle\wt{w}^*, \wt{\theta} \rangle  \\
&= D \langle  \wt{\theta}, \ws - w \rangle \\
&=  LD \langle \theta, \ws - w \rangle \\
&= LD \left[F(w) - F^* \right].
\end{align*}
This shows that the excess risk of the scaled instance scales by $LD$, completing the lower bound proofs. 

\textbf{Upper bounds:} The first term in each of the upper bounds ($LD/\sqrt{\npu}$ for convex, and $L^2/(\mu \npu)$ for strongly convex) is attained by the throw-away algorithm that runs $\npu$ steps of (one-pass) SGD on $\xpu$~\cite{ny}. 

The second term in the convex upper bound follows from the $\eps$-LDP (hence semi-LDP) upper bound of~\citet[Proposition 3]{duchi13}. 

For the second term in the strongly convex upper bound, we run $\eps$-LDP-SGD as in~\cite{duchi13}. We also return a non-uniform weighted average $\hat{w}_n$ of the iterates $w_1, \ldots, w_n$ as in~\cite{rakhlin12} to obtain
\[
\expec F(\hat{w}_n) - F^* \leq C \frac{G^2}{\mu n},
\]
where $G^2 = \sup_{t \in [n]} \expec \|\duchisample(\nabla f(w_t, x_t))\|^2 \lesssim L^2 \left(1 + \frac{d}{\eps^2}\right)$~\cite{duchi13}. 
Thus, \[
\expec F(\hat{w}_n) - F^* \leq C' \frac{L^2}{\mu} \left(\frac{d}{n \eps^2}\right),
\]
since $d \geq 1 \geq \eps^2$. 
This completes the proof. 
\end{proof}
\subsubsection{An ``Even More Optimal'' Semi-LDP Algorithm for SCO}
\label{app: more optimal LDP SCO}

\begin{algorithm}[ht]
\caption{Semi-LDP-SGD}
\label{alg: semi-Ldp-sgd}
\begin{algorithmic}[1]
\STATE {\bfseries Input:} 
clip threshold $C > 0$, 
stepsize $\eta$,
privacy parameter $\eps \in (0, 1]$. 
 \STATE Initialize $w_0 \in \WW$.
 \FOR{$t \in \{0, 1, \cdots, n-1\}$} 
  \STATE Draw random sample $x_t$ from $X$ without replacement. 
  \IF{$x_t \in \xpr$}
  \STATE $\widetilde{g}_t \gets \duchisample(\text{clip}_C(\nabla f(w_t, x_t)))$.
  \ELSE
  \STATE $\widetilde{g}_t \gets \nabla f(w_t, x_t)$
  \ENDIF
 \STATE Update $w_{t+1} := \Pi_{\WW}[w_t - \eta \widetilde{g}_t]$.
 \ENDFOR \\
\STATE {\bfseries Output:} $\Bar{w}_n = \frac{1}{n} \sum_{i=1}^n w_i$.
\end{algorithmic}
\end{algorithm}
\begin{proposition}[Re-statement of Proposition~\ref{prop: more optimal semi-LDP SCO}]
Let $f \in \FFmuz$ and $P$ be any distribution and $\eps \leq d$. \cref{alg: semi-Ldp-sgd} is $\eps$-semi-LDP. Further, 
there is an absolute constant $c$ such that 
the output $\alg(X) = \Bar{w}_n$ of~\cref{alg: semi-Ldp-sgd} satisfies \[
\expec_{\alg, X \sim P^n} F(\Bar{w}_n) - F^* \leq c \frac{LD}{\sqrt{n}} \max\left\{
\sqrt{\frac{d}{\eps^2}} \sqrt{\frac{\npr}{n}}, \sqrt{\frac{\npu}{n}}\right\}.
\]
\end{proposition}
\begin{proof}
\textbf{Privacy:} Since $\duchisample$ is an $\eps$-LDP randomizer and we are applying $\duchisample$ to the gradients of all the private samples $x \in \xpr$, \cref{alg: semi-Ldp-sgd} is $\eps$-semi-LDP. 

\textbf{Excess risk: }Choose $C > L$: i.e. we don't clip, since stochastic subgradients are already uniformly bounded by the $L$-Lipschitz assumption. 
By the classical analysis of the stochastic subgradient method (see e.g.~\cite{bubeck2015convex}), we can obtain \begin{align*}
\expec F(\Bar{w}_n) - F^* &\leq \frac{1}{n} \sum_{t=1}^n \expec[\widetilde{g}_t^T(w_t - w^*)]\\
&\leq \frac{D^2}{\eta n} + \frac{\eta}{n}\left(\npr G_a^2 + \npu G_b^2\right),
\end{align*}
where $G_a^2 := \sup_t \expec\|\duchisample(\nabla f(w_t, x_t))\|^2$ and $G_b^2 := \sup_t \expec\|\nabla f(w_t, x_t)\|^2$. By the uniform Lipschitz assumption, we have $G_b^2 \leq L^2$. By~\cite{duchi13}, we have $G_a^2 \leq c^2 L^2 \frac{d}{\eps^2}$ for some absolute constant $c > 0$. Thus, choosing $\eta = \frac{D}{L} \min\left(\sqrt{\frac{\eps^2}{\npr c^2 d}}, \frac{1}{\sqrt{\npu}} \right)$ yields 
\begin{align*}
\expec F(\Bar{w}_n) - F^* &\leq 3 \frac{LD}{\sqrt{n}} \max\left( 
\sqrt{\frac{c^2 d}{\eps^2}} \sqrt{\frac{\npr}{n}}, \sqrt{\frac{\npu}{n}}\right).
\end{align*} 
This completes the proof. 
\end{proof}

\section{Numerical Experiments}
\label{app: experiments}

Code for all of the experiments is available here: 
\url{https://github.com/optimization-for-data-driven-science/DP-with-public-data}.

\subsection{Central Semi-DP Experiments}

\subsubsection{Semi-DP Linear Regression with Gaussian Data}
\label{app: linreg setup}
\paragraph{Data Generation:} We implement~\cref{alg: weighted dp-sgd} on synthetic data designed for a linear regression problem of dimension $2,000$, using the squared error loss: $f(w,x) = (\langle w, x^{(1)} \rangle - x^{(2)})^2$, where $x^{(1)} \in \mathbb{R}^{d}$ denotes the feature vector and $x^{(2)} \in \mathbb{R}$ is the target. Here, $d = 2,000$. Our synthetic dataset consists of $n = 30,000$ training samples, $7500$ validation samples, and $37,500$ test samples. The feature vectors $x_i := x^{(1)}_i$ and the optimal parameter vector $w^*$ are drawn i.i.d. from a multivariate Gaussian distribution $\mathcal{N}\left(0, \mathbf{I}_{2000}\right)$. We generate predicted values $x^{(2)}_i$ from a Gaussian distribution $\mathcal{N}\left( \langle w^*, x_i \rangle, 1\right)$. Thus, the optimal linear regression model has which ensures an optimal mean squared error of $1$. 

\paragraph{Experimental Setup:}

Our experiments investigate two phenomena: 1) the effect of the ratio $\frac{\npu}{n}$ on test loss when $\eps \in \{2, 4\}$ is fixed, for values of $\frac{\npu}{n}$ ranging from $0.01$ to $0.95$, see Figures \ref{fig:acc_vs_ratio_eps_2}-\ref{fig:acc_vs_ratio_eps_4}; and 2) the effect of privacy (quantified by $\varepsilon$) on test loss, for fixed $\frac{\npu}{n} \in \{0.1, 0.25\}$, and varying $\varepsilon \in \{0.1, 0.5, 1.0, 2.0, 4.0, 8.0\}$, see Figures \ref{fig:acc_vs_eps_ratio_0_1}-\ref{fig:acc_vs_eps_ratio_0_25} . We maintain the privacy parameter $\delta$ at a constant value of $10^{-5}$ throughout our experiments. We set the private batch size $\kpr=500$, public batch size $\kpu=200$, and iterations $T=5000$. All algorithms undergo extensive hyperparameter tuning using the validation dataset, and the performance of each tuned algorithm is subsequently assessed using the test dataset. (See ``\textbf{Hyperparameter Tuning}'' paragraph below for details on the tuning process.)  %

\paragraph{Details on Implementations of Algorithms:} 
We compare four different semi-DP algorithms: 1. \textit{Throw-away}.   
2. \textit{DP-SGD}~\cite{abadi16,de2022unlocking}.  
3. \textit{PDA-MD}~\citep[Algorithm 1]{amid2022public}.
4. \textit{Our~\cref{alg: weighted dp-sgd}}---specifically, the \textit{sample-with-replacement} version of our algorithm. 
If $\npu \geq d$, the \textit{Throw-away} algorithm simply returns a minimizer $w_{\text{pub}}$ of the public loss: $w_{\text{pub}}=(\xpu^T  \xpu )^{-1}\xpu^T y_{\text{pub}}$. Otherwise, we used pretrained warm-start models with all public samples. (See ``\textbf{warm-start}'' paragraph below for details.)
\textit{DP-SGD} adds noise to all (public or private) gradients. We use the state-of-the-art (for image classification) implementation of DP-SGD of~\cite{de2022unlocking}.
We adopt the re-parameterization of DP-SGD in~\citep[Equation 3]{de2022unlocking} to ease the hyperparameter tuning. For \textit{PDA-MD}, we implement~\citep[Algorithm]{amid2022public}. We use their exact Mirror descent form by multiplying the inverse of the hessian $\xpu^T \xpu$ by the private gradient \cite{amid2022public}. In their original implementation, they added a small constant, Hessian Regularization, times the identity matrix to the Hessian before calculating the inverse for numerically stable. We choose the hessian regularization constant as $0.01$, the same value in their original implementation. 

\label{Appendix: effect of different ratio}
\textit{Effect of increasing ratio $\npu/\npr$:}
More public data always improves the performance of \cref{alg: weighted dp-sgd}, but does not always benefit PDA-MD. The primary reason for this is that our algorithm more effectively handles the increasing privacy noise that is needed to maintain semi-DP with increasing ratio $\npu/n$. Our algorithm achieves this by using the weight parameter to reduce the variance of the increasingly noisy private gradients and leverage public gradients. 
By contrast, as $\npu/n$ ratio rises, the efficacy of PDA-MD may diminish due to its over-reliance on increasingly noisy private gradients. We verify our reasoning numerically in Appendix \ref{sec: tables}: Tables \ref{table: std_dev_eps_2} and \ref{table: std_dev_eps_4} record the standard deviation $\sigma$ of the privacy noise for different ratios.

\textit{Effect of hessian regularization parameter on PDA-MD:}
Upon reproducing the results of PDA-MD, we discovered a high sensitivity to the choices of Hessian Regularization, especially when 
the ratio of the largest and the smallest eigenvalue of the data matrix is large. We test PDA-MD on the dataset proposed in the original study as well as on our own dataset. The results of these tests are displayed in~\Cref{fig:hessian_reg}. We see PDA-MD is sensitive to hessian regularization parameter, which requires extra tuning on complicated tasks. We implement PDA-MD with the optimal regularization value of $0.01$.

\textit{Details on warm-start and cold-start:} In our evaluation, we adopt a ``warm start'' strategy for all algorithms: we first find a minimizer $w_{\text{pub}}$ of the public loss, and then initialize the training process with $w_{\text{pub}}$. Note that there are two cases: $\npu \geq n$ and $\npu < n$. In the case of $\npu \geq n$, the minimizer $w_{\text{pub}}$ of the public loss can be obtained via $w_{\text{pub}}=\left(\xpu^T  \xpu \right)^{-1}\xpu^T y_{\text{pub}}$. In the case of $\npu < n$, we run SGD on linear regression with all $\npu$. Specifically, we minimize $w_{\text{pub}}$ of the public loss by running the SGD with public batch size $\kpu=200$, stepsize $\eta_t=0.5$, and iterations $T=50000$ to allow the models fully converge. For cold start scenarios, we initialize all model parameters $w$ to $0$. The cold start experiment results can be seen in Figure \ref{fig:acc_vs_ratio_eps_4_non_warm} and Figure \ref{fig:acc_vs_eps_ratio_0_1_non_warm}.

\textit{\ul{Clipping public gradients improves performance:}} 
We have empirically found that a slight modification of~\cref{alg: weighted dp-sgd} offers performance benefits. Namely, it is beneficial to project the public gradients onto the $\ell_2$-sphere of radius $C$; that is, we re-scale the public gradients to have $\ell_2$-norm equal to $C$. 
Note that semi-DP still holds regardless of whether or not the public gradients are re-scaled. Re-scaling the public gradients helps balance the effects of the public and private gradients on the optimization trajectory. 
In the original method stated in \cref{alg: weighted dp-sgd}, 
if unclipped public gradients and the private gradient are of very different magnitudes, 
then one gradient direction might dominate the optimization procedure, leading to a sub-optimal model.
Thus, our public gradient re-scaling technique promotes a more balanced update, which gracefully combines the public and private data in each iteration. To the best of our knowledge, this technique is novel.

\textit{Privacy accounting:} We compute the privacy loss of each algorithm by using the moments acccountant of Abadi et al.~\cite{abadi16}. For a fixed clip threshold $C$, privacy level $(\eps, \delta)$ is determined by three parameters: the variance of privacy noise $\sigma^2$, the private sampling ratio $q:= \kpr/\npr$, and the total number of iterations $T$. In our setting, the privacy parameters $(\varepsilon, \delta)$ are given, and we use the moments accountant to compute an approximation of $\sigma^2$ for any choice of hyperparameters $T$ and $q$. We utilize the implementation of the privacy accountant provided by the Pytorch privacy framework, Opacus.

\label{hyperparameter tuning}
\paragraph{Hyperparameter Tuning:} 

The results reported are for each algorithm with the hyperparameters (step size and $\alpha$ in Semi-DP) that attain the best performance for a given experiment. For simplicity and computation efficiency, we keep clip threshold $C=1$ for all of our experiments. Preliminary experiments found that a clip threshold  of $C = 1$ worked well for all algorithms. To tune all algorithms, we use grid search. See~\Cref{tab:grid_dp_sgd,tab:baye-lr}~in~\cref{sec: hyperparameter search} for detailed descriptions of the hyperparameter search grids.

\begin{figure}
    \centering
    \begin{minipage}{.4\textwidth}
        \centering
        \includegraphics[width=\textwidth]{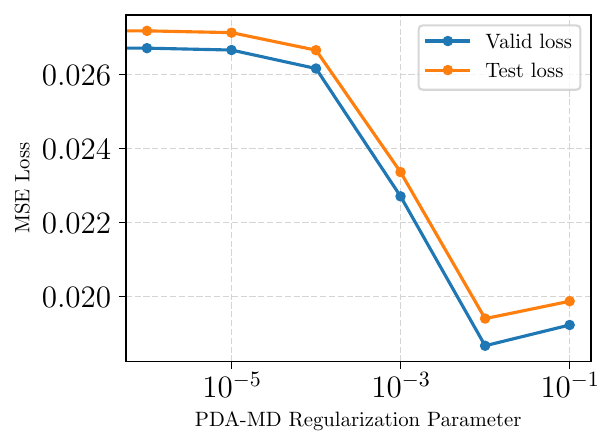}        
    \end{minipage}
    \hspace{0.25in}
    \begin{minipage}{.4\textwidth}
        \centering
        \includegraphics[width=\textwidth]{MSE_PDA_MD_Regualrization_Parameter.pdf}
    \end{minipage}
    \caption{Loss v.s. PDA-MD Regularization Parameter. Left: Results on proposed dataset in~\cite{amid2022public}. Right: Results on our dataset. PDA-MD is sensitive to hessian regularization parameter, which requires extra tuning on complicated tasks.}
    \label{fig:hessian_reg}
\end{figure}

\begin{figure*}[h]
    \centering
    \begin{minipage}{.4\textwidth}
        \includegraphics[width=2.5in]{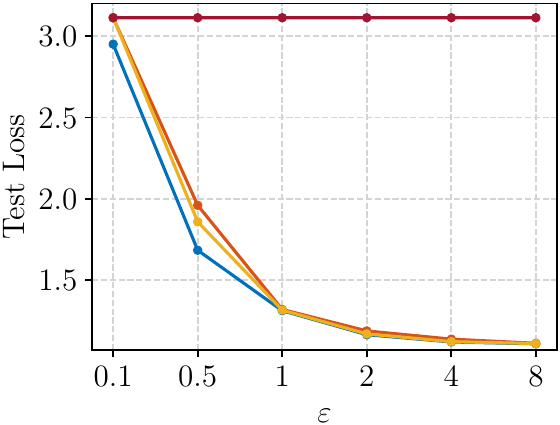}
        \caption{Test loss vs. $\eps$. $\frac{\npu}{n}=0.1$.}
        \label{fig:acc_vs_eps_ratio_0_1}
    \end{minipage}
    \begin{minipage}{.4\textwidth}
        \includegraphics[width=2.5in]{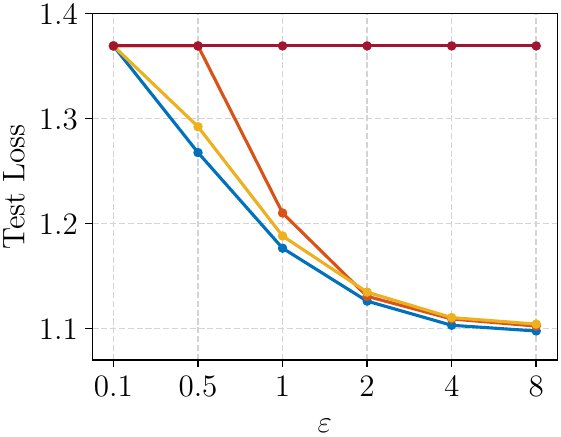}  
        \caption{Test loss vs. $\eps$. $\frac{\npu}{n}=0.25$.}
        \label{fig:acc_vs_eps_ratio_0_25}
    \end{minipage}
    \begin{minipage}{.15\textwidth}
        \includegraphics[width=1.3\textwidth]{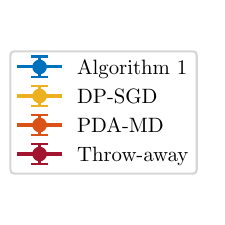} 
    \end{minipage}
\end{figure*}

\begin{figure*}[h]
    \centering
    \begin{minipage}{.4\textwidth}
        \includegraphics[width=2.5in]{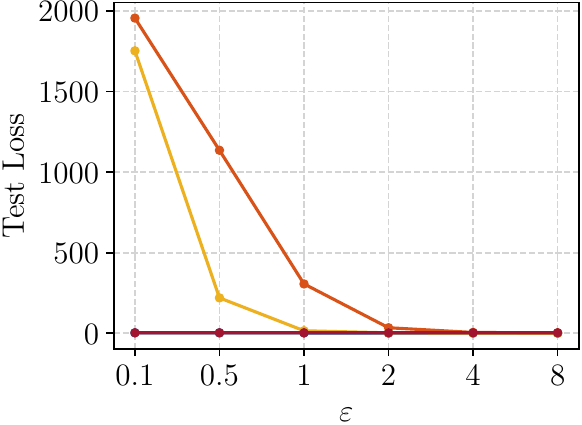}  
        \caption{Test loss vs. $\eps$. $\frac{\npu}{n}=0.1$, without warm-start.}  \label{fig:acc_vs_eps_ratio_0_1_non_warm}
    \end{minipage}
    \begin{minipage}{.15\textwidth}
        \vspace{-1.1in}
        \hspace{1in}
    \includegraphics[width=1.3\textwidth]{legend.pdf} 
    \end{minipage}
\end{figure*}

\subsubsection{Semi-DP Mean Estimation}
\label{app:semi-dp-mean-estimation}

We present an experiment with mean estimation: We fix $\rho$-semi-zCDP with privacy parameter $\rho = 0.5$. We draw $n$ i.i.d. samples from a $(d = 1000)$-dimensional $\textbf{Bernoulli}(p = 1/2)$ product distribution.
We investigated the effect of the ratio $\frac{n_\text{pub}}{n}$ on mean $\ell_2$ error, for values of $\frac{n_\text{pub}}{n}$ ranging from $0.05$ to $0.95$. We presented the experiment in three different high-dimensional and low-dimensional settings: $n<d$, $n=d$, and $n>d$. \Cref{fig:mean-est-d-1000-n-500-appendix} shows that when $n<d$, throw-away outperforms Gaussian mechanism. \Cref{fig:mean-est-d-1000-n-1000-appendix} and~\cref{fig:mean-est-d-1000-n-2000-appendix} show that when $d \leq n$, throw-away outperforms Gaussian mechanism except when $\frac{n_\text{pub}}{n} = 0.05$. Moreover, \textbf{our Weighted Gaussian estimator outperforms both throw-away and the Gaussian mechanism in every case}. 

\vspace{-0.1in}
\begin{figure}[h]
    \begin{minipage}{.32\textwidth}
        \centering
        \includegraphics[width=0.95\linewidth]{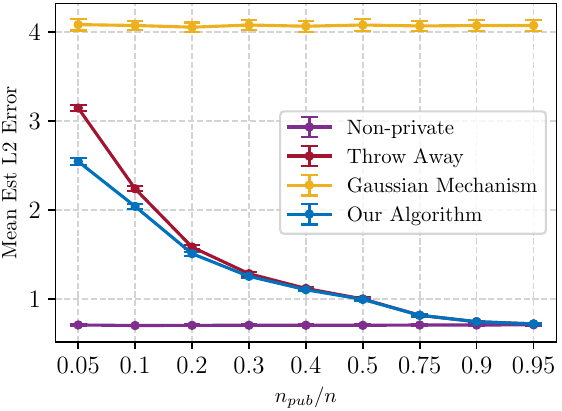}
        \caption{$d=1000$, $n=500$}
        \label{fig:mean-est-d-1000-n-500-appendix}
    \end{minipage}
    \hfill
    \begin{minipage}{.32\textwidth}
        \centering
        \includegraphics[width=0.95\linewidth]{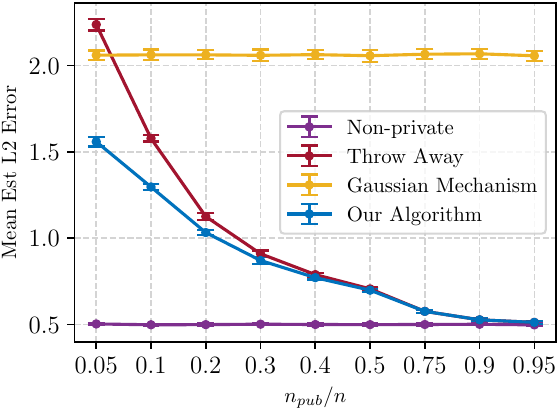}
        \caption{$d=1000$, $n=1000$}
        \label{fig:mean-est-d-1000-n-1000-appendix}
    \end{minipage}
    \hfill
    \begin{minipage}{.32\textwidth}
        \centering
        \includegraphics[width=0.95\linewidth]{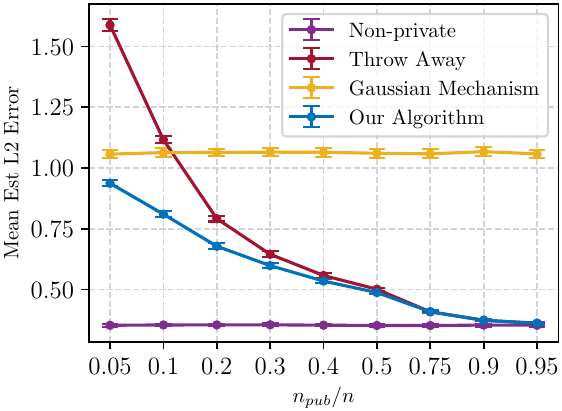}
        \caption{$d=1000$, $n=2000$}
        \label{fig:mean-est-d-1000-n-2000-appendix}
    \end{minipage}
\end{figure}

\subsubsection{Semi-DP Logistic Regression with CIFAR-10}
\label{app:cifar10-setup}

We evaluate the performance of~\cref{alg: weighted dp-sgd} in training a logistic regression model to classify digits in the CIFAR-10 dataset~\cite{krizhevsky2009learning}. We compare~\cref{alg: weighted dp-sgd} against DP-SGD and throw-away. Note that PDA-MD does not have an efficient implementation for logistic regression. This is because there does not exist a closed form update rule for the mirror descent step. Thus, we do not compare against PDA-MD. 

We flatten the images and feed them to the logistic (softmax) model. Cross-entropy loss is used here; therefore, the model is convex. Implementations of the algorithms are similar to the linear regression case. However, in this case, throw-away consists of running non-private SGD on $\xpu$ to find an approximate minimizer $w_{\text{pub}}$. For all three algorithms, we fixed batch-size 256 and privacy parameter $\delta=10^{-6}.$ The remaining hyperparameters are tuned by grid search, using the same grid for each algorithm. Also, in contrast to the linear regression experiments, \textit{we do not use warm-start} for any of the algorithms and \textit{we use SGD} for all algorithms in these experiments. See~\Cref{tab:baye-cifar}~in~\cref{sec: hyperparameter search} for detailed descriptions of the hyperparameter search grids.

Results are reported in~\Cref{fig:cifar_acc_vs_ratio_1,fig:cifar_acc_vs_ratio_2,fig:cifar_acc_vs_eps_1,fig:cifar_acc_vs_eps_2,}. 
\textit{Our~\cref{alg: weighted dp-sgd} always outperforms the baselines} in terms of minimizing test accuracy/error. The advantage of~\cref{alg: weighted dp-sgd} over these baselines is even more pronounced than it was for linear regression. This might be partially due to the fact that we did not use warm-start. 
\begin{figure}[H]
    \centering
    \begin{minipage}{.25\textwidth}
        \includegraphics[width=0.95\textwidth]{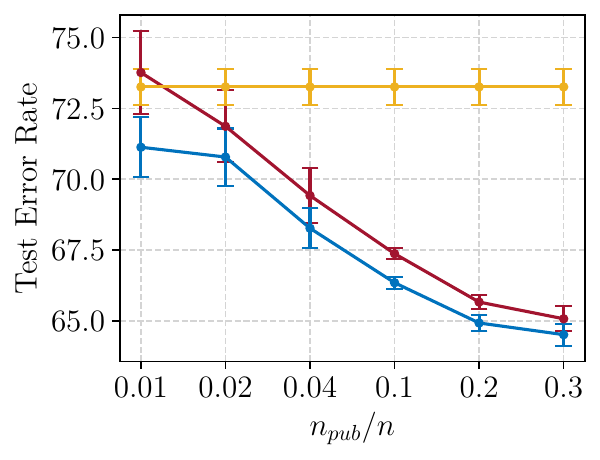}
    \end{minipage}
    \begin{minipage}{.25\textwidth}
        \includegraphics[width=0.95\textwidth]{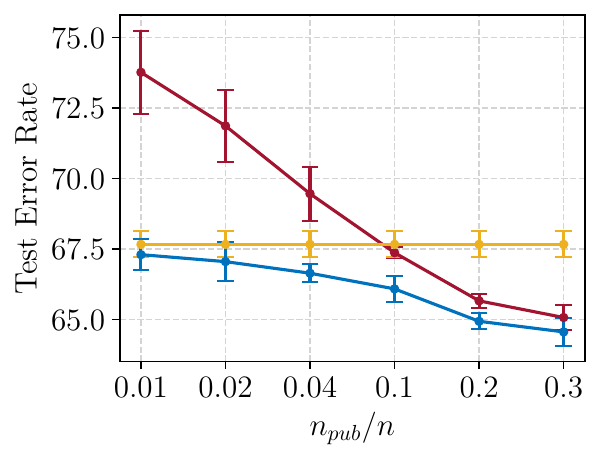} 
    \end{minipage}
    \begin{minipage}{.25\textwidth}
        \includegraphics[width=0.95\textwidth]{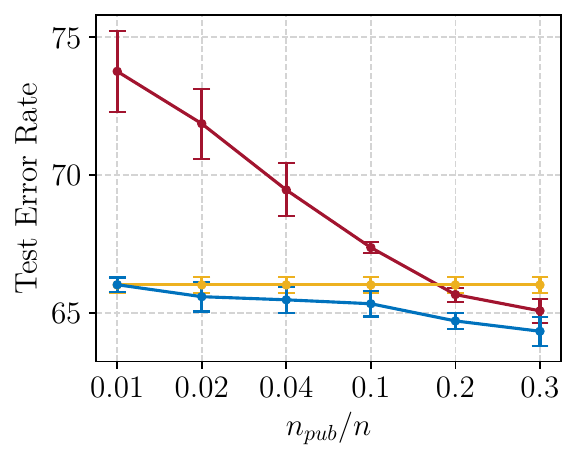} 
    \end{minipage}
    \begin{minipage}{.05\textwidth}
        \vspace{-0.2in}
    \includegraphics[width=3\textwidth]{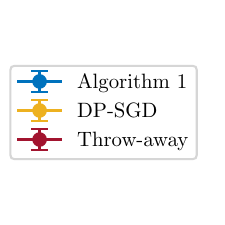} 
    \end{minipage}
    \caption{Test error rate vs. $\frac{\npu}{n}$. Left: $\varepsilon=0.1$. Middle: $\varepsilon=0.5$. Right: $\varepsilon=1.0$}
    \label{fig:cifar_acc_vs_ratio_1}
\end{figure}

\begin{figure}[h]
    \centering
    \begin{minipage}{.25\textwidth}
        \includegraphics[width=0.95\textwidth]{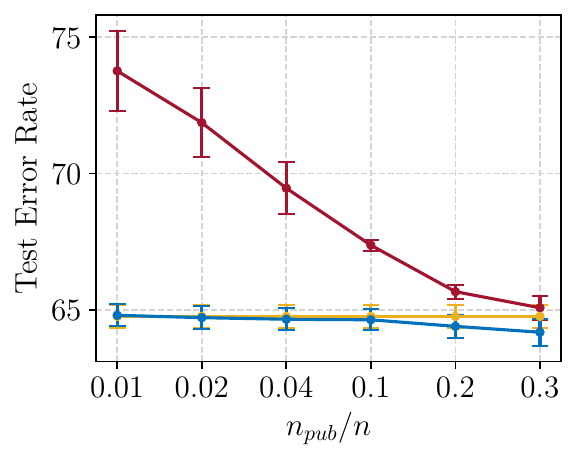}
    \end{minipage}
    \begin{minipage}{.25\textwidth}
        \includegraphics[width=0.95\textwidth]{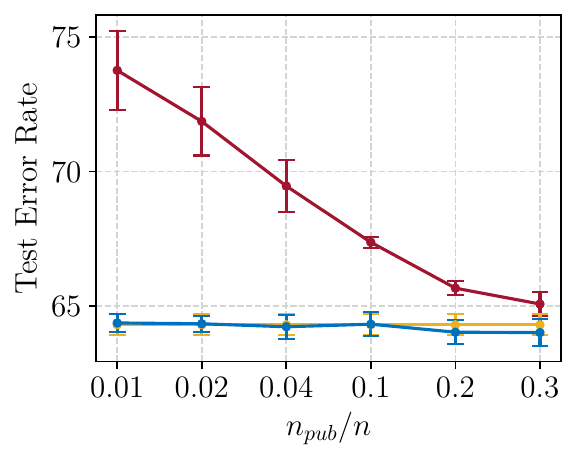} 
    \end{minipage}
    \begin{minipage}{.25\textwidth}
        \includegraphics[width=0.95\textwidth]{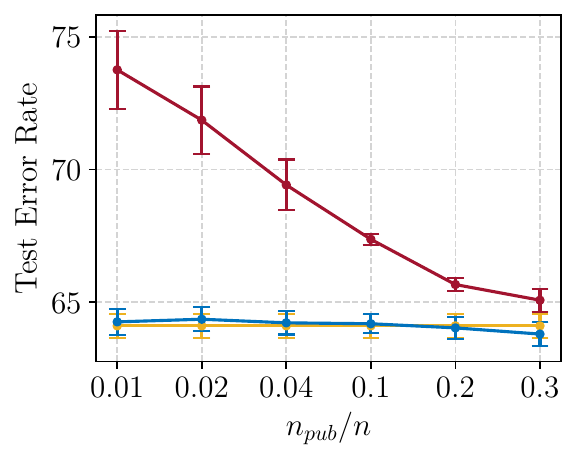} 
    \end{minipage}
    \begin{minipage}{.05\textwidth}
        \vspace{-0.2in}
    \includegraphics[width=3\textwidth]{legend_cifar10.pdf} 
    \end{minipage}
    \caption{Test error rate vs. $\frac{\npu}{n}$. Left: $\varepsilon=2.0$. Middle: $\varepsilon=4.0$. Right: $\varepsilon=8.0$}
    \label{fig:cifar_acc_vs_ratio_2}
\end{figure}

\begin{figure}[h]
    \centering
    \begin{minipage}{.25\textwidth}
        \includegraphics[width=0.95\textwidth]{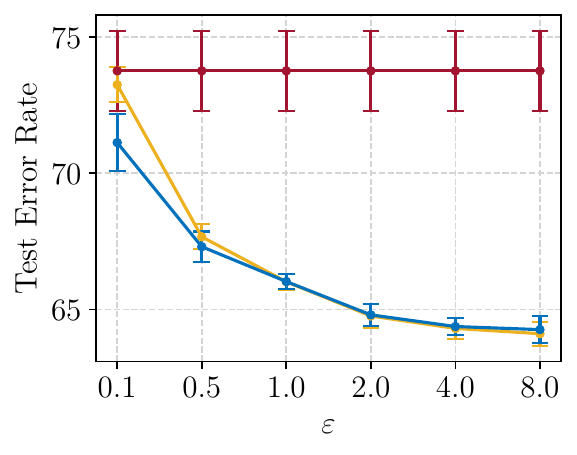}
    \end{minipage}
    \begin{minipage}{.25\textwidth}
        \includegraphics[width=0.95\textwidth]{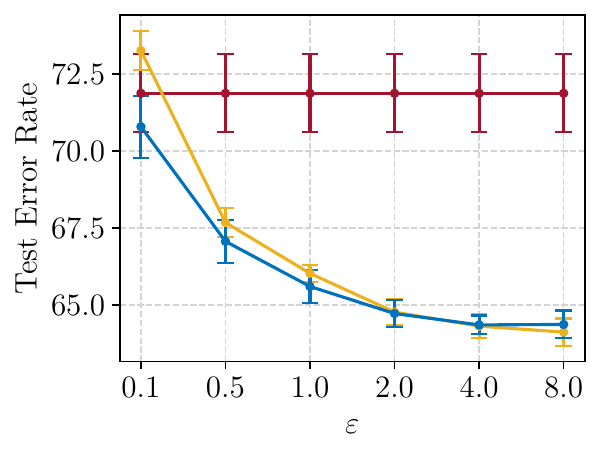} 
    \end{minipage}
    \begin{minipage}{.25\textwidth}
        \includegraphics[width=0.95\textwidth]{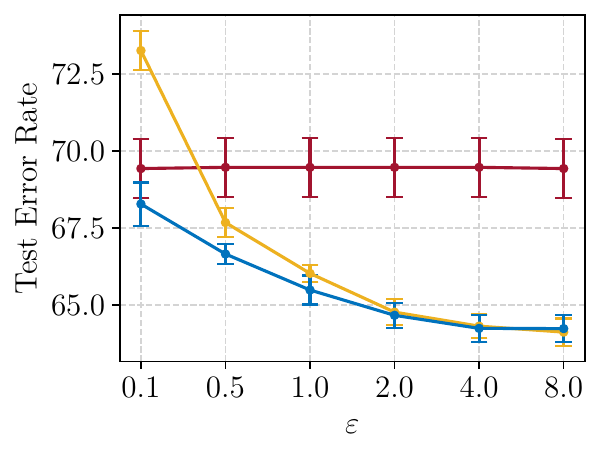} 
    \end{minipage}
    \begin{minipage}{.05\textwidth}
        \vspace{-0.2in}
    \includegraphics[width=3\textwidth]{legend_cifar10.pdf} 
    \end{minipage}
    \caption{Test error rate vs. $\varepsilon$. Left: $\frac{\npu}{n}=0.01$. Middle: $\frac{\npu}{n}=0.02$. Right: $\frac{\npu}{n}=0.04$}
    \label{fig:cifar_acc_vs_eps_1}
\end{figure}

\begin{figure}[H]
    \centering
    \begin{minipage}{.25\textwidth}
        \includegraphics[width=0.95\textwidth]{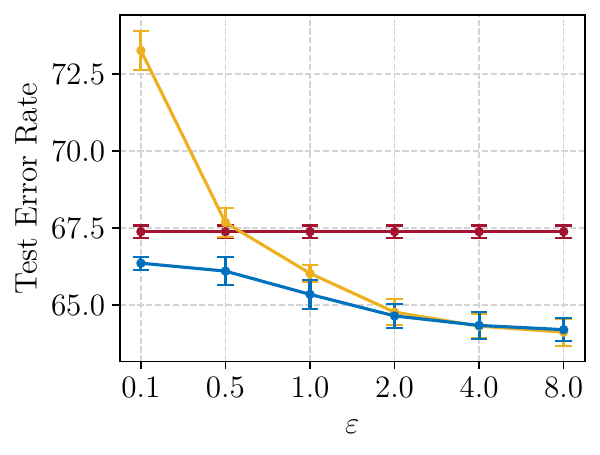}
    \end{minipage}
    \begin{minipage}{.25\textwidth}
        \includegraphics[width=0.95\textwidth]{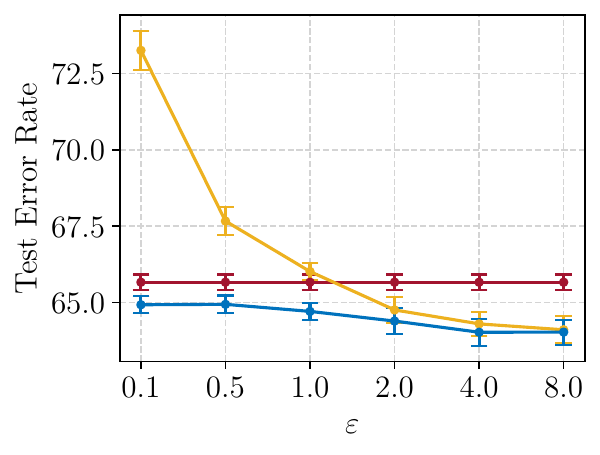} 
    \end{minipage}
    \begin{minipage}{.25\textwidth}
        \includegraphics[width=0.95\textwidth]{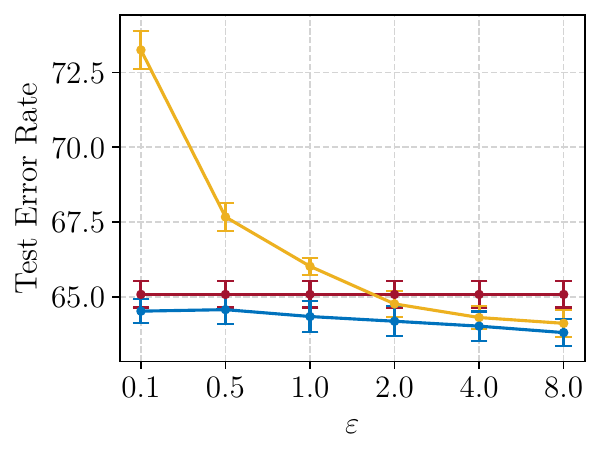} 
    \end{minipage}
    \begin{minipage}{.05\textwidth}
        \vspace{-0.2in}
    \includegraphics[width=3\textwidth]{legend_cifar10.pdf} 
    \end{minipage}
    \caption{Test error rate vs. $\varepsilon$. Left: $\frac{\npu}{n}=0.1$. Middle: $\frac{\npu}{n}=0.2$. Right: $\frac{\npu}{n}=0.3$}
    \label{fig:cifar_acc_vs_eps_2}
\end{figure}

\subsubsection{Semi-DP Wide-ResNet-16-4 with CIFAR-10}
\label{app:cifar10-wrn-14-4-setup}

To evaluate the performance of~\cref{alg: weighted dp-sgd} in training non-convex (neural) models, we use the Wide-ResNet with 16 layers and a width factor of 4 (WRN-16-4)~\cite{zagoruyko2017wide}. When trained non-privately on CIFAR-10 dataset, this model achieves an error rate of $5.02$~\cite{zagoruyko2017wide}. 

\paragraph{Experimental Setup:} We followed the same experimental setup as~\cite{de2022unlocking}. For all algorithms, 
we used a constant learning rate and did not use momentum, weight decay, or dropout. 
We fixed the batch size of 256 and privacy parameter $\delta=10^{-5}.$ For simplicity and to isolate the effects of the weighted gradient estimator, we used the WRN-16-4 without batch/group normalization or data augmentation.\footnote{We believe that accuracy could be improved by combining these tricks with our semi-DP gradient estimator~\cite{de2022unlocking,nasr23a}.}We split the CIFAR-10 dataset~\cite{krizhevsky2009learning} of 50,000 training examples into 40,000 training samples and 10,000 validation samples. We used their 10,000 test images as our test set. Same as~\cref{app:cifar10-setup}, we did not use warm-start here.

\paragraph{Hyperparameter Tuning:} We selected the hyperparameters settings with the highest validation accuracy for all algorithms and reported their test accuracy on the official test set. To tune the hyperparameters, we used the bayesian hyperparameter optimization technique. That is, we build a probability
model of the objective function and use it to select the most promising hyperparameters to evaluate in the true objective function. See~\cref{tab:baye-cifar-wrn-16-4}~in~\cref{sec: hyperparameter search} for detailed descriptions of the hyperparameter search range.

\paragraph{Results:} 
We investigate two cases: 1) the effect of the ratio of public samples $\frac{\npu}{n}$ on accuracy when $\eps=8$ is fixed. We test values of $\frac{\npu}{n}$ ranging from $0.01$ to $0.3$; and 2) the effect of privacy (quantified by $\varepsilon$) on accuracy, for fixed $\frac{\npu}{n} = 0.04$. We vary $\varepsilon \in \{0.1, 0.5, 1.0, 2.0, 4.0, 8.0\}$. Results are reported in~\Cref{fig:loss_vs_ratio_cifar_10_non_cvx_1,fig:loss_vs_ratio_cifar_10_non_cvx_2}. 
\textit{Our~\cref{alg: weighted dp-sgd} always outperforms the baselines} (DP-SGD and Throw-away) in terms of minimizing test error, across all experimental setups.

\begin{figure}[H]
    \centering
    \begin{minipage}{.4\textwidth}
        \includegraphics[width=.9\textwidth]{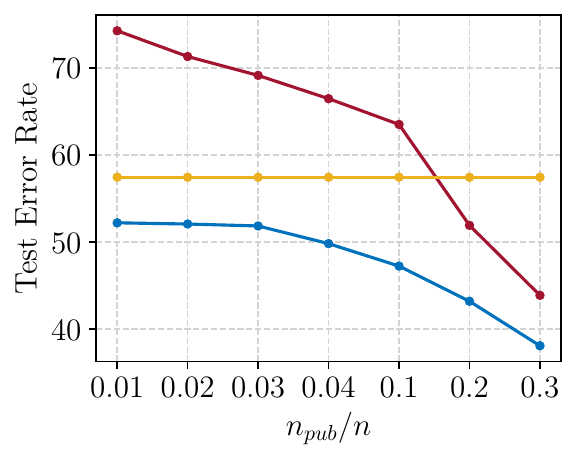}
        \vspace{-0.3in}
    \end{minipage}
    \begin{minipage}{.4\textwidth}
        \includegraphics[width=.9\textwidth]{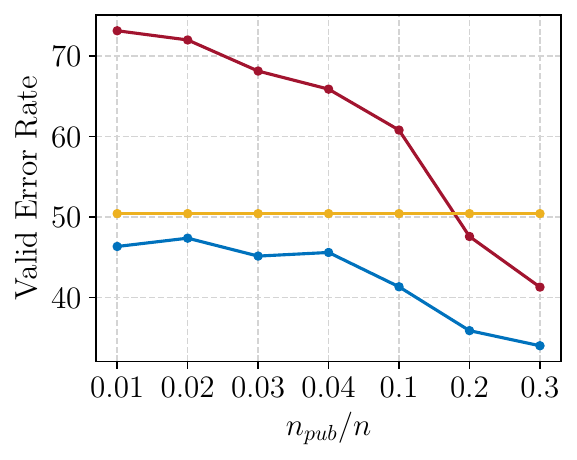} 
        \vspace{-0.3in}
    \end{minipage}
    \begin{minipage}{.15\textwidth}
        \vspace{-0.3in}
    \includegraphics[width=1.2\textwidth]{legend_cifar10.pdf} 
    \end{minipage}
    \vspace{0.1in}
    \caption{Left: Test error rate vs. $\frac{\npu}{n}$. Right: Validation error rate vs. $\frac{\npu}{n}$. $\eps=8$} \label{fig:loss_vs_ratio_cifar_10_non_cvx_1}
\end{figure}

\begin{figure}[H]
    \centering
    \begin{minipage}{.4\textwidth}
        \includegraphics[width=.9\textwidth]{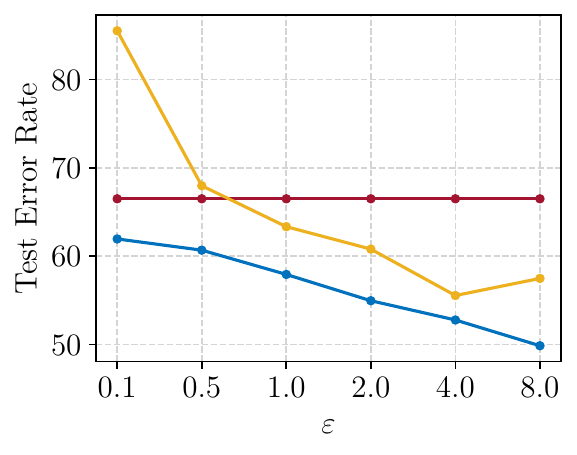}
        \vspace{-0.3in}
    \end{minipage}
    \begin{minipage}{.4\textwidth}
        \includegraphics[width=.9\textwidth]{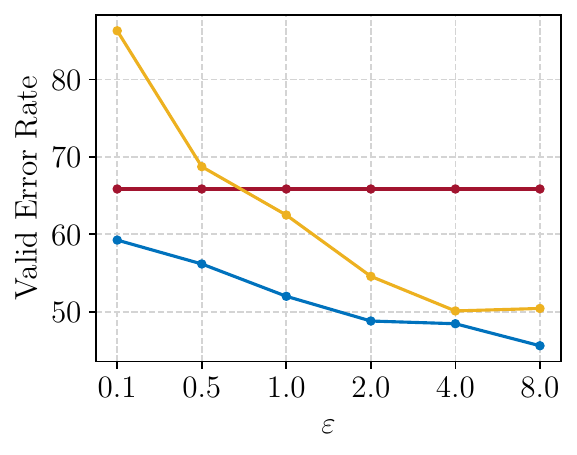} 
        \vspace{-0.3in}
    \end{minipage}
    \begin{minipage}{.15\textwidth}
        \vspace{-0.3in}
    \includegraphics[width=1.2\textwidth]{legend_cifar10.pdf} 
    \end{minipage}
    \vspace{0.1in}
    \caption{Left: Test error rate vs. $\eps$. Right: Validation error rate vs. $\eps$. $\frac{\npu}{n}=0.04$} \label{fig:loss_vs_ratio_cifar_10_non_cvx_2}
\end{figure}

\subsection{Hyperparameters Search Grids}
\label{sec: hyperparameter search}
\begin{table}[H]
\centering
\begin{tabular}{@{}ccc@{}}
\toprule
hyperparameter &   learning-rate \\ \midrule
value          &  $0, 0.01, 0.03, 0.05, 0.07, 0.09, 0.1, 0.3, 0.5, 0.7, 0.9, 1.1, 1.3, 1.5, 1.7, 1.9$                   \\ \bottomrule
\end{tabular}
\caption{Grid used for hyperparameter search for PDA-MD and DP-SGD in~\Cref{app: linreg setup}}
\label{tab:grid_dp_sgd}
\end{table}

\begin{table}[H]
\centering
\begin{tabular}{@{}cccc@{}}
\toprule
hyperparameter &  learning-rate & $\alpha$ \\ \midrule
value          & $0, 0.01, 0.03, 0.05, 0.07, 0.09, 0.1, 0.3, 0.5, 0.7, 0.9, 1.1, 1.3, 1.5, 1.7, 1.9$ & $0, 0.1, 0.2 \ldots 1$                    \\ \bottomrule
\end{tabular}
\caption{Grid used for hyperparameter search for \cref{alg: weighted dp-sgd} in~\Cref{app: linreg setup}}
\label{tab:baye-lr}
\end{table}

\begin{table}[H]
\centering
\begin{tabular}{@{}cc@{}}
\toprule
hyperparameter & value \\ \midrule
iterations & $2000, 4000, 6000, \ldots, 20000$\\
learning-rate &  $0.005, 0.01, 0.05, 0.1, 0.15, 0.2, 0.25, 0.5, 0.75, 1, 1.5, 2, 3, 4, 5$   \\
$\alpha$ & $0, 0.1, 0.2, \ldots, 1$       \\\bottomrule
\end{tabular}
\caption{Grid used for hyperparameter search for DP-SGD and Semi-DP in~\Cref{app:cifar10-setup}}
\label{tab:baye-cifar}
\end{table}

\begin{table}[H]
\centering
\begin{tabular}{@{}cccc@{}}
\toprule
hyperparameter & iterations &  learning-rate & $\alpha$ \\ \midrule
value          &  $[3000, 7000]$                & $[0.1,3]$ & $[0,1]$                    \\ \bottomrule
\end{tabular} \\
\caption{Range used for Bayesian hyperparameter search for DP-SGD and Semi-DP in~\Cref{app:cifar10-wrn-14-4-setup}}
\label{tab:baye-cifar-wrn-16-4}
\end{table}

\subsection{Exact Numerical Results}
\label{sec: tables}

\begin{table}[H]
\centering
\begin{tabular}{@{}ccccc@{}}
\toprule
$\npu/n$ & Algorithm 1 & PDA-MD    & DP-SGD & Throw-away \\ \midrule
0.01     & 2.3526      & 1689.5905 & 2.7935 & 1689.5905  \\
0.03     & 1.9719      & 1084.2257 & 2.1457 & 1084.2256  \\
0.04     & 1.7626      & 776.1302  & 2.1417 & 776.1301   \\
0.1      & 1.1648      & 1.1880    & 1.1695 & 3.1133     \\
0.25     & 1.1260      & 1.1306    & 1.1346 & 1.3691     \\
0.5      & 1.1043      & 1.1394    & 1.1065 & 1.1539     \\
0.75     & 1.0946      & 1.1013    & 1.0937 & 1.1013     \\
0.9      & 1.0787      & 1.0787    & 1.0787 & 1.0788     \\
0.95     & 1.0725      & 1.0725    & 1.0725 & 1.0725     \\ \bottomrule
\end{tabular}
\caption{Exact training results of curves reported in Figure~\ref{fig:acc_vs_ratio_eps_2}.}
\label{tab:acc_vs_ratio_eps_2_detailed}
\end{table}

\begin{table}[H]
\centering
\begin{tabular}{@{}ccccc@{}}
\toprule
$\npu/n$ & Algorithm 1 & PDA-MD    & DP-SGD & Throw-away \\ \midrule
0.01     & 1.5020      & 1689.5905 & 1.5206 & 1689.5905  \\
0.03     & 1.3363      & 1084.2257 & 1.3769 & 1084.2256  \\
0.04     & 1.3196      & 776.1302  & 1.3776 & 776.1302   \\
0.1      & 1.1201      & 1.1376    & 1.1221 & 3.1133     \\
0.25     & 1.1030      & 1.1090    & 1.1103 & 1.3691     \\
0.5      & 1.0911      & 1.1019    & 1.0999 & 1.1539     \\
0.75     & 1.0863      & 1.1013    & 1.0877 & 1.1013     \\
0.9      & 1.0787      & 1.0787    & 1.0787 & 1.0788     \\
0.95     & 1.0725      & 1.0725    & 1.0725 & 1.0725     \\ \bottomrule
\end{tabular}
\caption{Exact training results of curves reported in Figure~\ref{fig:acc_vs_ratio_eps_4}.}
\label{tab:acc_vs_ratio_eps_4_detailed}
\end{table}

\begin{table}[H]
\centering
\begin{tabular}{@{}ccccc@{}}
\toprule
$\varepsilon$ & Algorithm 1 & PDA-MD & DP-SGD & Throw-away \\ \midrule
0.10  & 2.9509   & 3.1133            & 3.1133     & 3.1133           \\
0.50  & 1.6841   & 1.9602            & 1.8588     & 3.1133           \\
1.00  & 1.3125   & 1.3201            & 1.3158     & 3.1133           \\
2.00  & 1.1648   & 1.1880            & 1.1695     & 3.1133           \\
4.00  & 1.1201   & 1.1376            & 1.1221     & 3.1133           \\
8.00  & 1.1088   & 1.1120            & 1.1104     & 3.1133           \\ \bottomrule
\end{tabular}
\caption{Exact training results of curves reported in Figure~\ref{fig:acc_vs_eps_ratio_0_1}.}
\label{tab:acc_vs_eps_ratio_0_1_detailed}
\end{table}

\begin{table}[H]
\centering
\begin{tabular}{@{}ccccc@{}}
\toprule
$\varepsilon$ & Algorithm 1 & PDA-MD & DP-SGD & Throw-away \\ \midrule
0.10 & 1.3691 & 1.3691 & 1.3691 & 1.3691 \\
0.50 & 1.2647 & 1.3691 & 1.2679 & 1.3691 \\
1.00 & 1.1764 & 1.2099 & 1.1880 & 1.3691 \\
2.00 & 1.1260 & 1.1306 & 1.1346 & 1.3691 \\
4.00 & 1.1030 & 1.1090 & 1.1103 & 1.3691 \\
8.00 & 1.0976 & 1.1023 & 1.1041 & 1.3691 \\ \bottomrule
\end{tabular}
\caption{Exact training results of curves reported in Figure~\ref{fig:acc_vs_eps_ratio_0_25}.}
\label{tab:acc_vs_eps_ratio_0_25_detailed}
\end{table}

\begin{table}[H]
\centering
\begin{tabular}{@{}ccccc@{}}
\toprule
$\npu/n$ & Algorithm 1 & PDA-MD    & DP-SGD & Throw-away \\ \midrule
0.01     & 1.5175      & 2010.2422 & 1.5411 & 1689.5905  \\
0.03     & 1.4225      & 2010.2422 & 1.5411 & 1084.2256  \\
0.04     & 1.3989      & 2010.2422 & 1.5413 & 776.1302   \\
0.1      & 1.3371      & 5.3822    & 1.5413 & 3.1133     \\
0.25     & 1.2574      & 2.6205    & 1.5411 & 1.3691     \\
0.5      & 1.1539      & 5.9698    & 1.5411 & 1.1539     \\
0.75     & 1.1013      & 78.9970   & 1.5411 & 1.1013     \\
0.9      & 1.0787      & 1053.6185 & 1.5411 & 1.0788     \\
0.95     & 1.0725      & 1662.3572 & 1.5413 & 1.0725     \\ \bottomrule
\end{tabular}
\caption{Exact training results of curves reported in Figure~\ref{fig:acc_vs_ratio_eps_4_non_warm}.}
\label{tab:acc_vs_ratio_eps_4_non_warm_detailed}
\end{table}

\begin{table}[H]
\centering
\begin{tabular}{@{}ccccc@{}}
\toprule
$\varepsilon$ & Algorithm 1 & PDA-MD & DP-SGD & Throw-away \\ \midrule
0.10 & 3.1133 & 1956.6069 & 1830.1820 & 3.1133 \\
0.50 & 3.1133 & 1092.1338 & 213.6728 & 3.1133 \\
1.00 & 2.1843 & 306.5518 & 15.8089 & 3.1133 \\
2.00 & 1.6313 & 34.4665 & 2.7226 & 3.1133 \\
4.00 & 1.3359 & 5.3648 & 1.4746 & 3.1133 \\
8.00 & 1.1986 & 2.4319 & 1.2472 & 3.1133 \\ \bottomrule
\end{tabular}
\caption{Exact training results of curves reported in Figure~\ref{fig:acc_vs_eps_ratio_0_1_non_warm}.}
\label{tab:acc_vs_eps_ratio_0_1_non_warm}
\end{table}

\begin{table}[H]
\centering
\begin{tabular}{cccccccccc}
\toprule
$\npu/n$ & 0.01 & 0.03 & 0.04 & 0.1 & 0.25 & 0.5 & 0.75 & 0.9 & 0.95 \\ \midrule
$\sigma$ & 2.490 & 2.529 & 2.568 & 2.744 & 3.252 & 4.805 & 9.531 & 23.672 & 47.344 \\ \bottomrule
\end{tabular}
\caption{Standard Deviation $\sigma$ of the private noise $v_t$ used in experiment shown in Figure~\ref{fig:acc_vs_ratio_eps_2}.}
\label{table: std_dev_eps_2}
\end{table}

\begin{table}[H]
\centering
\begin{tabular}{cccccccccc}
\toprule
$\npu/n$ & 0.01 & 0.03 & 0.04 & 0.1 & 0.25 & 0.5 & 0.75 & 0.9 & 0.95 \\ \midrule
$\sigma$ & 1.470 & 1.489 & 1.509 & 1.597 & 1.860 & 2.671 & 5.176 & 12.812 & 25.586 \\ \bottomrule
\end{tabular}
\caption{Standard Deviation $\sigma$ of the private noise $v_t$ used in experiment shown in Figure~\ref{fig:acc_vs_ratio_eps_4}~and~\ref{fig:acc_vs_ratio_eps_4_non_warm}.}
\label{table: std_dev_eps_4}
\end{table}

\section{Limitations}
\paragraph{Limitations of Theoretical Results:} Our theoretical results rely on certain assumptions (e.g convex, Lipschitz loss, i.i.d. data for SCO), that may be violated in certain applications. We leave it as future work to investigate the questions considered in this work under different assumptions (e.g. non-convexity, semi-DP SCO with \textit{out-of-distribution} public data). Also, we reiterate that our theoretical results describe the optimal \textit{worst-case} error. It might be the case that the worst-case distributions we construct in our lower bound proofs are unlikely to appear in practice. Thus, another interesting direction for future work would be to analyze ``instance-optimal''~\cite{asi2020near,asi20instanceoptimal} semi-DP error rates. 

\paragraph{Limitations of Experiments:} It is important to note that pre-processing and hyperparameter tuning were not done in a DP manner, since we did not want to detract focus from evaluation of the (fully tuned) semi-DP algorithms.\footnote{See, e.g. \cite{liu2019private, steinkehyper} and the references therein for discussion of DP hyperparameter tuning.} As a consequence, the overall privacy loss for the entire experimental process is higher than the $\eps$ indicated in the plots: the $\eps$ indicated in the plots solely reflects the privacy loss from running the algorithms with fixed hyperparameters and (pre-processed) data.

\end{document}